\documentclass[accepted]{uai2026} % after acceptance, for a revised version; 
\usepackage{titletoc}
\usepackage[american]{babel}
\usepackage{algorithm}
\usepackage{algpseudocode}
\usepackage{arydshln}
\usepackage{hyperref}

\usepackage{booktabs}
\usepackage{graphicx}

\newtheorem{lemma}{Lemma}
\newtheorem{theorem}{Theorem}
\newtheorem{assumption}{Assumption}

\newtheorem{corollary}{Corollary} % corollary numbered like 

\newcommand{\best}[1]{%
	{\textbf{\color[HTML]{D52815}#1}}%
}
\newcommand{\secondbest}[1]{%
	{\underline{\color[HTML]{00008A}#1}}%
}

\usepackage{natbib} % has a nice set of citation styles and commands
\usepackage{mathtools} % amsmath with fixes and additions
\usepackage{booktabs} % commands to create good-looking tables
\usepackage{tikz} % nice language for creating drawings and diagrams
\usepackage{array}

\usepackage{amsmath,amsfonts,bm}

\algnewcommand{\algorithmicswitch}{\textbf{switch}}
\algnewcommand{\algorithmiccase}{\textbf{case}}
\algnewcommand{\algorithmicendcase}{}
\algnewcommand{\algorithmicendswitch}{\textbf{end switch}}

\algdef{SE}[SWITCH]{Switch}{EndSwitch}[1]{\algorithmicswitch\ #1}{\algorithmicendswitch}
\algdef{SE}[CASE]{Case}{EndCase}[1]{\algorithmiccase\ #1:}{\algorithmicendcase}

\newcommand{\qedbox}{\hfill\ensuremath{\square}}
\def\eqref#1{equation~\ref{#1}}
\def\1{\bm{1}}

\newcommand{\BigO}{\mathcal{O}}

\def\ve{{\bm{e}}}

\def\vg{{\bm{g}}}

\def\vn{{\bm{n}}}

\def\vu{{\bm{u}}}

\def\vx{{\bm{x}}}
\def\vy{{\bm{y}}}
\def\vz{{\bm{z}}}

\def\mA{{\bm{A}}}
\def\mB{{\bm{B}}}

\def\mF{{\bm{F}}}

\def\mI{{\bm{I}}}

\def\mP{{\bm{P}}}

\def\mR{{\bm{R}}}
\def\mS{{\bm{S}}}

\def\mZ{{\bm{Z}}}

\DeclareMathAlphabet{\mathsfit}{\encodingdefault}{\sfdefault}{m}{sl}
\SetMathAlphabet{\mathsfit}{bold}{\encodingdefault}{\sfdefault}{bx}{n}

\def\gF{{\mathcal{F}}}

\def\gN{{\mathcal{N}}}

\def\gS{{\mathcal{S}}}

\def\gU{{\mathcal{U}}}

\def\sN{{\mathbb{N}}}

\def\sP{{\mathbb{P}}}

\def\sR{{\mathbb{R}}}

\newcommand{\E}{\mathbb{E}}

\DeclareMathOperator*{\argmin}{arg\,min}

\title{Variance Reduction for Non-Log-Concave Sampling \\with Applications to Inverse Problems}

\author[1]{M. Berk Sahin}
\author[2]{Ahmet Ege Tanriverdi}
\author[1,3]{Behzad Sharif}
\author[1]{Abolfazl Hashemi}
\affil[1]{%
    School of Electrical and Computer Engineering\\
    Purdue University\\
    West Lafayette, Indiana, USA
}
\affil[2]{%
    School of Electrical and Computer Engineering\\
    University of Southern California\\
    California, USA
}
\affil[3]{%
    School of Biomedical Engineering\\
    Purdue University\\
    West Lafayette, Indiana, USA
}
  
\begin{document}
\maketitle
\begin{abstract}
    Sampling from high-dimensional, non-log-concave distributions with unnormalized densities is a fundamental challenge in machine learning, particularly when the exact gradient of the potential is unavailable and must be approximated via stochastic gradients that exhibit high variance under a fixed budget of gradient computations per iteration. Although variance reduction techniques such as SGD with momentum, STORM, and PAGE have demonstrated improved convergence properties in non-convex optimization, their implications for sampling from non-log-concave distributions remain largely unexplored. In this work, we develop the first unified analysis of these estimators for sampling from non-log-concave distributions. We establish improved non-asymptotic convergence rates in $\varepsilon$-relative Fisher information and, under a Poincar\'e inequality assumption, in squared total variation distance, and further prove weak convergence to the target distribution. We extend our analysis to solving inverse problems with score-based generative priors. We empirically validate our theory and demonstrate that, under a fixed gradient computations per iteration, variance-reduction techniques consistently improve sample quality in two standard imaging applications. 
\end{abstract}
\section{Introduction}\label{sec:intro}
One of the central themes in machine learning is the use of stochasticity to design procedures that are both computationally efficient and statistically consistent. This idea is shared by both Bayesian and frequentist perspectives. From a Bayesian perspective, posterior distributions provide a principled characterization of uncertainty in estimates, and stochasticity is used to represent those distributions through Monte Carlo (MC) samples. On the frequentist side, stochasticity is used to develop computationally efficient gradient-based optimization methods for obtaining point estimates by replacing intractable or expensive exact gradients with unbiased stochastic estimates. Although they are two different frameworks, there are overlapping methodological challenges. In particular, MC sampling must drive an out-of-equilibrium state towards the steady-state posterior distribution, while achieving fast convergence; consequently, optimization ideas are relevant. On the other hand, frequentist inference often depends on sampling and resampling procedures, which makes MC methods relevant.

Variance reduction has emerged as one of the key tools to address the challenges shared by the two frameworks. In the context of stochastic optimization, stochastic gradient descent (SGD) replaces exact gradients with unbiased stochastic estimates to reduce the computational burden of large-scale learning problems. However, this computational advantage comes at the cost of additional variance, which can slow convergence and lead to a suboptimal stationary point. This raises the question of whether there is a principled way to manage this trade-off. This question has been answered positively in a seminal line of work on stochastic optimization~\citep{johnson2013accelerating, schmidt2017minimizing, fang2018spider, cutkosky2019momentum, li2021page}. These methods achieve improved convergence rates compared to SGD while relying on stochastic gradients with lower gradient complexity.

In parallel, a related line of research has developed in the Bayesian framework based on MC sampling. Stochastic gradient Langevin diffusion (SGLD) has been proposed to approximate the Langevin dynamics via unbiased stochastic gradients~\citep{welling2011bayesian} and several theoretical results have established mixing time bounds for both well behaved (e.g., log-concave) distributions~\citep{cheng2018convergence, dalalyan2019user, crespo2024stochastic} and general non-log-concave settings~\citep{raginsky2017non, cheng2020stochastic, balasubramanian2022towards}. These developments have in turn motivated variance reduction techniques for SGLD~\citep{dubey2016variance, zou2018subsampled, chatterji2018theory, ding2020variance}. However, these studies rely on the restrictive assumption that the target distribution is strongly log-concave. Although~\citet{zou2019sampling} establish non-asymptotic convergence rates for SAGA- and SVRG-based Langevin diffusion, their analysis relies on a finite-sum structure of the log density together with a dissipativity assumption, which does not hold for non-log-concave distributions in general.

In this work, we develop a unified framework that adapts variance reduction techniques from non-convex optimization to sampling from non-log-concave distributions under relaxed assumptions. In particular, we propose two estimators that provide unified treatments of SGD with momentum and its probabilistic counterpart, and of STORM~\citep{cutkosky2019momentum} and PAGE~\citep{li2021page}. We refer to these estimators as \textit{mean-Lipschitz variance-reduced Langevin diffusion (ML-VRLD)} and \textit{sample-Lipschitz variance-reduced Langevin diffusion (SL-VRLD)}, respectively. This terminology reflects the underlying regularity conditions: ML-VRLD assumes Lipschitz continuity of the exact gradients, whereas SL-VRLD assumes Lipschitz continuity of unbiased stochastic gradients in the mean-squared sense. The proposed methods improve upon the non-log-concave analysis of SGLD~\citep{balasubramanian2022towards}, which establishes $\varepsilon$-relative Fisher information convergence with $\BigO(Ld\sigma^2/\varepsilon)$ stochastic gradient computations per iteration and a total computational complexity of $\BigO(L^3d^3\sigma^2/\varepsilon^3)$ (see Theorem~\hyperref[thm:sgld-batch-size]{B.1} in Appendix~\ref{appendix:theorems}). Here, $d$ denotes the dimension, $L$ the gradient Lipschitz constant, and $\sigma^2$ an upper bound on the variance of the stochastic gradients. In contrast to SGLD, both ML-VRLD and SL-VRLD achieve the same notion of stationarity with $\BigO(1)$ stochastic gradient computations per iteration, providing memory-efficient algorithms under different assumptions. Moreover, SL-VRLD improves the total computational complexity to $\BigO(L^2d^2 \sigma^2/\varepsilon^3)$. We additionally establish convergence rates in squared total variation distance under a Poincar\'e inequality assumption and prove weak convergence to the target distribution for both algorithms. 

One of the most important practical applications of sampling is solving inverse problems, where score-based generative models serve as implicit priors for modeling complex high-dimensional distributions~\citep{song2021scorebased, chung2022diffusion}. Recently, several works have explored high-quality image generation as a solution to inverse problems across diverse applications, including image deblurring, MRI and CT reconstruction~\citep{chung2022score, rout2023solving, wu2024principled, chen2025cross}. However, most existing posterior sampling methods that employ SGM priors work in a \textit{full-batch} setting, where the entire set of measurements is used at every iteration. This effectively limits their applicability to large datasets~\citep{bottou2007tradeoffs} common in three-dimensional imaging (3D) or in dynamic imaging~\citep{kamilov2016optical, ong2020extreme}. To the best of our knowledge, existing posterior sampling methods with SGM priors do not address this limitation. Related approaches have instead been proposed in other settings~\citep{sun2019online, tang2020fast, liu2022online}. However, these methods rely on mini-batch gradient estimates without variance-reduction mechanisms, resulting in high-variance updates that necessitate larger batch sizes in high-noise regimes. Our second contribution is that we propose two principled, variance-reduced posterior sampling algorithms for inverse problems based on annealed Langevin Monte Carlo (LMC) sampling with an SGM prior. Assuming bounded SGM estimation error, we establish convergence in relative FI, squared TV distance under Poincar\'e inequality assumption, and prove weak convergence to the target distribution, with $\BigO(1)$ gradient computations per iteration. We corroborate our theoretical findings with extensive experiments. The code is provided in the \href{https://github.com/mberk-sahin/variance-reduced-sampling.git}{\textcolor{blue}{project page}.} 

Our contributions are summarized as follows:
\begin{itemize}[leftmargin=*]
    \item We propose a unified framework built on two estimators ML-VRLD and SL-VRLD for analyzing variance-reduction methods, including SGD with momentum, STORM, and PAGE, in non-log-concave sampling.
    \item We establish convergence in $\varepsilon$-relative FI with $\BigO(1)$ gradient computations per iteration, improving upon the $\BigO(Ld\sigma^2/\varepsilon)$ complexity of SGLD. Moreover, SL-VRLD improves the total computational complexity from $\BigO(L^3d^3\sigma^2/\varepsilon^3)$ to $\BigO(L^2d^2\sigma^2/\varepsilon^3)$. We further prove weak convergence to the target distribution and, under a Poincar\'e inequality assumption, convergence in squared TV distance. 
    \item We extend our framework to posterior sampling via annealed LMC sampling with score-based generative models (SGMs) and establish the same convergence guarantees under bounded SGM estimation error.
    \item Through synthetic Gaussian mixture and inverse problem experiments, we verify our theoretical results, and obtain more accurate posterior statistics than SGLD. On MRI and sparse-angle CT reconstruction tasks, we achieve superior sample quality compared to the baselines using the same number of gradient computations per iteration. 
\end{itemize}
\section{Preliminaries}\label{sec:background}

\subsection{Sampling Analogue of Stationary Point Analysis}\label{sec:prelim:fisher-information}

While optimization and sampling are distinct frameworks, their theoretical analysis share deep structural connections. In particular, tools from non-convex optimization theory can be leveraged to characterize the behavior of sampling algorithms~\citep{balasubramanian2022towards}. Consider the minimization of \textit{Kullback-Leibler (KL)} divergence over the space of probability distributions equipped with the optimal transport geometry:
\begin{equation*}
    \argmin_\mu \mathrm{KL}(\mu\|\pi),
    %\!\!\!\quad\!\text{where}\!\!\!\!\quad \mathrm{KL}(\mu\|\pi)\!\!\coloneqq\!\!\!\int_{\sR^d}\!\!\!\!\!\mu(\vx)\log \frac{\mu(\vx)}{\pi(\vx)}\,d\vx.
\end{equation*}where $\mathrm{KL}(\mu\|\pi)\coloneqq\int_{\sR^d}\mu(\vx)\log \frac{\mu(\vx)}{\pi(\vx)}\,d\vx$. Here, $\pi$ denotes the target distribution, and $\mu$ denotes the sampling distribution induced by the generated samples. In this setting, Langevin diffusion can be interpreted as the \textit{exact} gradient flow of the $\mathrm{KL}$ divergence~\citep{jordan1998variational}. For example, the Wasserstein gradient of $\mathrm{KL}(\cdot\|\pi)$ at $\mu$ is $\nabla \log (\mu/\pi)$, and its expected squared norm is known as the \textit{relative Fisher information (FI)}
\begin{equation*}
    \mathrm{FI}(\mu\|\pi)\coloneqq \int_{\sR^d}\mu(\vx)\left\|\nabla \log \frac{\mu(\vx)}{\pi(\vx)}\right\|^2_2\,d\vx.
\end{equation*}If $\mu_t$ denotes the law of the Langevin diffusion~(\ref{eq:langevin-diffusion}) at time $t$, then the calculation rules for gradient flows imply that $\frac{d}{dt} \mathrm{KL}(\mu_t\|\pi) =- \mathrm{FI}(\mu_t\|\pi)$~\citep{villani2009ot,ambrosio2008gradient,santambrogio2015-us}. From a non-convex optimization point of view, $\mathrm{FI}(\mu_t\|\pi)$ is analogous to the $\ell$-2 norm of the gradient in $\sR^d$~\citep{balasubramanian2022towards}. Using this analogy, we analyze the convergence of $\mathrm{FI}(\mu_t\|\pi)$ under a \textit{linear interpolation} of the distributions obtained by the LMC sampling algorithms. Different than non-convex optimization, if $\mu$ and $\pi$ have positive and smooth densities, $\mathrm{FI}(\mu\|\pi) = 0$ implies $\mu=\pi$, which we use to establish weak convergence guarantees. We refer to~\citet{balasubramanian2022towards} for further discussion on this topic.

\subsection{Langevin Diffusion}\label{sec:prelim:langevin-diffusion}

Langevin diffusion is defined as the solution to the stochastic differential equation
\begin{equation}\label{eq:langevin-diffusion}
    d\vx_t = -\nabla f(\vx_t)dt + \sqrt{2}d\mB_t,
\end{equation}and admits $\pi\propto e^{-f}$ as its unique stationary distribution. Under mild conditions, the process converges to $\pi$ as $t \to \infty$. Here, $(\mB_t)_{t\geq 0}$ denotes a standard $d$-dimensional Brownian motion. Discretizing this process with step size $\gamma\!>\!0$ yields the standard Langevin Monte Carlo (LMC) sampling algorithm
\begin{equation}\label{eq:lmc-formulation}
    \vx_{(k+1)\gamma}= \vx_{k\gamma} - \gamma \nabla f(\vx_{k\gamma}) + \sqrt{2\gamma}\mZ_{k}, 
\end{equation}where $\mZ_k\sim \gN(0,I)$. In many practical settings, the exact gradient of the potential $\nabla f(\vx_{k\gamma})$ is unavailable. Thus, it is replaced by its stochastic estimate $\nabla f(\vx_{k\gamma},\xi_k)$, where $\xi_k$ denotes a random variable capturing the external randomness at iteration $k$ and is independent of all other sources of randomness. However, as shown in prior work~\citep{balasubramanian2022towards}, these iterates fail to achieve arbitrarily small $\varepsilon$-relative FI. A standard remedy is to replace the single-sample gradient with a mini-batch estimator, leading to the SGLD iterates
\begin{equation}\label{def:lmc-formulation-batch}
    \vx_{(k+1)\gamma}
    =
    \vx_{k\gamma}
    - \frac{\gamma}{b}\sum_{i=1}^{b}\nabla f(\vx_{k\gamma}, \xi_k^i)
    + \sqrt{2\gamma}\mZ_{k},
\end{equation}
where $b$ denotes the batch size and $(\xi_k^i)_{i=1}^b$ are independent and identically distributed. While this reduces the estimator variance, $\varepsilon$-relative FI is guaranteed by $b\!=\!\BigO(L d \sigma^2 / \varepsilon)$ gradient computations per iteration (see Theorem~\hyperref[thm:sgld-batch-size]{B.1}), leading to substantial per-iteration cost in high-dimensional and high-variance settings. Consequently, SGLD requires very large batch sizes, resulting in both high computational cost and increased memory requirements. In contrast, the proposed ML-VRLD and SL-VRLD estimators achieve $\varepsilon$-relative FI error with $\BigO(1)$ gradient computations per iteration. 

\subsection{Posterior Sampling with SGM Prior}\label{sec:prelim:posterior-sampling}

For the posterior sampling part of our work, we consider a general inverse problem setting modeled as
\begin{equation}\label{def:inverse_problem}
    \vy = \mA(\vx) + \vn, 
\end{equation} where the objective is to recover $\vx\in\sR^d$ from measurements $\vy\in\sR^m$. The measurement operator $\mA: \sR^d \rightarrow \sR^m$ characterizes the imaging system's response, and $\vn\in \sR^m$ denotes the measurement noise, generally modeled as additive Gaussian noise (AWGN). The mapping $\vx\!\rightarrow\!\vy$ is not one-to-one, making the inverse problem ill-posed since $\vx$ cannot be uniquely recovered from $\vy$. In Bayesian framework, uncertainty about $\vx$ is modeled through a \textit{prior} $p(\vx)\propto e^{-h(\vx)}$, and inference is performed by sampling from the \textit{posterior} $\pi(\vx|\vy)\propto p(\vy | \vx) p(\vx)$ obtained via Bayes' rule, where $p(\vy|\vx)\propto e^{-f(\vx)}$ indicates the \textit{likelihood}. Here, $f$ and $h$ denote the potential functions associated with the likelihood and prior, respectively, and are central to our analysis. One standard approach is to sample from the posterior to solve the inverse problem. Using Bayes’ rule and the SGLD update rule in~(\ref{def:lmc-formulation-batch}), we obtain
\begin{align}
    \vx_{(k+1)\gamma}
        & =
        \vx_{k\gamma}
        - \gamma\left(\frac{1}{b}\sum_{i=1}^{b}\nabla f(\vx_{k\gamma}, \xi_k^i) +\nabla h(\vx_{k\gamma})\right)
        \notag \\
        &\quad\quad\quad\quad\quad\quad\quad\quad\quad\quad\quad\quad\quad\quad + \sqrt{2\gamma}\mZ_{k}. \label{def:posterior-lmc}
\end{align}Since $\nabla h(\vx_{k\gamma})$ is generally intractable for complex, high-dimensional priors, score-based generative models (SGMs), parameterized by a neural network, have been proposed to provide computable approximations of this quantity. This approach is commonly referred to as \textit{score matching}~\citep{hyvarinen2005estimation, vincent2011connection}. One elegant approximation replaces the intractable score $-\nabla h(\vx_{k\gamma})$ by smoothed score $-\nabla h_{\sigma}(\vx_{k\gamma})$, which admits a closed form expression by Tweedie's formula~\citep{efron2011tweedie}. In particular, let $\vz\!\sim\! p(\vz)$ be a random sample from the prior, $\ve\!\sim\!\gN(0,\sigma^2 I)$ the AWGN, and $\vx = \vz + \ve$ the noisy image. Then, $\nabla h_\sigma(\vx)\!=\!(\vx-\E[\vz|\vx])/\sigma^2$, where $p_\sigma(\vx)\!\propto\!e^{-h_\sigma(\vx)}$ is the distribution of noisy image with potential function $h_\sigma$, and $p_\sigma(\vx) \coloneqq \int_{\sR^d}p(\vz)\gN(\vx|\vz, \sigma^2 I) \,d\vz$. SGMs approximate the smoothed score by minimizing the denoising score matching (DSM) loss derived from Tweedie’s formula~\citep{vincent2011connection}.

\begin{equation}\label{def:dsm-loss}
    \mathrm{DSM}(\theta) \coloneqq \E\left[\left\|\frac{\vz - \vx}{\sigma^2} - \gS_{\theta}(\vx, \sigma)\right\|^2_2\right],
\end{equation}where $\gS_{\theta}(\vx,\sigma)$ denotes SGM parametrized by $\theta$. Assuming enough data and model capacity, minimization of DSM loss satisfies $\gS_{\theta^*}(\vx, \sigma) = -\nabla h_\sigma(\vx)$ almost everywhere~\citep{song2019generative}. Thus, for sufficiently small $\sigma>0$, we have $\gS_{\theta^*}(\vx_{k\gamma},\sigma)\approx-\nabla h(\vx_{k\gamma})$. 

Standard LMC algorithms exhibit slow convergence and mode collapse when targeting high-dimensional, multimodal distributions. To mitigate these issues, we consider a more general LMC formulation incorporating \textit{weighted annealing}~\citep{sun2024provable}. Specifically, we define a sequence of intermediate target distributions 
\begin{equation*}
    \pi_{\sigma_k}^{(\alpha_k)}(\vx|\vy) \propto p(\vy|\vx)\, p_{\sigma_k}^{\alpha_k}(\vx),
\end{equation*}where $(\alpha_k)_{k=0}^{N-1}$ and $(\sigma_k)_{k=0}^{N-1}$ denote non-increasing annealing and smoothing schedules, respectively. This formulation corresponds to sampling from a sequence of weighted posterior distributions. Early in the sampling process, the smoothed prior $p_{\sigma_k}$ dominates, yielding a smoother energy landscape that facilitates exploration. As the iterations progress, the influence of the likelihood increases while the smoothed prior gradually approaches the true prior $p$, guiding the iterates toward the target posterior. Incorporating this weighted annealing strategy yields the following annealed SGLD iterates with an SGM prior:
\begin{align}
    \vx_{(k+1)\gamma}\!&=\!\vx_{k\gamma}\!-\!\gamma \Big(\frac{1}{b}\sum_{i=1}^b\nabla f(\vx_{k\gamma}, \xi_k^i) \notag \\ 
    & \quad\quad\quad\quad\quad\quad\quad - \alpha_k \gS_{\theta}(\vx_{k\gamma}, \sigma_k)\Big) + \sqrt{2\gamma}\mZ_k. \label{def:annealed-posterior-lmc}
\end{align}

\begin{table}[t]
\caption{Comparison of the total and per-iteration computational complexities of Langevin Monte Carlo sampling algorithms for generating samples with distributions that are $\varepsilon$-close to the target distribution in relative Fisher information.}
\resizebox{\columnwidth}{!}{%
\begin{tabular}{@{}lcc@{}}
\toprule
Algorithm & Total Computation & Per-Iteration Computation \\ \midrule
SGLD      & $\BigO(L^3d^3\sigma^2/\varepsilon^3)$   &           $\BigO(Ld\sigma^2/\varepsilon)$               \\
ML-VRLD   &            $\BigO(L^2d^2\sigma^4/\varepsilon^4)$             &             $\BigO(1)$             \\
SL-VRLD   &             $\BigO(L^2d^2\sigma^2/\varepsilon^3)$            &              $\BigO(1)$            \\ \bottomrule
\end{tabular}%
}
\label{tab:theoretical-results}
\end{table}

\section{Variance Reduction Techniques}\label{sec:methods}

In \citet{balasubramanian2022towards}, SGLD~(\ref{def:lmc-formulation-batch}) achieves $\varepsilon$-relative FI with $\BigO(L d \sigma^2 / \varepsilon)$ gradient computations per iteration. Such a requirement is computationally prohibitive in high-dimensional and high-variance settings. To address this, we propose two unified estimators, ML-VRLD and SL-VRLD, that generalize popular variance reduction techniques from optimization, and establish convergence guarantees with $\BigO(1)$ gradient computations per iteration. Table~\ref{tab:theoretical-results} summarizes the computational complexities. 
\subsection{ML-VRLD}\label{sec:methods:ml-vrld}
Our first estimator unifies SGD with momentum and its stochastic counterpart, both of which exhibit lower variance than the SGD estimator. We make the following standard assumptions on the target distribution $\pi \propto e^{-f}$ and the stochastic gradients $\nabla f(\vx,\xi)$.
\begin{assumption}\label{ass:potential-lipschitz}
    A function $f:\sR^d\rightarrow \sR$ has $L$-Lipschitz continuous gradient: $\|\nabla f(\vx_1) - \nabla f(\vx_2)\|_2\leq L\|\vx_1-\vx_2\|_2$, $\forall \vx_1,\vx_2\in \sR^d$ and $\exists L>0$. 
\end{assumption}
\begin{assumption}\label{ass:sfo}
The stochastic gradient satisfies $\E[\nabla f(\vx,\xi)\mid\vx]=\nabla f(\vx)$ and $\E[\|\nabla f(\vx,\xi)-\nabla f(\vx)\|_2^2\mid\vx]\leq \sigma^2$, $\forall \vx\in\sR^d$ and $\exists\sigma>0$.
\end{assumption}

We define the \textit{mean-Lipschitz variance-reduced Langevin diffusion (ML-VRLD)} unified estimator as a replacement for the mini-batch SGD estimator in~(\ref{def:lmc-formulation-batch}):
\begin{equation}\label{def:ml-vrld:estimator}
    \vg_k \coloneqq
    \begin{cases}
    \displaystyle \frac{1}{b}\sum^{b}_{i=1}\nabla f(\vx_{k\gamma}, \xi_k^i), & \!\!\text{with prob. }p \\
    \beta\vg_{k-1} + (1-\beta)\nabla f(\vx_{k\gamma},\xi_k), & \!\!\text{with prob.  } 1-p
    \end{cases}
\end{equation}where $(\beta,p)\in (0,1]^2 \setminus \{(1,1)\}$, $k\ge 1$ is the iteration index, and $\vg_0 \coloneqq \frac{1}{b}\sum_{i=1}^b\nabla f(\vx_0, \xi_0^i)$ is the initial estimate. For the special case of $p=0$, ML-VRLD reduces to SGD with momentum. When $\beta=1$, we obtain its stochastic counterpart, which is equivalent to SGD with momentum in expectation and $1-p$ becomes the momentum parameter. We combine these estimators to develop a unified analysis of both algorithms, while preserving the convergence guarantees obtained by analyzing each algorithm separately. To intuitively see that ML-VRLD estimator has lower variance than the SGLD estimator~(\ref{def:lmc-formulation-batch}), consider the setting when $p=0$. The estimator is constructed as a weighted average of a stochastic gradient and past gradient estimates. By reusing past gradients, stochastic noise is averaged across iterations and, under Assumption~\ref{ass:potential-lipschitz}, this yields an accurate, lower-variance estimate with fixed batch size. To analyze the resulting dynamics, we consider the following interpolation of ML-VRLD:
\begin{equation}\label{def:ml-vrld:interpolation}
    \vx_t =\vx_{k\gamma} - (t-k\gamma)\vg_k + \sqrt{2}(\mB_t - \mB_{k\gamma})
\end{equation}for $t\!\in\![k\gamma, (k+1)\gamma]$. We are now ready to state our first result. 
\begin{theorem}(Informal)\label{thm:ml-vrld-ficonv}
    Let $\pi\propto e^{-f}$ be the target distribution, where the potential function $f$ satisfies Assumptions~\ref{ass:potential-lipschitz} and~\ref{ass:sfo}. Let $N \ge 1$ denote the total number of iterations, and let $(\mu_t)_{t\geq 0}$ denote the law of continuous-time interpolation~(\ref{def:ml-vrld:interpolation}) with estimator~(\ref{def:ml-vrld:estimator}). Then, for suitable choices of $\gamma, b, p$, and $\beta$, the time-averaged law $\bar\mu_{N\gamma} = \int_0^{N\gamma}\mu_t\,dt$ satisfies $\mathrm{FI}(\Bar{\mu}_{N\gamma}\|\pi) \leq \varepsilon$ with $\BigO (\sigma^4L^2d^2/\varepsilon^4)$ iterations and $\BigO(1)$ gradient computations per iteration.
\end{theorem}

\textbf{Remark 1.} To sample from $\bar\mu_{N\gamma}$ in practice, one first draws $t$ uniformly from $[0,N\gamma]$ and identifies the unique integer $k$ such that $t\in[k\gamma,(k+1)\gamma)$. The sample $\vx_t$ is then obtained by evaluating the interpolation~(\ref{def:ml-vrld:interpolation}) at time $t$.

Theorem~1 (see Appendix~\ref{proof:ml-vrld-ficonv} for the formal statement and proof) establishes the first non-asymptotic convergence guarantee for LMC with stochastic gradients in the non-log-concave setting achieving $\BigO(1)$ gradient computations per iteration. The dependence on $L$ and $d$ in iteration complexity arises from the discretization error of LMC and matches that of LMC with exact gradients, whose convergence rate scales as $\BigO(L^2 d^2 / \varepsilon^2)$~\citep{balasubramanian2022towards}. Leveraging the result of Theorem~\ref{thm:ml-vrld-ficonv}, we show that if the target distribution satisfies Poincar\'e inequality, we obtain stronger convergence guarantees in squared TV distance. 
\begin{assumption}\label{ass:poincare-inequality}
    A distribution $\pi$ satisfies a Poincar\'e inequality if $\exists C_{\mathrm{PI}}>0$ such that, for all compactly supported functions $\phi:\sR^d\to\sR$, $\mathrm{Var}_{\vx \sim \pi} \phi(\vx) \leq C_{\mathrm{PI}}\E_{\vx \sim \pi}[\|\nabla \phi(\vx)\|_2^2]$. 
\end{assumption}
\begin{corollary}(Informal)\label{corr:ml-vrld-tvconv}
    Let $\pi\propto e^{-f}$ be the target distribution, where the potential function $f$ satisfies Assumptions~\ref{ass:potential-lipschitz},~\ref{ass:sfo}, and~\ref{ass:poincare-inequality}. Let $N \ge 1$ denote the total number of iterations, and let $(\mu_t)_{t\geq 0}$ denote the law of continuous-time interpolation~(\ref{def:ml-vrld:interpolation}) with estimator~(\ref{def:ml-vrld:estimator}). Then, for suitable choices of $\gamma, b, p$, and $\beta$, the time-averaged law $\bar\mu_{N\gamma} = \int_0^{N\gamma}\mu_t\,dt$ satisfies $\|\bar\mu_{N\gamma} - \pi\|_\mathrm{TV}^2 \leq \varepsilon$ with $\BigO (C_\mathrm{PI}^4\sigma^4L^2d^2/\varepsilon^4)$ iterations and $\BigO(1)$ gradient computations per iteration.
\end{corollary}
To the best of our knowledge, only two prior works~\citep{balasubramanian2022towards, chewi2024analysis} establish sampling guarantees for smooth potentials satisfying a Poincar\'e inequality, both requiring access to exact gradients. In contrast, our results accommodate stochastic
gradient estimators and hold under $\BigO(1)$ gradient computations per iteration. Moreover, Theorem~\ref{thm:ml-vrld-ficonv} implies weak convergence to the target distribution under decaying $\gamma_k$ and $p_k$, and increasing $\beta_k$ and $b_k$.   
\begin{theorem}(Informal)\label{thm:ml-vrld-weakconv}
    Let $\pi\propto e^{-f}$ be the target distribution, where the potential function $f$ satisfies Assumptions~\ref{ass:potential-lipschitz} and~\ref{ass:sfo}. Let $(\mu_t)_{t\geq  0}$ denote the law of continuous-time interpolation~(\ref{def:ml-vrld:interpolation}) with estimator~(\ref{def:ml-vrld:estimator}), and decreasing step size $\gamma_k$ and probability $p_k$, and increasing momentum $\beta_k$ and batch size $b_k$. Then, the time-averaged law $\bar\mu_{\tau_n}=\frac{1}{\tau_n}\int_0^{\tau_n}\mu_t\,dt$, where $\tau_n=\sum_{k=1}^n \gamma_k$ converges weakly to $\pi$. 
\end{theorem}To our knowledge, existing weak convergence results for LMC either assume Lyapunov-type conditions~\citep{LambertonPages2002, PagesPanloup2012} or having access to exact gradient~\citep{balasubramanian2022towards}. In contrast, our result holds when the gradient is replaced by a variance-reduced stochastic gradient estimator, while maintaining constant per-iteration cost. Our analysis builds on the observation that the coefficient of the discrete error difference forms a nonnegative and decreasing sequence, ensuring stability of the resulting error recursion. Combined with the fact that $\mathrm{FI}(\mu\|\pi)=0$ implies $\mu=\pi$, this establishes weak convergence.

\subsection{SL-VRLD}
To improve the $\sigma^2$-dependence of the iteration complexity, we propose a gradient estimator that unifies PAGE~\citep{li2021page} and STORM~\citep{cutkosky2019momentum} under the following assumption, standard in variance-reduced stochastic optimization.
\begin{assumption}\label{ass:stochastic-potential-lipschitz}
A function $f:\sR^d\to \sR$ has $L$-Lipschitz continuous stochastic gradients in the mean-squared sense: $\E_\xi\left[\|\nabla f(\vx_1,\xi)-\nabla f(\vx_2,\xi)\|_2^2\right]
\leq L^2\|\vx_1-\vx_2\|_2^2$, $\forall \vx_1,\vx_2\in\sR^d$ and $\exists L>0$.
\end{assumption}

We define the \textit{sample-Lipschitz variance-reduced Langevin diffusion (SL-VRLD)} unified estimator as a replacement for the mini-batch SGD estimator in~(\ref{def:lmc-formulation-batch}):
\begin{equation}\label{def:sl-vrld:estimator}
    \vg_k \coloneq
    \begin{cases}
    \displaystyle \frac{1}{b}\sum_{i=1}^b\nabla f(\vx_{k\gamma}, \xi_k^i), & \text{with prob. }p \\
    \nabla f(\vx_{k\gamma}, \xi_k) + \beta (\vg_{k-1} \\
    \quad\quad\quad- \nabla f(\vx_{(k-1)\gamma},\xi_k)), & \text{with prob.  } 1-p
    \end{cases}
\end{equation}where $(\beta,p)\in (0,1]^2 \setminus \{(1,1)\}$, $k\geq 1$ is the iteration index, and $\vg_0\coloneqq \frac{1}{b}\sum_{i=1}^b\nabla f(\vx_0,\xi_0^i)$ is the initial estimate. Setting $\beta=1$ recovers PAGE, while setting $p=0$ recovers STORM. Moreover, when $\beta=1$, taking expectation with respect to $p$ shows that PAGE is equivalent to STORM with momentum parameter $1-p$, and $p<1$ yields a lower per-iteration cost than mini-batch SGD. To highlight the advantage of SL-VRLD over ML-VRLD, consider the case $\beta=1$. With probability $1-p$, ML-VRLD relies solely on previous estimates, whereas SL-VRLD incorporates the correction term $\nabla f(\vx_{k\gamma},\xi_k)-\nabla f(\vx_{(k-1)\gamma},\xi_k)$, which, under Assumption~\ref{ass:stochastic-potential-lipschitz}, improves estimation accuracy and leads to faster convergence with improved FI and squared TV rates. 
\begin{theorem}(Informal)\label{thm:sl-vrld-ficonv}
    Let $\pi\propto e^{-f}$ be the target distribution, where the potential function $f$ satisfies Assumptions~\ref{ass:sfo} and~\ref{ass:stochastic-potential-lipschitz}. Let $N \ge 1$ denote the total number of iterations, and let $(\mu_t)_{t\geq 0}$ denote the law of continuous-time interpolation~(\ref{def:ml-vrld:interpolation}) with estimator~(\ref{def:sl-vrld:estimator}). Then, for suitable choices of $\gamma, b, p$, and $\beta$, the time-averaged law $\bar\mu_{N\gamma} = \int_0^{N\gamma}\mu_t\,dt$ satisfies $\mathrm{FI}(\Bar{\mu}_{N\gamma}\|\pi) \leq \varepsilon$ with $\BigO (\sigma^2L^2d^2/\varepsilon^3)$ iterations and $\BigO(1)$ gradient computations per iteration.
\end{theorem}
%\begin{corollary}(Informal)\label{corr:sl-vrld-tvconv}
%    Let $(\mu_t)_{t\geq 0}$ denote the law of interpolation~(\ref{def:ml-vrld:interpolation}) with estimator~(\ref{def:sl-vrld:estimator}), and let $\pi\propto\mathrm{exp}(-f)$ be the target distribution, where the potential function $f$ satisfies Assumptions~\ref{ass:potential-lipschitz},~\ref{ass:sfo},~\ref{ass:poincare-inequality}, and~\ref{ass:stochastic-potential-lipschitz}. For a sufficiently small step size $\gamma$, the distribution $\Bar{\mu}_{N\gamma}$ of an iterate chosen uniformly at random from the algorithmic trajectory satisfies $\|\Bar{\mu}_{N\gamma} - \pi\|_{\mathrm{TV}}^2 \leq \varepsilon$ after at least $\BigO (C_\mathrm{PI}^3\sigma^2L^2d^2/\varepsilon^3)$ stochastic gradient evaluations, with $\BigO (1)$ computational cost per iteration. 
%\end{corollary}
Theorem~\ref{thm:sl-vrld-ficonv} (with a formal statement and proof provided in Appendix~\ref{proof:sl-vrld-ficonv}) establishes an improved
FI convergence rate over that of ML-VRLD, reducing the iteration complexity from
$\BigO(\sigma^4 L^2 d^2 / \varepsilon^4)$ to
$\BigO(\sigma^2 L^2 d^2 / \varepsilon^3)$. A similar improvement holds for the squared TV distance, as stated in Corollary~\ref{proofs:theorem-and-lemmas}.1 in Appendix~\ref{proof:sl-vrld-tvconv}, with the complexity reduced from
$\BigO(C_{\mathrm{PI}}^4 \sigma^4 L^2 d^2 / \varepsilon^4)$ to
$\BigO(C_{\mathrm{PI}}^3 \sigma^2 L^2 d^2 / \varepsilon^3)$. In addition, our analysis of FI convergence implies weak convergence of SL-VRLD to the target distribution $\pi$; see Theorem~\ref{proofs:theorem-and-lemmas}.1 in Appendix~\ref{proof:sl-vrld-weakconv}.

%\begin{theorem}(Informal)\label{thm:sl-vrld-weakconv}
%    Let $(\mu_t)_{t\geq  0}$ denote the law of interpolation~(\ref{def:ml-vrld:interpolation}) with estimator~(\ref{def:sl-vrld:estimator}) generated with decreasing time-varying step size $\gamma_k$, probability $p_k$, momentum parameter $\beta_k$, and increasing batch size $b_k$. Let $\pi\propto \mathrm{exp}(-f)$ be the target distribution whose potential function $f$ satisfies Assumptions~\ref{ass:potential-lipschitz},~\ref{ass:sfo}, and~\ref{ass:stochastic-potential-lipschitz}. Then, the distribution $\Bar{\mu}_{\tau_n}$, where $\tau_n\coloneqq\sum_{k=0}^n \gamma_k$, of an iterate chosen uniformly at random from the algorithmic trajectory converges weakly to $\pi$ as $n\rightarrow\infty$. 
%\end{theorem}

\section{Variance Reduction with SGM Prior}\label{sec:methods-2}
Existing posterior sampling algorithms with SGM priors operate in full-batch regime by using the entire measurement set at each iteration. This limits their applicability to large datasets~\citep{bottou2007tradeoffs,ong2020extreme}. While mini-batch strategies may appear computationally attractive, accurate sampling typically necessitates large batch sizes. To address this limitation, we develop two principled variance-reduced posterior sampling algorithms by extending our theory in Section~\ref{sec:methods}, and instantiating our unified ML-VRLD and SL-VRLD estimators within the SGM-prior framework. In particular, we replace the mini-batch SGD estimator in~(\ref{def:annealed-posterior-lmc}) by our proposed unified estimators and define the resulting interpolation for posterior sampling
\begin{align}
    \vx_{t} \coloneqq \vx_{k\gamma} - (t-k\gamma)\Big(\vg_k - \alpha_k \gS_\theta(\vx_{k\gamma}, \sigma_k)\Big)\notag \\
    + \sqrt{2}(\mB_t - \mB_{k\gamma}), \label{def:SGM-prior-interpolation}
\end{align}
where $\vg_k$ is instantiated as either ML-VRLD~(\ref{def:ml-vrld:estimator}) or SL-VRLD~(\ref{def:sl-vrld:estimator}). We provide theoretical guarantees for both choices with the annealing schedule ${\alpha_k}$ and smoothing schedule ${\sigma_k}$ defined as follows
\begin{equation}\label{def:annealing-smoothing-schedules}
    \alpha_k \coloneqq \max\{\alpha_0\rho_1^k, 1\} \quad \text{and} \quad \sigma_k \coloneqq \max\{\sigma_0 \rho_2^k, \sigma_{\mathrm{min}}\}
\end{equation}for $k\geq 0$, where $\rho_1, \rho_2\in(0,1)$ denote decay rates, $\sigma_0\ge \sigma_{\mathrm{min}}$ and $\alpha_0\ge 1$ are initial values, and $\sigma_{\mathrm{min}}>0$ is the minimum noise level. These parameters are selected using the principled techniques of~\citep{song2020improved} such that there exists indices $K_\alpha$, $K_\sigma< N-1$ satisfying $\alpha_k=1$, $\forall k\ge K_\alpha$ and $\sigma_k = \sigma_{\mathrm{min}}$, $\forall k \ge K_\sigma$~\citep{sun2024provable, sahin2026zero}. Our analysis builds on these schedules together with the interpolation~(\ref{def:SGM-prior-interpolation}). Moreover, since the SGM estimates the true prior score, i.e., $\gS_\theta(\vx,\sigma_{\mathrm{min}}) \approx -\nabla h(\vx)$, it is necessary to quantify the estimation error. To this end, we adopt the following assumptions from prior work~\citep{lee2022convergence, lee2023convergence, sun2024provable, sahin2026zero}. 
\begin{assumption}\label{ass:smoothing-err-bounded}
Let $p_{\sigma_k}(\vx)\!\!\coloneqq\!\!\int p(\vz)\gN(\vx|\vz, \sigma_k^2 I)\,d\vz$, where $h(\vx)$ and $h_{\sigma_k}(\vx)$ denote the potential function of $p(\vx)$ and $p_{\sigma_k}(\vx)$, respectively. Then, $\exists C_1>0$ such that $\|\nabla h_{\sigma_k}(\vx)-\nabla h(\vx)\|_2\le C_1\sigma_k, \forall\vx\in\sR^d$ and $\forall\sigma_k>0$. 
\end{assumption}
\begin{assumption}\label{ass:score-network}
    The score network satisfies $\|\gS_\theta(\vx, \sigma_k) + \nabla h_{\sigma_k}(\vx)\|_2 \leq \varepsilon_{\sigma_k}$ and  $\|\gS_\theta(\vx, \sigma_k)\|_2\leq C_2\sigma_k^{-1}$ for $\forall\vx\in\sR^d$, $\forall \sigma_k>0$, and $\exists C_2>0$.
\end{assumption}
Assumption~\ref{ass:smoothing-err-bounded} bounds the prior smoothing error. When $p(\vx)$ is Gaussian, this error admits a closed-form expression but does not hold for general distributions. Assumption~\ref{ass:score-network} bounds the SGM estimation error, thereby relaxing the idealized population setting in which, given sufficient model capacity and data, minimizing the DSM loss~(\ref{def:dsm-loss}) recovers the exact score, i.e., $\gS_{\theta^*}(\vx,\sigma_k)=-\nabla h_{\sigma_k}(\vx)$ with probability $1$. The SGM norm bound is supported by empirical observations in~\citet{song2019generative,song2020improved}. Unlike prior work~\citep{sun2024provable}, we do not assume Lipschitz continuity of the SGM. Further discussion of these assumptions is provided in Appendix~\ref{appendix:assumptions-discussion}. Under these assumptions, we establish the following convergence result.

\begin{theorem}(Informal)\label{thm:ficonv-SGM-prior}
Let $\pi(\vx|\vy) \propto p(\vy|\vx)p(\vx)$ be the target posterior, where
$p(\vy|\vx)\propto e^{-f(\vx)}$ and $p(\vx)\propto e^{-h(\vx)}$.
Assume the likelihood potential $f$ satisfies Assumption~\ref{ass:sfo},
the prior potential $h$ satisfies Assumptions~\ref{ass:potential-lipschitz} and~\ref{ass:smoothing-err-bounded},
and the SGM satisfies Assumption~\ref{ass:score-network} with error
$\varepsilon_{\sigma_k}=\mathcal{O}(k^{-1/2})$. Define $L_\pi = L_f + L_h$. Let $N\ge 1$ denote the total number of iterations, and let $(\mu_t)_{t\geq 0}$ denote the law of continuous-time interpolation~(\ref{def:SGM-prior-interpolation})
with schedules~(\ref{def:annealing-smoothing-schedules}). Then, $\mathrm{FI}(\Bar{\mu}_{N\gamma}\|\pi)\le \varepsilon$ is achieved with $\BigO(1)$ gradient computations per iteration, and with the following iteration complexities:
\begin{enumerate}[label=(\alph*)]
\item \textbf{(ML-VRLD)} $\BigO\!\left(\sigma^4 L_\pi^2 d^2/\varepsilon^4\right)$ if $\vg_k$ is defined as in~(\ref{def:ml-vrld:estimator}) and $f$ satisfies Assumption~\ref{ass:potential-lipschitz}.
\item \textbf{(SL-VRLD)} $\BigO\!\left(\sigma^2 L_\pi^2 d^2/\varepsilon^3\right)$ if $\vg_k$ is defined as in~(\ref{def:sl-vrld:estimator}) and
$f$ satisfies Assumption~\ref{ass:stochastic-potential-lipschitz}. 
\end{enumerate}
\end{theorem}

\begin{figure}[t]
    \centering
    \includegraphics[width=0.9\linewidth]{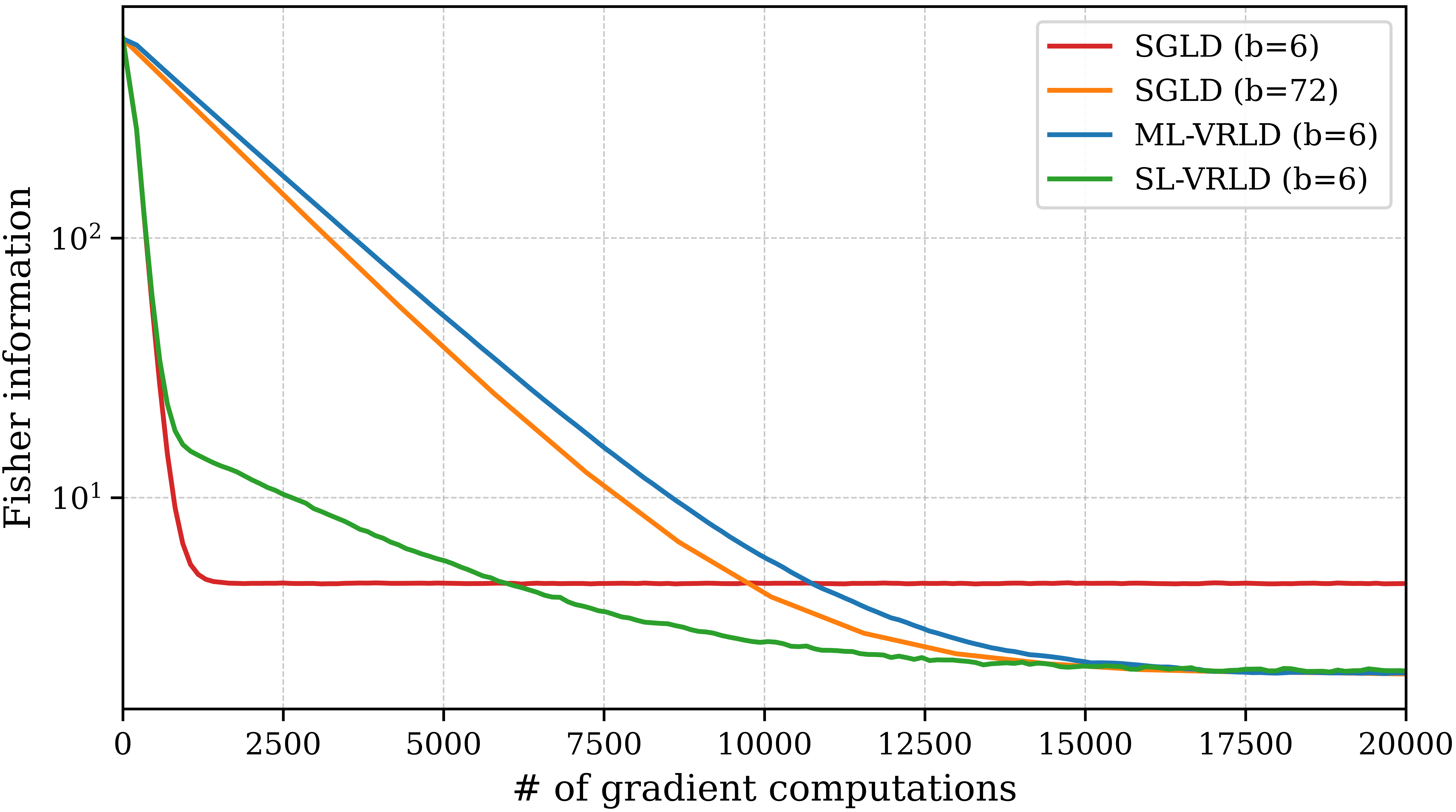}
    \caption{Results on a synthetic 2D Gaussian mixture model (GMM) sampling problem showing the relative FI convergence of each method, where $b$ denotes the batch size. For a fixed batch size, SGLD converges to a suboptimal FI level, whereas ML-VRLD and SL-VRLD achieve substantially lower FI values, consistent with our theoretical results.}
    %\caption{Results on a synthetic 2D inverse problem showing FI convergence for each method with batch size $b\!=\!6$. SGLD converges to a suboptimal solution, whereas ML-VRLD and SL-VRLD attain lower FI, consistent with our theory.}
    \label{fig:fi-convergence-gmm}
    %\label{fig:fi-convergence-synthetic-problem}
\end{figure}

Formal statements and proofs are provided in Appendices~\ref{proof:ml-vrld:ficonv-SGM-prior} and~\ref{proof:sl-vrld:ficonv-SGM-prior}. To the best of our knowledge, this result provides the first non-asymptotic $\varepsilon$-accurate convergence guarantee to the target posterior for inverse problems using $\BigO(1)$ stochastic gradient computations per iteration. Prior work~\citep{sun2019online, tang2020fast, liu2022online} do not establish such guarantees. Our results further imply weak convergence to the target posterior. 
\begin{theorem}(Informal)
\label{thm:weakconv-SGM-prior}
Let $\pi(\vx|\vy) \propto p(\vy|\vx)p(\vx)$ be the target posterior, where
$p(\vy|\vx)\propto e^{-f(\vx)}$ and $p(\vx)\propto e^{-h(\vx)}$.
Assume the likelihood potential $f$ satisfies Assumption~\ref{ass:sfo},
the prior potential $h$ satisfies Assumptions~\ref{ass:potential-lipschitz} and~\ref{ass:smoothing-err-bounded},
and the SGM satisfies Assumption~\ref{ass:score-network} with error
$\varepsilon_{\sigma_k}=\mathcal{O}(k^{-1/2})$. Let $(\mu_t)_{t\ge 0}$ denote the law of continuous-time interpolation~(\ref{def:SGM-prior-interpolation})
with schedules~(\ref{def:annealing-smoothing-schedules}), and decreasing step size $\gamma_k$ and probability $p_k$, and increasing momentum $\beta_k$ and batch size $b_k$. Then, the time-averaged law  $\Bar{\mu}_{\tau_n}=\frac{1}{\tau_n}\int_0^{\tau_n}\mu_t\,dt$, where $\tau_n \coloneqq \sum_{k=1}^n \gamma_k$, converges weakly to $\pi$, provided that one of the following holds:

\begin{enumerate}[label=(\alph*)]
\item \textbf{(ML-VRLD)} $\vg_k$ is defined as in~(\ref{def:ml-vrld:estimator}) and $f$ satisfies Assumption~\ref{ass:potential-lipschitz}.
\item \textbf{(SL-VRLD)} $\vg_k$ is defined as in~(\ref{def:sl-vrld:estimator}) and $f$ satisfies
Assumption~\ref{ass:stochastic-potential-lipschitz}.
\end{enumerate}
\end{theorem}
\section{Experiments}

\begin{figure}[t]
    \centering
    \includegraphics[width=0.9\linewidth]{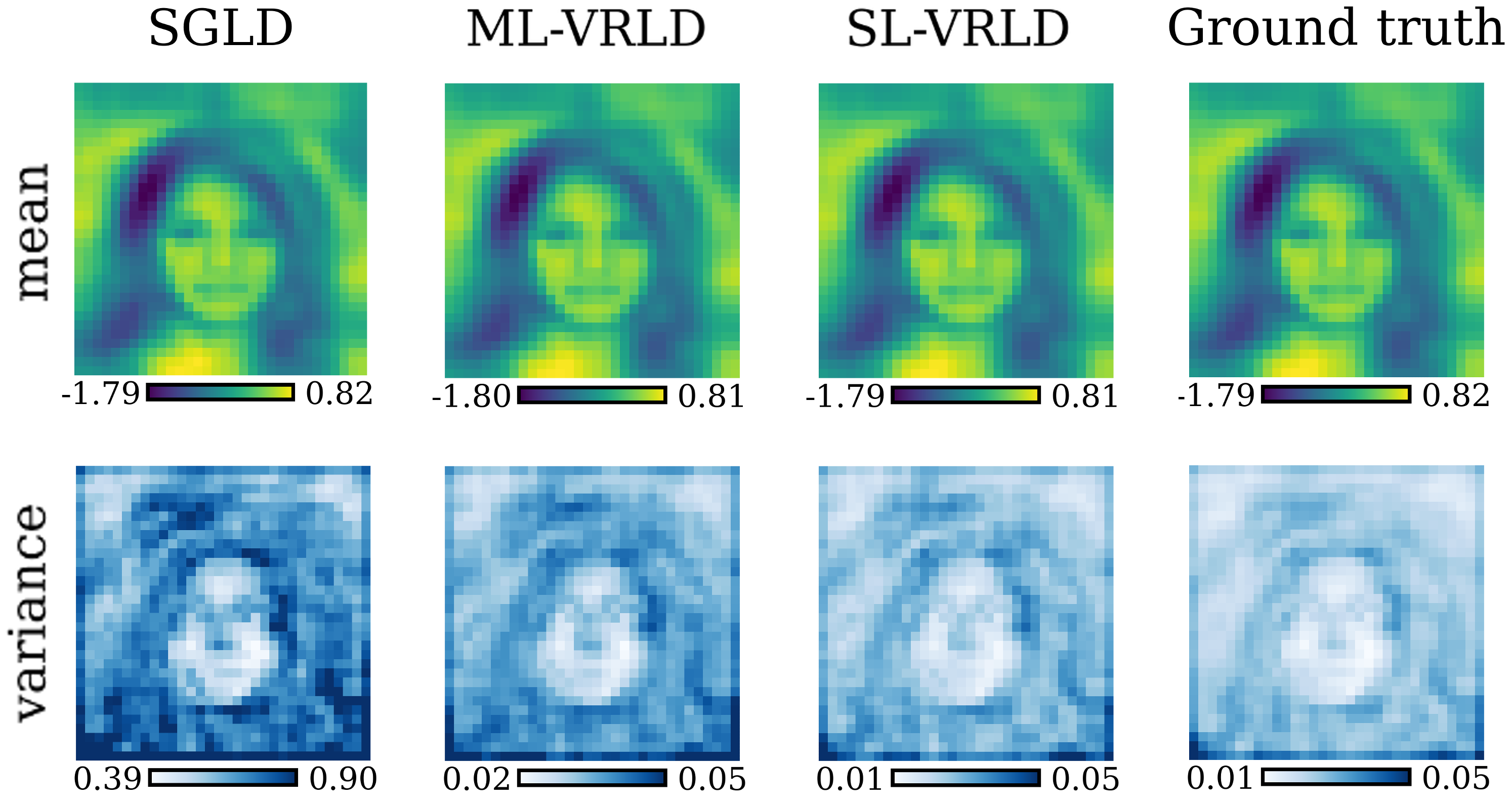}
    \caption{Results on a synthetic inverse problem with known posterior mean and variance. ML-VRLD and SL-VRLD provide more accurate variance estimates than SGLD, which generates samples with significantly higher variance. }
    \label{fig:statistical-validation}
\end{figure}

\subsection{Synthetic Experiments}

\textbf{Numerical Validation.} We validate the theoretical results for ML-VRLD and SL-VRLD summarized in Table~\ref{tab:theoretical-results}. To this end, we consider two synthetic problems: sampling from a Gaussian mixture model (GMM) and a compressed sensing inverse problem.

For the GMM experiments, we construct an 8-mode target distribution and simulate stochastic gradients by adding Gaussian noise with std. dev. of 30 to the gradient of the potential function. We compare SGLD, ML-VRLD with $\beta=0.9$, and SL-VRLD with $\beta=0.999$ over $N=2000$ iterations using 20 different initializations with $p=0$. Fig.~\ref{fig:fi-convergence-gmm} reports the mean relative FI as a function of the total number of gradient computations. With a fixed batch size of $b=6$, SGLD converges to a suboptimal relative FI, suggesting that a larger batch size is needed to achieve lower error. Increasing the batch size to $b=72$ allows SGLD to achieve similar relative FI as ML-VRLD and SL-VRLD, albeit at the cost of substantially more gradient computations per iteration and memory usage. In contrast, SL-VRLD reaches the same error faster than both SGLD and ML-VRLD. These observations are consistent with the theoretical results in Table~\ref{tab:theoretical-results}. Visualizations of the generated samples and additional experimental details are provided in Appendix~\ref{appendix:synthetic-experiments}.

\begin{figure*}[t]
    \centering
    \includegraphics[width=0.95\linewidth]{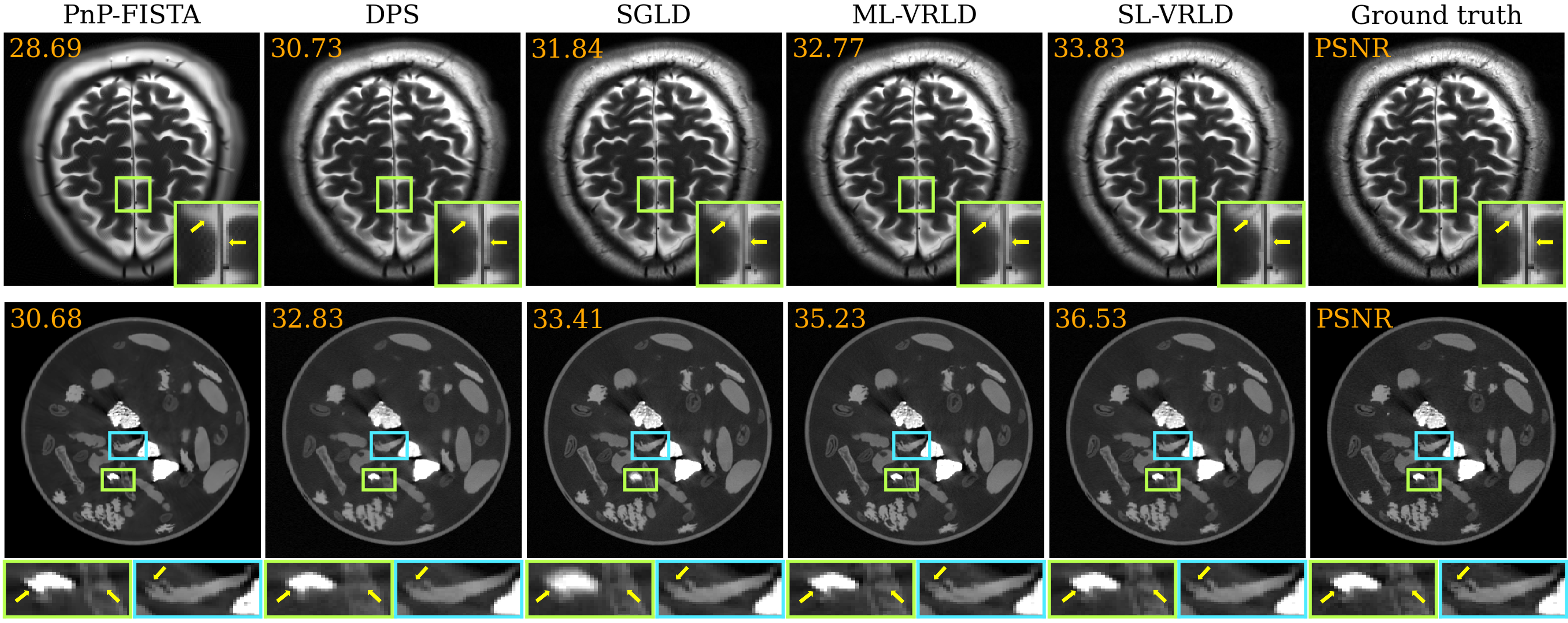}
    \caption{Visual comparison of ML-VRLD, SL-VRLD, and baseline methods for MRI and sparse-angle CT reconstruction. The first row shows MRI results with an acceleration factor of 8 under radial subsampling, using batch size of 32 coils out of 96 coils at each iteration. The second row shows sparse-angle CT reconstructions with a batch size of 30 out of 180 projection views. PSNR values are displayed in the top-left corner of each image. Zoomed-in regions with yellow arrows highlight reconstruction differences. In both tasks, ML-VRLD and SL-VRLD produce sharper structures and fewer artifacts than the baseline methods. }
    \label{fig:visual-comparison}
\end{figure*}

For the compressed sensing experiment, we consider a 2D posterior distribution defined by the Gaussian likelihood in~(\ref{def:inverse_problem}) and a bimodal Gaussian mixture prior. We simulate stochastic likelihood gradients by adding Gaussian noise with std. dev. of 10 to the likelihood score. We compare SGLD, ML-VRLD, and SL-VRLD with fixed batch size $b\!=\!6$ over 20 test posteriors generated using different realizations of $\mA$, and report the mean relative FI in Fig.~\ref{fig:fi-convergence} (Appendix~\ref{appendix:synthetic-experiments}). Consistent with our theory, SL-VRLD converges faster than ML-VRLD, while both achieve lower relative FI than SGLD, which converges to a suboptimal level.

\textbf{Statistical Validation. } We validate ML-VRLD and SL-VRLD for sampling from a posterior distribution over images. We construct a unimodal Gaussian prior obtained from female images in CelebA~\citep{liu2018large}. Each image is rescaled to $32\!\times\!32$ and normalized to $[-1,1]$. The forward model $\mA\in \sR^{115 \times 1024}$ is a Gaussian matrix, and the measurements $\vy$ are obtained according to~(\ref{def:inverse_problem}), with AWGN of std. dev. 0.1. We run SGLD, ML-VRLD, and SR-VRLD with batch size $b=6$ to generate 1000 samples and compute the sample mean and per-pixel variance of the posterior for each method. We compare these sample statistics to the exact posterior statistics, which are available in closed form, as shown in Fig.~\ref{fig:statistical-validation}. SGLD exhibits substantially larger sample variance than ML-VRLD, SL-VRLD, and the ground truth statistics. Both ML-VRLD and SL-VRLD significantly reduce this variance, with SL-VRLD achieving the closest match to the ground truth variance. These results support our theoretical findings. Under a fixed batch size, SGLD fails to recover the correct posterior variance and therefore does not converge to the target posterior, whereas ML-VRLD and SL-VRLD achieve substantially lower estimation error in posterior statistics. Additional results and experimental details are provided in Appendix~\ref{appendix:synthetic-experiments}.

\subsection{Accelerated Parallel MRI Reconstruction}

We simulate a multi-coil MRI acquisition with a radial subsampling pattern~\citep{lustig2008compressed}. The measurement operator $\mA$ is constructed from $K$ receiver coil sensitivity maps $\{\mS_i\}_{i=1}^K$~\citep{pruessmann1999sense}, where the measurement operator for the $i$-th coil is given by $\mA_i = \mP \mF \mS_i$. Here, $\mF$ denotes the Fourier transform and $\mP$ is a diagonal matrix with radial subsampling mask. We consider two acceleration factors $R\!\in\!\{8, 12\}$, representing highly undersampled regimes used in fast and volumetric MRI acquisitions~\citep{zhang2020acceleration}. To simulate a multi-coil MRI system~\citep{lustig2008compressed}, we synthesize $K=96$ coil sensitivity maps using $\texttt{SigPy}$~\citep{ong2019sigpy}.

\begin{table*}[!h]
\centering
\renewcommand\arraystretch{1.1}
\setlength{\tabcolsep}{3pt}

% left block: Method + (8x,12x)
% separator
% right block: (8x,12x) copy for now
\begin{tabular}[t]{lccc ccc @{\hspace{16pt}} ccc ccc}
\toprule
& \multicolumn{6}{c}{\textbf{MRI}}
& \multicolumn{6}{c}{\textbf{CT}} \\

\cmidrule(r{20pt}){2-7}
\cmidrule(r){8-13}

& \multicolumn{3}{c}{Accel. = $8\times$} 
& \multicolumn{3}{c}{Accel. = $12\times$}
& \multicolumn{3}{c}{$360$ views} 
& \multicolumn{3}{c}{$180$ views} \\

\cmidrule(r){2-4}
\cmidrule(r{20pt}){5-7}
\cmidrule(r){8-10}
\cmidrule(r){11-13}

\textbf{Method}
& PSNR & SSIM & MSE
& PSNR & SSIM & MSE
& PSNR & SSIM & MSE
& PSNR & SSIM & MSE \\
\midrule

Zero-filled/FBP
& 20.01 & 0.709 & 10.5e-3
& 17.67 & 0.664 & 17.7e-3
& 25.05 & 0.727 & 31.25e-4
& 22.31 & 0.598 & 58.94e-4 \\

TV
& 25.12 & 0.916 & 32.4e-4
& 22.12 & 0.861 & 65.5e-4
& 28.58 & 0.938 & 15.33e-4
& 28.25 & 0.932 & 16.35e-4 \\

RED
& 30.36 & 0.915 & 9.3e-4
& 29.09 & 0.908 & 12.5e-4
& 29.07 & 0.944 & 13.56e-4
& 28.93 & 0.943 & 13.96e-4 \\

PnP-FISTA
& 30.56 & 0.942 & 9.0e-4
& 29.04 & 0.935 & 12.8e-4
& 29.97 & 0.954 & 10.81e-4
& 29.39 & 0.951 & 12.29e-4 \\

DPS
& 33.22 & 0.968 & 5.0e-4
& 31.27 & 0.959 & 7.8e-4
& 31.8 & 0.961 & 7.17e-4
& 31.64 & 0.961 & 7.43e-4 \\

PnP-DM
& 33.46 & 0.969 & 4.6e-4
& \secondbest{31.53} & 0.961 & \secondbest{7.3e-4}
& 33.59 & 0.971 & 4.53e-4
& 31.66 & 0.959 & 7.02e-4 \\

SGLD
& 33.84 & 0.975 & 4.4e-4
& 30.81 & 0.966 & 9.1e-4
& 33.12 & 0.976 & 5.36e-4
& 32.6 & 0.974 & 6.06e-4 \\

\midrule
ML-VRLD
& \secondbest{34.63} & \secondbest{0.977} & \secondbest{3.66e-4}
& 31.42 & \secondbest{0.968} & 7.99e-4
& \secondbest{35.02} & \secondbest{0.98} & \secondbest{3.41e-4}
& \secondbest{34.35} & \secondbest{0.978} & \secondbest{4.06e-4} \\

SL-VRLD
& \best{35.34} & \best{0.978} & \best{3.02e-4}
& \best{33.11} & \best{0.971} & \best{5.06e-4}
& \best{35.52} & \best{0.981} & \best{2.97e-4}
& \best{34.54} & \best{0.979} & \best{3.72e-4} \\
\bottomrule
\end{tabular}

\caption{Quantitative comparison of ML-VRLD and SL-VRLD against baseline methods for MRI reconstruction on the fastMRI brain dataset~\citep{zbontar1811fastmri} and sparse-angle CT reconstruction on the 2DeteCT dataset~\citep{kiss20232detect}. The \best{best} and \secondbest{second-best} results are highlighted. For MRI reconstruction, acceleration factors of 8 and 12 are used with a batch size of 32 coils (out of 96 total coils) across all methods. For sparse-angle CT reconstruction, batch sizes of 60 and 30 (out of 360 and 180 projection views, respectively) are used for all methods. See Table~\ref{tab:mem-results} (Appendix~\ref{appendix:additional-experiments}) for memory usage.}
\label{tab:quantitative-comparison}
\end{table*}

We compare SGLD, ML-VRLD, and SL-VRLD, implemented as annealed LMC sampling with likelihood score estimates defined as in~(\ref{def:annealed-posterior-lmc}), ~(\ref{def:ml-vrld:estimator}), and~(\ref{def:sl-vrld:estimator}), respectively. The prior score is approximated using the pretrained SGM from~\citep{sun2024provable}. As diffusion-based baselines, we evaluate DPS~\citep{chung2022diffusion} and PnP-DM~\citep{wu2024principled}, which adopt different sampling strategies for solving inverse problems, and additionally include denoiser-based methods, namely PnP-FISTA~\citep{kamilov2018plug, sun2019online} and RED~\citep{romano2017little}. We use U-Net architecture~\citep{dhariwal2021diffusion} as a denoiser. We further include classical baselines, including TV regularization~\citep{rudin1992nonlinear} and zero-filled reconstruction, which is obtained by setting unsampled points to zero followed by inverse Fourier transform. We use brain MRI data from the fastMRI dataset~\citep{zbontar1811fastmri, knoll2020fastmri} to train the denoiser and diffusion models used in the baselines, and to construct the test set. The training and test sets consist of 70,684 and 100 slices, respectively, where each slice has a shape of $256\!\times\!256$. For the annealed LMC methods, we adopt the annealing schedule defined in~(\ref{def:annealing-smoothing-schedules}) and set $p\!=\!0$ for both ML-VRLD and SL-VRLD. The joint use of $\beta<1$ and $p>0$ is introduced for theoretical unification. In practice, one may instead set either $p=0$ or $\beta=1$. We tune the step size $\gamma$ and momentum parameter $\beta$ for ML-VRLD and SL-VRLD using \texttt{Optuna}~\citep{akiba2019optuna}. All methods use a fixed batch size of $b=32$ randomly selected coils at each iteration. Additional experimental details are provided in Appendix~\ref{appendix:mri-reconstruction}.

The first row in Fig.~\ref{fig:visual-comparison} shows representative reconstructions at an acceleration factor of $R=8$. SL-VRLD and ML-VRLD more effectively preserve fine anatomical details, particularly in the zoomed-in regions, whereas the baseline methods tend to oversmooth or introduce artifacts. Table~\ref{tab:quantitative-comparison} reports quantitative comparison of all methods for two different acceleration factors. SL-VRLD consistently achieves the best performance across all methods, with ML-VRLD yielding the second-best results. These findings demonstrate the advantage of incorporating variance reduction into annealed LMC posterior sampling.

\subsection{Sparse-Angle CT Reconstruction}

We consider a sparse-angle CT reconstruction task, which aims to recover high-quality images from a limited number of projection views, resulting in an ill-posed inverse problem due to incomplete angular sampling. We simulate parallel-beam CT measurements with 360 and 180 projection views uniformly distributed over a half circle using the implementation of the measurement operator $\mA$ in \texttt{DeepInverse}~\citep{tachella2025deepinverse}. We adopt the same baselines as in the MRI experiments, except that zero-filled reconstruction is replaced by filtered back-projection (FBP), which is more appropriate for CT reconstruction. For all methods, we use the same batch sizes of 60 and 30 out of 360 and 180 projection views, respectively, to estimate the likelihood score. We use abdominal CT scans of size $256\times 256$ from the 2DeteCT dataset~\citep{kiss20232detect} to construct the training and test sets. The training set comprises 4,500 slices and is used to train the denoiser, diffusion model, and SGM for their corresponding methods described in the previous section, while the test set consists of 100 slices and is used for evaluation. Additional details are provided in Appendix~\ref{appendix:ct-reconstruction}. 

The second row in Fig.~\ref{fig:visual-comparison} shows a visual comparison of reconstructions using a batch size of 30 out of 180 projection views. As highlighted in the zoomed-in regions, SGLD and the baseline methods exhibit a loss of structural sharpness, whereas ML-VRLD and SL-VRLD better preserve fine anatomical details. We provide quantitative comparison in Table~\ref{tab:quantitative-comparison}. ML-VRLD and SL-VRLD consistently outperforms the baseline methods at two different projection views, demonstrating the benefit of variance reduction. 

\section{Conclusion}
We proposed ML-VRLD and SL-VRLD, two unified estimators for non-log-concave Langevin sampling under different regularity conditions. These estimators encompass variance reduction techniques from non-convex optimization, including SGD with momentum, STORM, and PAGE. We showed that both methods achieve $\varepsilon$-relative Fisher information convergence with $\BigO(1)$ gradient computations per iteration, enabling convergence under fixed memory requirements. In contrast, SGLD requires a batch size scaling as $\BigO(Ld\sigma^2/\varepsilon)$, making it increasingly expensive in high-dimensional and high-variance settings. Moreover, SL-VRLD reduces the total computational complexity compared to ML-VRLD and SGLD. We further established weak convergence to the target distribution and, under a Poincar\'e inequality assumption, convergence in squared total variation distance. Motivated by the limitations of existing full-batch posterior sampling methods in large-scale inverse problems, we extended our framework to score-based generative priors and developed principled posterior sampling algorithms with the same convergence guarantees. Experiments on synthetic problems validated our theoretical findings, while MRI and sparse-angle CT reconstruction experiments demonstrated improved sample quality using the proposed estimators.

%\begin{contributions} % will be removed in pdf for initial submission 
					  % (without ‘accepted’ option in \documentclass)
                      % so you can already fill it to test with the
                      % ‘accepted’ class option
%    Briefly list author contributions. 
%    This is a nice way of making clear who did what and to give proper credit.
%    This section is optional.

%    H.~Q.~Bovik conceived the idea and wrote the paper.
%    Coauthor One created the code.
%    Coauthor Two created the Fig.s.
%\end{contributions}

\begin{acknowledgements} % will be removed in pdf for initial 
This work was partially supported by National Insititutes of Health grant NIH/NHLBI R01-HL153430 (PI: Sharif), and supported in part by NSF DMS-2502560 and CNS-2313109. 
\end{acknowledgements}

%\begin{table}[h]\small
%\centering
%\renewcommand\arraystretch{1.1}
%\setlength{\tabcolsep}{3pt}
%\begin{tabular}{lcccccc}
%\toprule
%Acceleration
% & \multicolumn{3}{c}{$8\times$} 
% & \multicolumn{3}{c}{$12\times$} \\ \hline
%Metrics 
% & PSNR & SSIM & MSE 
% & PSNR & SSIM & MSE \\ \hline

%Zero-filled 
% & 20.01 & 0.709 & 10.5e-3
% & 17.67 & 0.664 & 17.7e-3 \\

%TV 
% & 25.12 & 0.916 & 32.4e-4 
% & 22.12 & 0.861 & 65.5e-4 \\

%RED
% & 30.36 & 0.915 & 9.3e-4 
% & 29.09 & 0.908 & 12.5e-4 \\

%PnP-SGD
% & 30.56 & 0.942 & 9.0e-4 
% & 29.04 & 0.935 & 12.8e-4 \\

%DPS
% & 33.22 & 0.968 & 5.0e-4 
% & 31.27 & 0.959 & 7.8e-4 \\

%PnP-DM
% & 33.46 & 0.969 & 4.6e-4
% & \secondbest{31.53} & 0.961 & \secondbest{7.3e-4} \\

%SGLD 
% & 33.84 & 0.975 & 4.4e-4 
% & 30.81 & 0.966 & 9.1e-4 \\

%\midrule
%ML-VRLD
% & \secondbest{34.63} & \secondbest{0.977} & \secondbest{3.66e-4} 
% & 31.42 & \secondbest{0.968} & 7.99e-4 \\

%SL-VRLD
% & \best{35.34} & \best{0.978} & \best{3.02e-4} 
% & \best{33.11} & \best{0.971} & \best{5.06e-4} \\

%\bottomrule
%\end{tabular}

%\caption{Quantitative comparison of ML-VRLD and SL-VRLD against baseline methods for MRI reconstruction on the fastMRI brain dataset at acceleration factors of $8$ and $12$. The \best{best} and \secondbest{second best} results are highlighted. ML-VRLD and SL-VRLD outperforms baseline methods based on denoisers and SGM priors.}
%\end{table}

% References
\bibliography{references}

\onecolumn

\title{Variance Reduction for Non-Log-Concave Sampling \\with Applications to Inverse Problems\\(Supplementary Material)}
\maketitle

\appendix

\section*{Appendix}
\addcontentsline{toc}{section}{Appendix} % Add to main TOC if needed
\markboth{Appendix}{Appendix} % For headers
 
% Create a mini-TOC for the appendix
\startcontents[appendix]
\printcontents[appendix]{l}{1}{\setcounter{tocdepth}{3}}

\section{DISCUSSION ON ASSUMPTIONS}\label{appendix:assumptions-discussion}

In this section, we explicitly list and discuss all the assumptions used for our theoretical results in Section~\ref{sec:methods} and~\ref{sec:methods-2}. We detail in which case each assumption is verified, especially in solving inverse problems with SGM priors. 

\begin{itemize}
    \item Assumption~\ref{ass:potential-lipschitz} is standard in both the sampling~\citep{dalalyan2017further, cheng2020stochastic, balasubramanian2022towards} and optimization~\citep{nesterov2018lectures} literature. 
    \item Assumptions~\ref{ass:sfo} and~\ref{ass:stochastic-potential-lipschitz} are standard in the analysis of stochastic gradient methods with variance reduction~\citep{fang2018spider,cutkosky2019momentum,liu2020improved,li2021page}, and are also commonly imposed in the Langevin sampling with stochastic gradients~\citep{chatterji2018theory, dalalyan2019user, cheng2020stochastic, balasubramanian2022towards}. 
    \item Assumption~\ref{ass:poincare-inequality} is satisfied by all strongly log-concave measures, but it does not hold in general for non-log-concave distributions. This assumption enables convergence in squared total variation distance with constant per-iteration computational cost $\BigO(1)$, a guarantee not established for strongly log-concave distributions in prior work on variance-reduced Langevin sampling~\citep{chatterji2018theory}.
    \item Assumption~\ref{ass:smoothing-err-bounded} bounds the score mismatch between the true score $-\nabla h(\vx)$ and the smoothed score $-\nabla h_{\sigma_k}(\vx)$. This assumption follows prior work on annealed Langevin sampling with SGM prior~\citep{sun2024provable}. In special cases, such as when the prior $p(\vx)$ is Gaussian, the upper bound can be derived analytically. However, obtaining a closed-form expression for the mismatch is generally intractable for arbitrary distributions. 
    \item Assumption~\ref{ass:score-network} imposes two conditions on the learned score network $\gS_\theta(\vx, \sigma_k)$: a bounded norm and a bounded approximation error. The bounded-norm condition is consistent with the scaling behavior of diffusion processes on manifolds, where the score scales as $1/\sigma_k$ to provide the restoring force~\citep{song2019generative, song2020improved}. The approximation error bound is justified by interpreting the target prior $p_{\sigma_k}$ as the population distribution rather than the empirical distribution. Under this interpretation, the assumption enforces a generalization condition on the score network, namely that it approximates the smooth population score with arbitrary precision ($\varepsilon_{\sigma_k} \to 0$)~\citep{pidstrigach2022score}. These assumptions align with prior theoretical analyses in~\citep{lee2023convergence, sun2024provable}.
\end{itemize}

\section{THEOREMS AND LEMMAS}\label{proofs:theorem-and-lemmas}

\subsection{Notation}

\textbf{Basic Notation.}
We denote by $\mathbb{R}^d$ the $d$-dimensional Euclidean space.
For $\vx \in \mathbb{R}^d$, $\|\vx\|$ denotes the Euclidean norm $\|\vx\|_2$. Vectors and matrices are denoted by \textbf{boldface} letters. 

\textbf{Langevin Diffusion.} We consider the Langevin diffusion defined as the solution to the stochastic differential equation (SDE)
\[
d \vx_t
=
- \nabla f(\vx_t)\, dt
+
\sqrt{2}\, d\mB_t,
\]
where $(\mB_t)_{t \ge 0}$ is a standard $d$-dimensional Brownian motion. The target distribution is denoted by $\pi$, and is given by
\[
\pi(\vx) \propto \exp(-f(\vx)).
\]
Under mild conditions, this SDE admits $\pi$ as its unique stationary distribution. Throughout, we assume that $f$ is continuously differentiable and that its gradient $\nabla f$ is Lipschitz continuous (see Assumption~\ref{ass:potential-lipschitz}). We assume access to stochastic first-order oracles $\nabla f(\vx,\xi)$, where  $\xi$ denotes a random variable capturing external randomness. Note that the stochastic first-order oracles are unbiased and have bounded variance (see Assumption~\ref{ass:sfo}). Discretizing the above SDE with step size $\gamma > 0$ yields the Langevin Monte Carlo (LMC) update
\[
\vx_{(k+1)\gamma}
=
\vx_{k\gamma}
-
\gamma \nabla f(\vx_{k\gamma})
+
\sqrt{2}\,\big(\mB_{(k+1)\gamma} - \mB_{k\gamma}\big),
\]where $\mB_{k\gamma}\in \sR^d$ is Brownian motion. We further define $\vg_k$ as estimator of the gradient
$\nabla f(\vx_{k\gamma})$ at iteration $k$.
Replacing the exact gradient by $\vg_k$,
the LMC update becomes
\[
\vx_{(k+1)\gamma}
=
\vx_{k\gamma}
-
\gamma \vg_k
+
\sqrt{2}\,\big(\mB_{(k+1)\gamma} - \mB_{k\gamma}\big).
\] The corresponding continuous-time interpolation is given by
\[
\vx_t
=
\vx_{k\gamma}
-
(t-k\gamma) \vg_k
+
\sqrt{2}\,\big(\mB_t - \mB_{k\gamma}\big),
\qquad
t \in [k\gamma,(k+1)\gamma],
\] and we denote by $\mu_t$ the law of $\vx_t$.

\textbf{Posterior Sampling.}
In the posterior sampling setting for ill-posed inverse problems of the form
\(
\vy = \mA(\vx) + \vn,
\)
we model the prior as
\(
p(\vx) \propto \exp(-h(\vx)),
\)
where $h(\vx)$ is the prior potential.
We model the likelihood as
\(
p(\vy|\vx) \propto \exp(-f(\vx)),
\)
where $f(\vx)$ is the likelihood potential.
By Bayes' rule, the posterior satisfies
\[
p(\vx|\vy)
=
\frac{p(\vy|\vx)\, p(\vx)}{\int p(\vy|\vu)\, p(\vu)\, d\vu}
\propto
p(\vy|\vx)\, p(\vx)
\propto
\exp\big(-(f(\vx) + h(\vx))\big).
\]Sampling from the posterior can be performed via Langevin dynamics
\[
d \vx_t
=
- \big( \nabla f(\vx_t) + \nabla h(\vx_t) \big)\, dt
+
\sqrt{2}\, d\mB_t,
\]
which admits $\pi(\vx|\vy)$ as its stationary distribution. For ease of notation, we omit the explicit dependence on $\vy$
and write $\pi(\vx|\vy)$ simply as $\pi(\vx)$ in the proofs. Discretizing this SDE with step size $\gamma > 0$ yields
\[
\vx_{(k+1)\gamma}
=
\vx_{k\gamma}
-
\gamma \big( \nabla f(\vx_{k\gamma})
+ \nabla h(\vx_{k\gamma}) \big)
+
\sqrt{2}\,\big(\mB_{(k+1)\gamma} - \mB_{k\gamma}\big).
\]We use an estimator $\vg_k$ (instantiated as ML-VRLD (\ref{def:ml-vrld:estimator}) or SL-VRLD (\ref{def:sl-vrld:estimator}))
to approximate $\nabla f(\vx_{k\gamma})$.
Moreover, we use a score network $\gS_\theta(\vx,\sigma_k)$ to approximate the smoothed prior score,
i.e., $\gS_\theta(\vx,\sigma_k) \approx -\nabla h_{\sigma_k}(\vx)$.
With annealing and smoothing schedules $(\alpha_k)_{k=0}^N$ and $(\sigma_k)_{k=0}^N$,
the annealed posterior update is
\[
\vx_{(k+1)\gamma}
=
\vx_{k\gamma}
-
\gamma\Big(\vg_k - \alpha_k \gS_\theta(\vx_{k\gamma},\sigma_k)\Big)
+
\sqrt{2}\,\big(\mB_{(k+1)\gamma}-\mB_{k\gamma}\big).
\]

The corresponding continuous-time interpolation on $t\in[k\gamma,(k+1)\gamma]$ is
\[
\vx_t
=
\vx_{k\gamma}
-
(t-k\gamma)\Big(\vg_k - \alpha_k \gS_\theta(\vx_{k\gamma},\sigma_k)\Big)
+
\sqrt{2}\,\big(\mB_t-\mB_{k\gamma}\big),
\qquad t\in[k\gamma,(k+1)\gamma].
\]
\textbf{KL Divergence and Fisher Information.}
The Kullback–Leibler divergence between two probability measures $\mu$ and $\pi$ is defined as
\[
\mathrm{KL}(\mu \| \pi)
:=
\int_{\mathbb{R}^d}
\mu(\vx)
\log \frac{\mu(\vx)}{\pi(\vx)} d\vx.
\]
The relative Fisher information is
\[
\mathrm{FI}(\mu \| \pi)
:=
\int_{\mathbb{R}^d}
\mu(\vx)
\left\|
\nabla \log \frac{\mu(\vx)}{\pi(\vx)}
\right\|_2^2
d\vx.
\]
\textbf{Probability Measures.} Expectation with respect to a distribution $\mu$
is denoted by $\mathbb{E}_{\vx \sim \mu}[\cdot]$.
When clear from context, we write $\mathbb{E}[\cdot]$.

\textbf{Big-O Notation.}
For two positive sequences $f(n)$ and $g(n)$,
we write $f(n) = \BigO(g(n))$ if there exists $C>0$
such that $f(n) \le C g(n)$ for sufficiently large $n$. In particular, $\BigO(1)$ denotes a positive constant.

\subsection{Lemmas}

We begin by presenting several technical lemmas that are used throughout the proofs of the main theorems. The following lemma provides the starting point for our analysis.

\begin{lemma}\textnormal{(Lemma 23 in~\citet{balasubramanian2022towards})}\label{lemma:sampling_first}
    Consider the stochastic process defined by 
    \begin{equation}
        \vx_t := \vx_0 - t\vg_0 + \sqrt{2}\mB_t, \quad \text{with $\vg_0 = g(\vx_0)$}, \quad \vx_0\sim\mu_0
    \end{equation} where $\vg_0$ is integrable and $(\mB_t)_{t\geq 0}$ is a standard Brownian motion in $\sR^d$ independent of $(\vx_0, \vg_0)$. Then, writing $\mu_t$ for the law of $\vx_t$, we have 
    \begin{align}
        \frac{d}{dt}\mathrm{KL}(\mu_t\|\pi) & \leq -\frac{3}{4}\mathrm{FI}(\mu_t\|\pi) + \E \left[\|\nabla f(\vx_t) - \E[\vg_0|\vx_t]\|^2\right] \notag \\
        & \leq  -\frac{3}{4}\mathrm{FI}(\mu_t\|\pi) + \E \left[\|\nabla f(\vx_t) - \vg_0\|^2\right] , 
    \end{align}where $\pi\propto e^{-f}$.
\end{lemma}

\begin{lemma}\textnormal{(Lemma 20 in~\citet{chewi2024analysis})}\label{lemma:sampling_final}
Assume that $f(\vx)$ has an $L$-Lipschitz gradient, and let $\pi \propto e^{-f}$. Then, for any probability measure $\mu$, it holds that 
\begin{equation}
\E_{\vx\sim\mu}\!\left[\|\nabla f(\vx)\|^2\right]
\leq \mathrm{FI}\!\left(\mu\|\pi\right) + 2dL .
\end{equation}
\end{lemma}
We use the transportation inequality to analyze the convergence in squared total variation (TV) distance under a Poincaré inequality assumption.
\begin{lemma}\textnormal{(Theorem 3.1 in ~\citet{guillin2009transportation})}\label{lemma:tv_lemma}
    If $\pi\propto e^{-f}$ satisfies a Poincaré inequality, i.e.\ for every smooth, compactly supported 
$f:\mathbb{R}^d\to\mathbb{R}$,
\[
\mathrm{Var}_{\pi}(f)\;\le\;C_{\mathrm{PI}}\,
\mathbb{E}_{\pi}\!\bigl[\|\nabla f\|^2\bigr],
\]
then for any probability measure $\mu$,
\[
\|\mu-\pi\|_{\mathrm{TV}}^2\;\le\;4C_{\mathrm{PI}}\,
\mathrm{FI}(\mu\|\pi).
\]
\end{lemma}

\subsection{Theorems}\label{appendix:theorems}

We first state a theorem characterizing the sufficient number of gradient computations per iteration to achieve $\varepsilon$-relative Fisher information for SGLD~(\ref{def:lmc-formulation-batch}) algorithm. 

\phantomsection
\paragraph{Theorem B.1}\textnormal{(Theorem 15 in \citet{balasubramanian2022towards})}\label{thm:sgld-batch-size}
    \textit{Let $(\mu_t)_{t\geq 0}$ denote the law of the interpolation generated by~(\ref{def:ml-vrld:interpolation}) with the estimator defined as $\vg_k= \frac{1}{b}\sum_{i=1}^b\nabla f(\vx_{k\gamma},\xi_k^i)$, where $b\geq 1$. Assume that the potential f defining the target distribution $\pi\propto e^{-f}$ satisfies Assumptions~\ref{ass:potential-lipschitz} and~\ref{ass:sfo}. Then, for any step size $\gamma\in\left(0, \frac{1}{14L}\right)$, and for any $N\ge 1$, it holds that}
    \begin{equation*}
        \frac{1}{N\gamma}\int_{0}^{N\gamma}\mathrm{FI}(\mu_t\|\pi) \,dt \leq \frac{2\mathrm{KL}(\mu_0\|\pi)}{N\gamma} + 16L^2d\gamma + \frac{8\sigma^2}{b}
    \end{equation*}\textit{Furthermore, choose $\gamma = \frac{1}{2L}\sqrt{\frac{\mathrm{KL}(\mu_0\|\pi)}{dN}}$ and $b = \lceil \sigma^2 \sqrt{N} \rceil$. Then, the time-averaged law $\bar{\mu}_{N\gamma}=\frac{1}{N\gamma}\int_0^{N\gamma}\mu_t\,dt$ satisfies $\mathrm{FI}(\bar{\mu}_{N\gamma}\|\pi)\leq \varepsilon$ with $\BigO(\frac{L^2d^2}{\varepsilon^2})$ iterations and $\BigO(\frac{Ld\sigma^2}{\varepsilon})$ gradient computations per iteration. }

\textit{Proof. } We consider the following interpolation for SGLD with mini-batch estimates 
\begin{equation*}
    \vx_{t}\coloneqq \vx_{k\gamma} - (t-k\gamma)\frac{1}{b}\sum_{i=1}^b\nabla f(\vx_{k\gamma}, \xi_k^i) +\sqrt{2}(\mB_t - \mB_{k\gamma})
\end{equation*}for $t\in[k\gamma, (k+1)\gamma]$. Under Assumptions~\ref{ass:potential-lipschitz} and~\ref{ass:sfo}, we follow the proof of Theorem 15 in~\citet{balasubramanian2022towards} to obtain
\begin{equation}\label{sgld-thm:fi-bound}
    \frac{1}{N\gamma}\int_{0}^{N\gamma}\mathrm{FI}(\mu_t\|\pi) \,dt \leq \frac{2\mathrm{KL}(\mu_0\|\pi)}{N\gamma} + 16L^2d\gamma + \frac{8\sigma^2}{b}
\end{equation}for all $\gamma\in\left(0, \frac{1}{14L}\right)$. Here, $\mathrm{KL(\mu_0\|\pi})$ denotes the initial Kullback-Leibler divergence and it scales on the order of $d$ (see, e.g., Lemma 1 in~\citet{vempala2019rapid} or Appendix A in~\citet{chewi2025analysis}). We then choose 
\begin{equation*}
    \gamma = \frac{1}{2L}\sqrt{\frac{\mathrm{KL}(\mu_0\|\pi)}{dN}}, \quad \text{and} \quad b = \left\lceil\sigma^2\sqrt{N}\right\rceil.
\end{equation*}Plugging these definitions into~(\ref{sgld-thm:fi-bound}), we get
\begin{equation}
    \mathrm{FI}(\bar{\mu}_{N\gamma}\|\pi )\leq \frac{1}{N\gamma}\int_{0}^{N\gamma}\mathrm{FI}(\mu_t\|\pi) \,dt \leq \frac{\left(12L\sqrt{\mathrm{KL}(\mu_0\|\pi)d}+8\right)}{\sqrt{N}},
\end{equation}where the first inequality follows from the convexity of the Fisher information combined with Jensen’s inequality. Since $\mathrm{KL}(\mu_0\|\pi) = \mathcal{O}(d)$, we have $\mathrm{FI}(\bar{\mu}_{N\gamma}\|\pi)\leq \varepsilon$ if
\begin{equation}
    N=\mathcal{O}\biggl(\frac{L^2d^2}{\varepsilon^2}\biggr),
\end{equation}which implies $b = \BigO(\frac{Ld\sigma^2}{\varepsilon})$. This concludes the proof.\qedbox

In the remainder of this section, we present our theoretical results and their proofs. First, we give results for ML-VRLD and SL-VRLD sampling algorithms to sample from a target distribution $\pi\propto e^{-f(\vx)}$. Then, we present the results for the application of these algorithms to inverse problems with score-based generative prior, where the goal is to sample from the posterior $p(\vx|\vy)\propto p(\vy|\vx)p(\vx)$ with prior $p(\vx)\propto e^{-h(\vx)}$ and likelihood $p(\vy | \vx)\propto e^{-f(\vx)}$. In practice, the settings of interest are often high-variance, where $\sigma^2$, the upper bound on the variance of stochastic gradients in Assumption~\ref{ass:sfo}, satisfies $\sigma^2 > 1$. Nevertheless, to also cover the case $\sigma^2 < 1$, we state our theoretical results in terms of $\sigma_m \coloneqq \max\{\sigma, 1\}$. For ease of exposition, we use $\sigma$ throughout the main manuscript.

\subsubsection{Proof of Theorem~\ref{thm:ml-vrld-ficonv} (ML-VRLD: FI convergence)}\label{proof:ml-vrld-ficonv}
\paragraph{Theorem~\ref{thm:ml-vrld-ficonv} (ML-VRLD: FI convergence)}\textit{Let $\pi\propto e^{-f}$ be the target distribution, where the potential function $f$ satisfies Assumptions~\ref{ass:potential-lipschitz} and~\ref{ass:sfo}. Let $(\mu_t)_{t\geq 0}$ denote the law of continuous-time interpolation generated by~(\ref{def:ml-vrld:interpolation}) with estimator~(\ref{def:ml-vrld:estimator}). Then, for any step size $\gamma \in \left(0, \frac{1}{L\sqrt{27\phi(p,\beta)}}\right]$, where $\phi(p,\beta)= 1+\frac{4(1-p)\beta^2}{(1-(1-p)\beta)^2}$, and for any $N\geq 1$, it holds that 
\begin{equation}
    \frac{1}{N\gamma}\int_{0}^{N\gamma} \mathrm{FI}(\mu_t || \pi)\, dt\leq \frac{2C_0}{N\gamma} + 9\sigma^2\left[\frac{\frac{p}{b} + (1-p)(1-\beta)^2}{1-(1-p)\beta}\right] + \frac{72\gamma L^2 d}{(1-(1-p)\beta)^2},
\end{equation}where 
\begin{equation}
    C_0 = \mathrm{KL}(\mu_0\|\pi) + \frac{9\gamma}{(1-(1-p)\beta)^2}\E[\|\nabla f(\vx_0) - \vg_0\|^2].
\end{equation}Furthermore, choose
\[
\gamma=\frac{\sqrt d}{\sigma_m N^{3/4}\sqrt{L}},\qquad 
p=1-\beta=\frac{\sqrt{Ld}}{\sigma_m N^{1/4}},\qquad 
b=\Bigl\lceil \tfrac{1}{p}\Bigr\rceil,
\quad \text{where} \quad \sigma_m = \max\{\sigma,1\}.
\]
Then, the time-averaged law $\bar{\mu}_{N\gamma}= \frac{1}{N\gamma}\int_{0}^{N\gamma}\mu_t\,dt$ satisfies $\mathrm{FI}(\bar{\mu}_{N\gamma}\|\pi)\leq \varepsilon$ with $\BigO(\frac{\sigma_m^{4}L^{2}d^{2}}{\varepsilon^{4}})$ iterations and $\BigO(1)$ gradient computations per iteration. }

\textit{Proof.}
Using the interpolation argument in~(\ref{def:ml-vrld:interpolation}), we begin the proof by recalling Lemma~\ref{lemma:sampling_first}
\begin{align*}
\frac{d}{dt}\mathrm{KL}(\mu_t \| \pi)
&\leq
-\frac{3}{4}\mathrm{FI}(\mu_t \| \pi)
+
\E \left[\|\nabla f(\vx_t) - \vg_k\|^2\right].
\end{align*}
We decompose the error term by introducing $\nabla f(\vx_{k\gamma})$ and then apply Young’s inequality (i.e., \(\|a+b\|^2 \le 2\|a\|^2+2\|b\|^2\)), obtaining
\begin{align*}
\frac{d}{dt}\mathrm{KL}(\mu_t \| \pi)
&\leq
-\frac{3}{4}\mathrm{FI}(\mu_t \| \pi) +
2\E \left[\|\nabla f(\vx_t) - \nabla f(\vx_{k\gamma})\|^2\right] +
2\E \left[\|\nabla f(\vx_{k\gamma}) - \vg_k\|^2\right].
\end{align*}
Invoking Assumption~\ref{ass:potential-lipschitz}, we upper bound $\E\left[\|\nabla f(\vx_t) - \nabla f(\vx_{k\gamma})\|^2\right]$  using Lipschitz continuity of the potential $\nabla f$, which gives
\begin{align}
\frac{d}{dt}\mathrm{KL}(\mu_t \| \pi)
&\leq
-\frac{3}{4}\mathrm{FI}(\mu_t \| \pi) +
2L^2 \E\left[\|\vx_t - \vx_{k\gamma}\|^2\right] +
2e_k^2, \label{ml-vrld-proof:first-main-inequality}
\end{align}
where we define
$e_k^2 \coloneq \E\!\left[\|\nabla f(\vx_{k\gamma}) - \vg_k\|^2\right]$.
The next step is to bound $e_k^2$.
Using the definition of the ML-VRLD estimator (\ref{def:ml-vrld:estimator}),
$e_{k+1}^2$ can be expressed as follows.
With probability $p$, a mini-batch estimator of size $b$ is used,
while with probability $1-p$ a momentum-based update is applied.
This yields
\begin{align*}
e_{k+1}^2 
&=
p\,\E\!\left[
\left\|
\nabla f(\vx_{(k+1)\gamma})
-\frac{1}{b}\sum_{i=1}^b \nabla f(\vx_{(k+1)\gamma}, \xi_{k+1}^i)
\right\|^2
\right] \\
&\quad
+ (1-p)\,\E\!\left[
\left\|
\nabla f(\vx_{(k+1)\gamma})
-
\bigl(\beta \vg_k + (1-\beta)\nabla f(\vx_{(k+1)\gamma}, \xi_{k+1})\bigr)
\right\|^2
\right].
\end{align*}
Using Assumption~\ref{ass:sfo}, we upper bound the variance of the mini-batch estimator, which gives
\begin{align*}
e_{k+1}^2
\leq
\frac{p\sigma^2}{b} 
+ (1-p)\,\E\left[
\left\|
\nabla f(\vx_{(k+1)\gamma})
-
\bigl(\beta \vg_k + (1-\beta)\nabla f(\vx_{(k+1)\gamma}, \xi_{k+1})\bigr)
\right\|^2
\right].
\end{align*}
We decompose the second term by adding and subtracting
$\nabla f(\vx_{k\gamma})$ inside expectation, regrouping the terms and using Assumption \ref{ass:sfo}
\begin{align*}
e_{k+1}^2
&\leq
\frac{p\sigma^2}{b} + (1-p)\E\left[\| \beta(\nabla f(\vx_{k\gamma}) - \vg_k) + \beta(\nabla f(\vx_{(k+1)\gamma}) - \nabla f(\vx_{k\gamma})) \right. \\
&\qquad\qquad\qquad\left. + (1-\beta)(\nabla f(\vx_{(k+1)\gamma}) - \nabla f(\vx_{(k+1)\gamma},\xi_{k+1}))\|^2\right] \\
&\leq
\frac{p\sigma^2}{b} + (1-p)\beta^2\E\left[\| (\nabla f(\vx_{k\gamma}) - \vg_k) + (\nabla f(\vx_{(k+1)\gamma}) - \nabla f(\vx_{k\gamma})) \|^2\right] + (1-\beta)^2(1-p)\sigma^2.
\end{align*}

Applying Young’s inequality with parameter \(\rho_0>0\)
(i.e., \(\|a+b\|^2 \le (1+\rho_0)\|a\|^2 + (1+1/\rho_0)\|b\|^2\)),
together with Assumption~\ref{ass:potential-lipschitz} to bound $\E\left[\| (\nabla f(\vx_{k\gamma}) - \vg_k) + (\nabla f(\vx_{(k+1)\gamma}) - \nabla f(\vx_{k\gamma}))\|^2\right]$, we obtain
\begin{align*}
e_{k+1}^2
&\leq
\frac{p\sigma^2}{b} +
(1-p)\beta^2L^2\left(1+\tfrac{1}{\rho_0}\right)\Delta_k + 
(1-p)\beta^2(1+\rho_0)e_k^2 +
(1-p)(1-\beta)^2\sigma^2 \\
&=
(1-p)\beta^2L^2\left(1+\tfrac{1}{\rho_0}\right)\Delta_k + (1-p)\beta^2(1+\rho_0)e_k^2 + \left[\frac{p}{b} + (1-p)(1-\beta)^2\right]\sigma^2, 
\end{align*} 
where $\Delta_k \coloneqq \E\left[\|\vx_{(k+1)\gamma} - \vx_{k\gamma}\|^2\right]$. Furthermore, we choose $\rho_0 = \frac{1-(1-p)\beta}{2(1-p)\beta}$. Rearranging the terms then yields
\begin{align}
    \nonumber
    e_k^2 
    &\leq
    -\frac{2}{1-(1-p)\beta}(e_{k+1}^2 - e_k^2) +
    \frac{\left[\frac{p}{b}+(1-p)(1-\beta)^2\right]2\sigma^2}{1-(1-p)\beta} +
    \frac{2(1-p)\beta^2L^2(1+(1-p)\beta)}{(1-(1-p)\beta)^2}\Delta_k \\ \label{eq:upper-bound-ek}
    &\leq
    -\frac{2}{1-(1-p)\beta}(e_{k+1}^2 - e_k^2) +
    \frac{\left[\frac{p}{b}+(1-p)(1-\beta)^2\right]2\sigma^2}{1-(1-p)\beta} +
    \frac{4(1-p)\beta^2L^2}{(1-(1-p)\beta)^2}\Delta_k.
\end{align} Since \(p\in(0,1]\) and \(\beta\in(0,1]\), it follows that \((1-p)\beta \in [0,1)\), and hence
\(1+(1-p)\beta \le 2\). We have used this bound to simplify the last term. We now bound $\Delta_k$ by using the continuous interpolation (\ref{def:ml-vrld:interpolation}) of LMC. For $t \in [k\gamma,(k+1)\gamma]$,
\begin{align*}
    \vx_{t} - \vx_{k\gamma} &= - (t-k\gamma)\vg_k + \sqrt{2}\bigl(\mB_{t} - \mB_{k\gamma}\bigr).
\end{align*}
Taking squared norms on both sides and then taking expectations,
and using the independence of $\vg_k$ and the Brownian increment,
we obtain
\begin{align*}
    \E\left[\|\vx_t - \vx_{k\gamma}\|^2\right]
    &= (t-k\gamma)^2\E\left[\|\vg_k\|^2\right]
    + 2\E\left[\|\mB_t - \mB_{k\gamma}\|^2\right] \\
    &=(t-k\gamma)^2\E\!\left[\|\vg_k\|^2\right] + 2(t-k\gamma)d\\
    & \leq \gamma^2 \E\!\left[\|\vg_k\|^2\right] + 2(t-k\gamma)d \\ 
    &= \Delta_k,
\end{align*}where we have used the identity $\E\!\left[\|\mB_t-\mB_{k\gamma}\|^2\right]=(t-k\gamma)d$.
Writing
\[
\vg_k
=
(\vg_k-\nabla f(\vx_{k\gamma}))
+ (\nabla f(\vx_{k\gamma})-\nabla f(\vx_t))
+ \nabla f(\vx_t),
\]
and applying the inequality
$\|a+b+c\|^2 \le 3(\|a\|^2+\|b\|^2+\|c\|^2)$, we obtain
\begin{align*}
\Delta_k
&= \gamma^2\E\!\left[\|\vg_k\|^2\right] + 2\gamma d \\
&\le
3\gamma^2\E\!\left[\|\vg_k-\nabla f(\vx_{k\gamma})\|^2\right]
+ 3\gamma^2\E\!\left[\|\nabla f(\vx_{k\gamma})-\nabla f(\vx_t)\|^2\right]
+ 3\gamma^2\E\!\left[\|\nabla f(\vx_t)\|^2\right]
+ 2\gamma d .
\end{align*}
Using Lipschitz continuity of $\nabla f$ (Assumption \ref{ass:potential-lipschitz}) and the bound
$\E\!\left[\|\vx_t - \vx_{k\gamma}\|^2\right]\le \Delta_k$ yields
\begin{align*}
\Delta_k
\le
3\gamma^2 e_k^2
+ 3\gamma^2 L^2 \Delta_k
+ 3\gamma^2\E\!\left[\|\nabla f(\vx_t)\|^2\right]
+ 2\gamma d .
\end{align*}Substituting the upper bound (\ref{eq:upper-bound-ek}) for $e_k^2$ and rearranging the terms yields
\begin{align*}
\Delta_k
&\le
-\frac{6\gamma^2}{1-(1-p)\beta}\,\bigl(e_{k+1}^2-e_k^2\bigr)
+ \frac{6\gamma^2\sigma^2}{1-(1-p)\beta}
\left[\frac{p}{b} + (1-p)(1-\beta^2)\right] \\
&\quad
+ 3\gamma^2L^2
\left[1 + \frac{4(1-p)\beta^2}{\bigl(1-(1-p)\beta\bigr)^2}\right]\Delta_k
+ 3\gamma^2\E\!\left[\|\nabla f(x_t)\|^2\right]
+ 2\gamma d.
\end{align*}
To simplify notation, we define
\(
\phi(p,\beta) \coloneqq 1 + \frac{4(1-p)\beta^2}{(1-(1-p)\beta)^2}.
\)
With this definition and rearranging the terms, the inequality becomes
\begin{align*}
\bigl[1 - 3\gamma^2 L^2 \phi(p,\beta)\bigr]\Delta_k
&\le
-\frac{6\gamma^2}{1-(1-p)\beta}\,\bigl(e_{k+1}^2 - e_k^2\bigr)
+ \frac{6\gamma^2\sigma^2}{1-(1-p)\beta}
\left[\frac{p}{b} + (1-p)(1-\beta)^2\right] \\
& \quad + 3\gamma^2\E\!\left[\|\nabla f(x_t)\|^2\right]
+ 2\gamma d .
\end{align*}
Assuming the step size satisfies $\gamma \leq \frac{1}{27L^2\phi(p,\beta)}$, the coefficient of $\Delta_k$ on the left-hand side satisfies $1 - 3\gamma^2L^2\phi(p,\beta) \geq \frac{8}{9}$. Applying this bound and dividing both sides by $\frac{8}{9}$ yields
\begin{align}\label{eq:upper-bound-delta-k}
\Delta_k
&\le
-\frac{27\gamma^2}{4\bigl(1-(1-p)\beta\bigr)}\,
\bigl(e_{k+1}^2 - e_k^2\bigr) 
+ \frac{27\gamma^2\sigma^2}{4\bigl(1-(1-p)\beta\bigr)}
\left[\frac{p}{b} + (1-p)(1-\beta)^2\right]
+ \frac{27\gamma^2}{8}\,
\E\!\left[\|\nabla f(x_t)\|^2\right]
+ \frac{9}{4}\,\gamma d .
\end{align} Now, recall from~(\ref{ml-vrld-proof:first-main-inequality}), we have
\begin{align*}
\frac{d}{dt}\mathrm{KL}(\mu_t \| \pi)
&\leq
-\frac{3}{4}\mathrm{FI}(\mu_t \| \pi) +
2L^2 \E\left[\|\vx_t - \vx_{k\gamma}\|^2\right] +
2e_k^2,
\end{align*}
Applying (\ref{eq:upper-bound-ek}) and (\ref{eq:upper-bound-delta-k}) together with
$\E\!\left[\|x_t-x_{k\gamma}\|^2\right] \leq \Delta_k$,
we obtain
\begin{align*}
\frac{d}{dt}\mathrm{KL}(\mu_t \| \pi)
&\le
-\frac{3}{4}\,\mathrm{FI}(\mu_t \| \pi)
-\frac{1}{1-(1-p)\beta}
\left[4+\frac{27}{2}\gamma^2L^2\phi(p,\beta)\right]
\bigl(e_{k+1}^2 - e_k^2\bigr) \\
&\quad
+ \frac{\sigma^2}{1-(1-p)\beta}
\left[4+\frac{27}{2}\gamma^2L^2\phi(p,\beta)\right]
\left[\frac{p}{b} + (1-p)(1-\beta)^2\right] \\
&\quad
+ \frac{27}{4}\gamma^2L^2\phi(p,\beta)\,
\E\!\left[\|\nabla f(\vx_t)\|^2\right]
+ \frac{9}{2}\gamma L^2\phi(p,\beta)\, d .
\end{align*}
Applying Lemma~\ref{lemma:sampling_final} to bound
$\E\!\left[\|\nabla f(\vx_t)\|^2\right]$, and using the step size condition
$\gamma \leq \frac{1}{27L^2\phi(p,\beta)}$, we simplify the above inequality as
\begin{align} \label{eq:ml-vrld-kl-bound}
\nonumber
\frac{d}{dt}\mathrm{KL}(\mu_t \| \pi)
&\le
-\frac{1}{2}\,\mathrm{FI}(\mu_t \| \pi)
-\frac{1}{1-(1-p)\beta}
\left[4+\frac{27}{2}\gamma^2L^2\phi(p,\beta)\right]
\bigl(e_{k+1}^2-e_k^2\bigr) \\
&\quad
+ \frac{9\,\sigma^2}{2\bigl(1-(1-p)\beta\bigr)}
\left[\frac{p}{b} +(1-p)(1-\beta)^2\right]
+ 9\gamma L^2\phi(p,\beta)\,d .
\end{align} We now integrate both sides over $t \in [k\gamma,(k+1)\gamma]$
and define \,\(
\mathcal{L}_k \coloneqq \mathrm{KL}(\mu_{k\gamma}\|\pi)
+ \frac{\gamma}{1-(1-p)\beta}\left[4 + \frac{27}{2}\gamma^2L^2\phi(p, \beta)\right]e_k^2 .
\)
This yields
\begin{align*}
\mathcal{L}_{k+1} - \mathcal{L}_k
&\le
-\frac{1}{2}
\int_{k\gamma}^{(k+1)\gamma} \mathrm{FI}(\mu_t \| \pi)\,dt
+ \frac{9\,\gamma\sigma^2}{2\bigl(1-(1-p)\beta\bigr)}
\left[\frac{p}{b} + (1-p)(1-\beta)^2\right]
+ 9\gamma^2 L^2 \phi(p,\beta)d .
\end{align*} Iterating the inequality for $k=0,1,\dots,N-1$, we obtain
\begin{align*}
\mathcal{L}_{N} - \mathcal{L}_0
&\le
-\frac{1}{2}
\int_{0}^{N\gamma} \mathrm{FI}(\mu_t \| \pi)\,dt
+ \frac{9N\gamma\sigma^2}{2\bigl(1-(1-p)\beta\bigr)}
\left[\frac{p}{b} + (1-p)(1-\beta)^2\right]
+ 9N\gamma^2 L^2 \phi(p,\beta)d .
\end{align*} Rearranging the terms, we get
\begin{align*}
\frac{1}{N\gamma}\int_{0}^{N\gamma}\mathrm{FI}(\mu_t \| \pi)\,dt
&\le
\frac{2\mathcal{L}_0}{N\gamma}
+ 9\sigma^2
\left[\frac{\frac{p}{b} + (1-p)(1-\beta)^2}{1-(1-p)\beta}\right]
+ 18\gamma L^2\phi(p,\beta)d.
\end{align*} Note that $\mathcal{L}_0 = \mathrm{KL}(\mu_0 \| \pi) + \frac{\gamma}{1-(1-p)\beta}\left(4 + \frac{27}{2}\gamma^2L^2\phi(p, \beta)\right)e_0^2\leq \mathrm{KL}(\mu_0 \| \pi) + \frac{9\gamma e_0^2}{2(1-(1-p)\beta)}\eqqcolon C_0$, then
\begin{align*}
    \frac{1}{N\gamma}\int_{0}^{N\gamma}\mathrm{FI}(\mu_t \| \pi)\,dt \leq \frac{2C}{N\gamma} + 9\sigma^2\left[\frac{\frac{p}{b}+(1-p)(1-\beta)^2}{1-(1-p)\beta}\right] + 18\gamma L^2\phi(p,\beta)d.
\end{align*} Also $\phi(p, \beta)=1+\frac{4(1-p)\beta^2}{(1-(1-p)\beta)^2}\leq1+\frac{4(1-p)\beta}{(1-(1-p)\beta)^2} = \left(\frac{1+(1-p)\beta}{1-(1-p)\beta}\right)^2 \leq \frac{4}{(1-(1-p)\beta)^2}$ since $1+(1-p)\beta \leq 2$,
then the inequality can be rewritten as 
\begin{align}\label{eq:ml-vrld-integral}
    \frac{1}{N\gamma}\int_{0}^{N\gamma}\mathrm{FI}(\mu_t \| \pi)\,dt \leq \frac{2C_0}{N\gamma} + 9\sigma^2\left[\frac{\frac{p}{b}+(1-p)(1-\beta)^2}{1-(1-p)\beta}\right] + \frac{72\gamma L^2d}{(1-(1-p)\beta)^2}.
\end{align} This concludes the first part of the proof. Next we choose \[\gamma = \frac{\sqrt{d}}{\sigma_mN^{3/4}\sqrt{L}} ,\quad p=(1-\beta)=\frac{\sqrt{Ld}}{N^{1/4}\sigma_m} \quad \text{and} \quad b=\Bigl\lceil \tfrac{1}{p}\Bigr\rceil \quad \text{where} \quad \sigma_m \coloneqq \mathrm{max}\{\sigma, 1\}
\] and define the time-averaged law $\bar{\mu}_{N\gamma} \coloneqq \frac{1}{N\gamma}\int_0^{N\gamma}\mu_t\,dt$. Since $\mathrm{FI}(\cdot\|\pi)$ is convex, using Jensen's inequality, we obtain the upper bound for $\mathrm{FI}(\bar{\mu}_{N\gamma}\|\pi)$ as

\begin{align*}
    \mathrm{FI}(\bar{\mu}_{N\gamma}\|\pi) \le \frac{1}{N\gamma}\int_{0}^{N\gamma}\mathrm{FI}(\mu_t \| \pi)\,dt \leq \underbrace{\frac{2C_0}{N\gamma}}_{\mathrm{(I)}} + \underbrace{9\sigma^2\left[\frac{\frac{p}{b}+(1-p)(1-\beta)^2}{1-(1-p)\beta}\right]}_{\mathrm{(II)}} + \underbrace{\frac{72\gamma L^2d}{(1-(1-p)\beta)^2}}_{\mathrm{(III)}}.
\end{align*}Substituting the definition of $\gamma$ into $\mathrm{(I)}$, we have
\begin{equation}
    \mathrm{(I)} = \frac{2C_0\sigma_m \sqrt{L}}{N^{1/4}\sqrt{d}}.\label{thm1:coeff:0}
\end{equation}To bound $\mathrm{(II)}$, we proceed as 
\begin{align}
    \mathrm{(II)}&\le 9\sigma_m^2 \biggl[\frac{p^2 + (1-p)p^2}{1-(1-p)^2}\biggr]\label{thm1:coeff:1} \\ 
    & \le 18\sigma_m^2p.\label{thm1:coeff:2}
\end{align}We use $\sigma \le \sigma_m$, $b\ge 1/p$, and $1-\beta = p$ to get~(\ref{thm1:coeff:1}). Following that, we use $p\in(0,1]$ to get~(\ref{thm1:coeff:2}). Substituting the definition of $p$, we have 
\begin{equation}
    \mathrm{(II)}\le \frac{18\sigma_m\sqrt{Ld}}{N^{1/4}}. \label{thm1:coeff:3}
\end{equation}Similarly, to bound $\mathrm{(III)}$, 
\begin{align}
    \mathrm{(III)} & \le \frac{72\gamma L^2 d}{p^2}\label{thm1:coeff:1.1}\\
    & \le \frac{72\sigma_m \sqrt{Ld}}{N^{1/4}},\label{thm1:coeff:1.2}
\end{align}where we use $p=1-\beta$ and $p\in(0,1]$ to obtain~(\ref{thm1:coeff:1.1}), and substitute the definitions of $\gamma$ and $p$ to get~(\ref{thm1:coeff:1.2}). Combining~(\ref{thm1:coeff:0}, \ref{thm1:coeff:3}, \ref{thm1:coeff:1.2}), we have

\begin{align*}
\mathrm{FI}(\bar{\mu}_{N\gamma}\|\pi)
& \leq \frac{2C_0\sigma_m\sqrt{L}}{N^{1/4}\sqrt{d}} + \frac{90\sigma_m\sqrt{Ld}}{N^{1/4}} \\ \nonumber
&\leq \left(\frac{2C_0}{d} + 90\right)\frac{\sigma_m\sqrt{Ld}}{N^{1/4}},
\end{align*} where $C_0$ is usually taken to be of order $d$, see e.g. Lemma 1 in~\citet{vempala2019rapid} or Appendix A in~\citet{chewi2025analysis}. Thus, the above bound implies that, to ensure
$\mathrm{FI}(\bar{\mu}_{N\gamma}\|\pi) \le \varepsilon$,
\begin{equation*}
    N=\BigO\biggl(\frac{\sigma_m^4L^2 d^2}{\varepsilon^4}\biggl)
\end{equation*}iterations is sufficient. And the average number of gradient computations per iteration is
\begin{equation}
    pb + (1-p) = \BigO(1).
\end{equation}We next verify that the parameter constraints imposed throughout the analysis
are compatible with this choice of $N$.
First, the requirement
$p=(1-\beta)=\frac{\sqrt{dL}}{N^{1/4}\sigma_m}\le 1$
imposes the lower bound
$N\geq \frac{d^2L^2}{\sigma_m^4}$,
which is satisfied by the above choice of $N$. Second, we show that the step-size condition
$\gamma \le \frac{1}{L\sqrt{27\phi(p,\beta)}}$
is satisfied. Since $\phi(p,\beta)\leq \frac{4}{(1-(1-p)\beta)^2}$, we have 

\begin{equation*}
    \gamma \le \frac{1-(1-p)\beta}{L\sqrt{108}} \le \frac{1}{L\sqrt{27\phi(p,\beta)}}.
\end{equation*}Substituting the definitions of $\beta = 1-p$ and $p= \frac{\sqrt{dL}}{\sigma_m N^{1/4}}$, we get 
\begin{align*}
    \frac{1-(1-p)\beta}{L\sqrt{108}} & = \frac{1-(1-p)^2}{L\sqrt{108}} \\
    & \ge \frac{p}{L\sqrt{108}}\\
    & = \frac{\sqrt{d}}{\sigma_m N^{1/4}\sqrt{108L}}.
\end{align*}Thus, the step-size condition holds provided that
\[
\gamma \leq \frac{\sqrt{d}}{\sigma_m N^{1/4}\sqrt{108L}}.
\]
Using the definition of step size
\[
\gamma = \frac{\sqrt{d}}{\sigma_m N^{3/4}\sqrt{L}},
\]
this condition reduces to the following lower bound $N \ge 108$, which is satisfied by $N=\BigO\biggl(\frac{\sigma_m^4L^2d^2}{\varepsilon^4}\biggr)$. \qedbox

\subsubsection{Proof of Corrolary~\ref{corr:ml-vrld-tvconv} (ML-VRLD: squared TV distance convergence)}\label{proof:ml-vrld-tvconv}
\paragraph{Corrolary~\ref{corr:ml-vrld-tvconv} (ML-VRLD: squared TV distance convergence)}\textit{
Let $\pi\propto e^{-f}$ be the target distribution, where the potential function $f$ satisfies Assumptions~\ref{ass:potential-lipschitz},~\ref{ass:sfo}, and~\ref{ass:poincare-inequality}. Let $(\mu_t)_{t\geq 0}$ denote the law of continuous-time interpolation generated by~(\ref{def:ml-vrld:interpolation}) with estimator~(\ref{def:ml-vrld:estimator}). Then, for any step size $\gamma \in \left(0, \frac{1}{L\sqrt{27\phi(p,\beta)}}\right]$, where $\phi(p,\beta)= 1+\frac{4(1-p)\beta^2}{(1-(1-p)\beta)^2}$, and for any $N\geq 1$, it holds that  
\begin{equation}
    \|\Bar{\mu}_{N\gamma} - \pi\|_{\mathrm{TV}}^2\leq \frac{8C_0C_{\mathrm{PI}}}{N\gamma} + 36C_{\mathrm{PI}}\sigma^2\left[\frac{\frac{p}{b} + (1-p)(1-\beta)^2}{1-(1-p)\beta}\right] + \frac{288C_{\mathrm{PI}}\gamma L^2 d}{(1-(1-p)\beta)^2},
\end{equation}where 
\begin{equation}
    C_0 = \mathrm{KL}(\mu_0\|\pi) + \frac{9\gamma}{(1-(1-p)\beta)^2}\E[\|\nabla f(\vx_0) - \vg_0\|^2].
\end{equation}Furthermore, choose
\[
\gamma=\frac{\sqrt d}{\sigma_m N^{3/4}\sqrt{L}},\qquad 
p=1-\beta=\frac{\sqrt{Ld}}{N^{1/4}\sigma_m},\qquad 
b=\Bigl\lceil \tfrac{1}{p}\Bigr\rceil,
\quad \text{where } \sigma_m = \max\{\sigma,1\}.
\]
Then, the time-averaged law $\bar{\mu}_{N\gamma}= \frac{1}{N\gamma}\int_{0}^{N\gamma}\mu_t\,dt$ satisfies $\|\Bar{\mu}_{N\gamma} - \pi\|_{\mathrm{TV}}^2\le \varepsilon$ with$~ \BigO\!\left(\frac{C_\mathrm{PI}^4\sigma_m^{4}L^{2}d^{2}}{\varepsilon^{4}}\right)$ iterations and $\BigO(1)$ gradient computations per iteration.  } 

\textit{Proof.} Recall from the proof of Theorem~\ref{thm:ml-vrld-ficonv} (Appendix \ref{proof:ml-vrld-ficonv}) that the Fisher information of the time-averaged law $\bar{\mu}_{N\gamma}$ satisfies (\ref{eq:ml-vrld-integral})
\begin{align*}
    \mathrm{FI}(\bar{\mu}_{N\gamma}\|\pi)
\leq \frac{1}{N\gamma}\int_{0}^{N\gamma}\mathrm{FI}(\mu_t \| \pi)\,dt \leq \frac{2C_0}{N\gamma} + 9\sigma_{m}^2\left[\frac{\frac{p}{b}+(1-p)(1-\beta)^2}{1-(1-p)\beta}\right] + \frac{72\gamma L^2d}{(1-(1-p)\beta)^2},
\end{align*}where we use $\sigma\le \sigma_m$. Using Lemma~\ref{lemma:tv_lemma} and Assumption~\ref{ass:poincare-inequality}, we get 
\begin{equation*}
        \|\Bar{\mu}_{N\gamma} - \pi\|_{\mathrm{TV}}^2\leq \frac{8C_0C_{\mathrm{PI}}}{N\gamma} + 36C_{\mathrm{PI}}\sigma_m^2\left[\frac{\frac{p}{b} + (1-p)(1-\beta)^2}{(1-(1-p)\beta)}\right] + \frac{288C_{\mathrm{PI}}\gamma L^2 d}{(1-(1-p)\beta)^2}.
\end{equation*}
Choose
    \[
    \gamma=\frac{\sqrt d}{\sigma_m N^{3/4}\sqrt{L}},\qquad 
    p=1-\beta=\frac{\sqrt{Ld}}{N^{1/4}\sigma_m},\qquad 
    b=\Bigl\lceil \tfrac{1}{p}\Bigr\rceil,
    \quad \text{where } \sigma_m \coloneqq \max\{\sigma,1\}.
    \]
Following the same steps in~(\ref{thm1:coeff:0}, \ref{thm1:coeff:3}, \ref{thm1:coeff:1.2}), we obtain
\begin{align*}
\|\Bar{\mu}_{N\gamma} - \pi\|_{\mathrm{TV}}^2\le \left(\frac{2C_0}{d}+90\right)\frac{4C_{\mathrm{PI}}\sigma_m\sqrt{Ld}}{N^{1/4}},
\end{align*}where $C_0$ is in the order of $d$. This implies that, to ensure
$\|\Bar{\mu}_{N\gamma} - \pi\|_{\mathrm{TV}}^2 \le \varepsilon$,
it suffices to choose $N$ such that
\(
N^{1/4} \;\ge\; \frac{\tilde{C}C_{\mathrm{PI}}\sigma_m\sqrt{Ld}}{\varepsilon},
\)
where $\tilde C\coloneqq \frac{8C_0}{d} + 360>0$.
Equivalently, this yields iteration complexity of
\(
\BigO\!\left(\frac{C_{\mathrm{PI}}^4\sigma_m^4 L^2 d^2}{\varepsilon^4}\right)
\) with $\BigO(1)$ gradient computations per iteration. Since the constraints on $N$ in Theorem~\ref{thm:ml-vrld-ficonv} remain unchanged, they continue to hold.  \qedbox

\subsubsection{Proof of Theorem~\ref{thm:ml-vrld-weakconv} (ML-VRLD: weak convergence)}\label{proof:ml-vrld-weakconv}
\paragraph{Theorem~\ref{thm:ml-vrld-weakconv} (ML-VRLD: weak convergence)}\textit{
Let $\pi\propto e^{-f}$ be the target distribution, where the potential function $f$ satisfies Assumptions~\ref{ass:potential-lipschitz} and~\ref{ass:sfo}. Let $(\mu_t)_{t\geq 0}$ denote the law of continuous-time interpolation generated by~(\ref{def:ml-vrld:interpolation}) with estimator~(\ref{def:ml-vrld:estimator}). Define the time-varying parameters as follows
\[
\gamma_k=\frac{C_\gamma}{k^{3/2}},\qquad 
p_k = 1-\beta_k =\frac{C_p}{k^{1/2}},\qquad
b_k=\left\lceil \frac{1}{p_k}\right\rceil,
\qquad \forall k\geq 1,
\]\
where \(0<C_{p}<1\) and \(C_\gamma>0\) are numerical constants. 
Then, the time-averaged law $\bar{\mu}_{\tau_n}= \frac{1}{\tau_n}\int_0^{\tau_n} \mu_t \,dt$, where $\tau_n= \sum_{k=1}^{n}\gamma_k$, converges weakly to $\pi$. }

\textit{Proof.} Given time-varying parameters \(\gamma_k,b_k,p_k\) at iteration \(k\), define the cumulative time \(\tau_n\) and averaged law \(\bar{\mu}_{\tau_n}\) at iteration \(n\) as
\[
\tau_{n}\coloneqq\sum_{k=1}^{n}\gamma_{k},
\qquad
\bar{\mu}_{\tau_{n}}\coloneqq
\frac{1}{\tau_{n}}\int_{0}^{\tau_{n}}\mu_{t}\,dt,
\]
where \(\mu_t\) denotes the law of the process \(\vx_t\) under the following continuous-time interpolation
\begin{equation}
\vx_{t}\coloneqq 
\vx_{\tau_{n}}
-(t-\tau_{n})\,\vg_n
+\sqrt{2}\,\bigl(\mB_{t}-\mB_{\tau_{n}}\bigr),
\qquad t\in[\tau_{n},\tau_{n+1}].
\end{equation}$\vg_n$ is defined as follows

\begin{equation}
\vg_{n} :=
\begin{cases}
\displaystyle
\frac{1}{b_{n}}\sum_{i=1}^{b_{n}} \nabla f(\vx_{\tau_n}, \xi_{n}^i)
& \text{w.p. } p_{n}, \\[1ex]
\displaystyle
\beta_n\vg_{n-1}
+ (1-\beta_n)\nabla f(\vx_{\tau_n}, \xi_n)
& \text{w.p. } 1-p_{n},
\end{cases}
\end{equation}for all $n \geq 1$. At initialization, we choose $\gamma_0, b_0>0$ and $(\beta_0,p_0)\in(0,1]^2 \setminus\{(1,1)\}$, and define $\vg_0$ as
\begin{equation*}
\vg_0 \coloneqq \frac{1}{b_0}\sum_{i=1}^{b_0}\nabla f(\vx_0, \xi_0^i).
\end{equation*}We establish weak convergence by adapting the proof of Theorem~\ref{thm:ml-vrld-ficonv}. To this end, we first verify the step sizes $(\gamma_k)_{k\ge 1}$ satisfy the step-size condition of Theorem~\ref{thm:ml-vrld-ficonv}, namely, 
\begin{equation}
    \gamma_k\in\left(0, \frac{1}{L\sqrt{27\phi(p_k,\beta_k)}}\right], \quad \text{where} \quad \phi(p_k, \beta_k) = 1 + \frac{4(1-p_k)\beta_k^2}{(1-(1-p_k)\beta_k)^2} \label{thm2:gamma_k:condition}
\end{equation} for all $k\ge 1$. Using the definitions of $p_k$ and $\beta_k$, we have
\begin{align}
    \phi(p_k,\beta_k) &= 1 + \frac{4(1-p_k)\beta_k^2}{(1-(1-p_k)\beta_k)^2}\notag \\
    & \leq 1 + \frac{4}{1-(1-p_k)^2} \notag \\
    & \leq 1 + \frac{4}{p_k} \notag \\ 
    & = 1 + \frac{4k^{1/2}}{C_p} \notag \\
    & \leq \frac{8k^{1/2}}{C_p}, \label{thm2:cp-bound:1}
\end{align}where we use the fact that $0<C_p<1$ to obtain~(\ref{thm2:cp-bound:1}). Then, we have
\begin{equation*}
    \frac{1}{L\sqrt{27\phi(p_k,\beta_k)}} \ge \frac{\sqrt{C_p}}{L\sqrt{216}k^{1/4}}.
\end{equation*}Choosing $C_\gamma=\frac{\sqrt{C_p}}{L\sqrt{216}}$, we have
\begin{equation*}
    \gamma_k = \frac{C_\gamma}{k^{3/2}} \leq \frac{\sqrt{C_p}}{L\sqrt{216}k^{1/4}} \le \frac{1}{L\sqrt{27\phi(p_k,\beta_k)}}
\end{equation*} for $k\ge 1$. This verifies~(\ref{thm2:gamma_k:condition}). In addition, by definition of $p_k$ and $\beta_k$, we have $p_k< 1$ and $\beta_k < 1$. Therefore, we can follow the same steps in Theorem~\ref{thm:ml-vrld-ficonv} (Appendix~\ref{proof:ml-vrld-ficonv}) up to~(\ref{eq:ml-vrld-kl-bound}) since the the same arguments remain valid for $t\in[\tau_{n-1}, \tau_n]$ with time-varying parameters. We obtain
\begin{align}
\frac{d}{dt}\mathrm{KL}(\mu_t \| \pi)
&\le
-\frac{1}{2}\,\mathrm{FI}(\mu_t \| \pi)
-\frac{1}{1-(1-p_n)\beta_n}
\left[4+\frac{27}{2}\gamma_n^2L^2\phi(p_n,\beta_n)\right]
\bigl(e_{n}^2-e_{n-1}^2\bigr) \notag\\
&\quad
+ \frac{9\,\sigma^2}{2\bigl(1-(1-p_n)\beta_n\bigr)}
\left[\frac{p_n}{b_n} +(1-p_n)(1-\beta_n)^2\right]
+ 9\gamma_n L^2\phi(p_n,\beta_n)\,d \label{thm2:diff-kl-ineq}
\end{align}for $t\in[\tau_{n-1}, \tau_n]$. Integrating both sides of the inequality over the interval \([\tau_{n-1}, \tau_n]\) yields
\begin{align*}
\nonumber
\mathrm{KL}(\mu_{\tau_n} \| \pi) - \mathrm{KL}(\mu_{\tau_{n-1}} \| \pi)
&\le
-\frac{1}{2}\int_{\tau_{n-1}}^{\tau_n}\mathrm{FI}(\mu_t \| \pi)\,dt
-\frac{\gamma_n}{1-(1-p_n)\beta_n}
\left[4+\frac{27}{2}\gamma_n^2L^2\phi(p_n,\beta_n)\right]
\bigl(e_{n}^2-e_{n-1}^2\bigr) \\
&\quad
+ \frac{9\gamma_n\,\sigma^2}{2\bigl(1-(1-p_n)\beta_n\bigr)}
\left[p_n^2 +(1-p_n)(1-\beta_n)^2\right]
+ 9\gamma_n^2 L^2\phi(p_n,\beta_n)d,
\end{align*}where we use \(b_n \ge \frac{1}{p_n}\). Iterating the above bound, we obtain
\begin{align}
\nonumber
\mathrm{KL}(\mu_{\tau_n} \| \pi) - \mathrm{KL}(\mu_{0} \| \pi)
&\le
- \frac{1}{2}\int_{0}^{\tau_n}\mathrm{FI}(\mu_t \| \pi)\,dt
\underbrace{-\sum_{k=1}^n\frac{\gamma_k}{1-(1-p_k)\beta_k}
\left[4+\frac{27}{2}\gamma_k^2L^2\phi(p_k,\beta_k)\right]
\bigl(e_{k}^2-e_{k-1}^2\bigr)}_{\text{(I)}} \\ \label{eq:ml-vrld-id-kl-bound}
&\quad
+\underbrace{ \frac{9\sigma^2}{2} \sum_{k=1}^n
\gamma_k\left(\frac{p_k^2 +(1-p_k)(1-\beta_k)^2}{1-(1-p_k)\beta_k}\right)}_{\text{(II)}}
+\underbrace{9L^2d \sum_{k=1}^n\gamma_k^2\phi(p_k,\beta_k)}_{\text{(III)}}.
\end{align} 
It remains to control the terms labeled \textup{(I)}--\textup{(III)} and show that their contribution remains uniformly bounded as \(k\to\infty\). This will imply that the upper bound (\ref{eq:ml-vrld-id-kl-bound}) stays bounded, which is a key step toward establishing weak convergence. 

First, analyze the term (I). Define 
\begin{align*}
    c_k \coloneqq \frac{\gamma_k}{1-(1-p_k)\beta_k}\left(4 + \frac{27}{2}\gamma_k^2L^2\phi(p_k, \beta_k)\right).
\end{align*}
If $(c_k)_{k\ge 1}$ is nonnegative and decreasing sequence, we can bound $\mathrm{(I)}$ as follows
\begin{align*}
(\textup{I}) =
    -\sum_{k=1}^nc_k(e_k^2-e_{k-1}^2) = c_1e_0^2 + \sum_{k=1}^{n-1}(c_{k+1}-c_{k})e_{k}^2 - c_ne_n^2 \leq c_1e_0^2.
\end{align*} Thus, it remains to show that $(c_k)_{k\ge 1}$ is nonnegative and decreasing sequence. Nonnegativity is obvious as all the terms are positive and $p_k <1$ and $\beta_k < 1$. To prove that the sequence is decreasing, we substitute $\gamma_k$, $\beta_k$, and $p_k$. We then have
\begin{align*}
    \frac{\gamma_k}{1-(1-p_k)\beta_k} & = \frac{\gamma_k}{1-\beta_k^2} \\
    & = \frac{C_\gamma}{2C_pk - C_p^2k^{1/2}},
\end{align*}which is decreasing for $k\ge 1$. In addition,
\begin{align*}
    \frac{27}{2}\gamma_k^2 L^2 \phi(p_k,\beta_k) = \frac{27C_\gamma^2L^2}{2k^3}\left(1+ \frac{4\left(1 - \frac{C_p}{k^{1/2}}\right)^3}{\left(\frac{2C_p}{k^{1/2}} - \frac{C_p^2}{k}\right)^2}\right).
\end{align*} We can write this sequence as a continuous function
\begin{equation*}
    c(x) = \frac{27C_\gamma^2L^2}{2x^3}\left(1+ \frac{4\left(1 - \frac{C_p}{x^{1/2}}\right)^3}{\left(\frac{2C_p}{x^{1/2}} - \frac{C_p^2}{x}\right)^2}\right)
\end{equation*}for $x\ge 1$. We apply change of variables $r=\frac{C_p}{x^{1/2}}$ and get
\begin{equation}
    c(x) = \frac{27C_\gamma^2L^2}{2C_p^6}\underbrace{r^6\left(1 + \frac{4(1-r)^3}{r^2(2 - r)^2}\right)}_{\eqqcolon h(r)}.
\end{equation}Differentiating $h(r)$ with respect to $r$, we have
\begin{align*}
    h'(r) & = \frac{2r^3}{(2-r)^3}\left[16 - 64r+114r^2-88r^3 + 28r^4 - 3r^5\right]\\ 
    & = \frac{2r^3}{(2-r)^3}\left[16(1-r)^5 + 16r(1-r)^4 + 18r^2(1-r)^3 + 30r^3(1-r)^2 + 30r^3(1-r)^2 + 18r^4(1-r) + 3r^5\right] \\ 
    & > 0
\end{align*}since $0<r<1$. Hence, differentiating $h(r(x))$ with respect to $x$, we get
\begin{align*}
c'(x) & = \frac{27C_\gamma^2L^2}{2C_p^6}h'(r(x))r'(x) \\
& = -\frac{27C_\gamma^2L^2}{4C_p^5x^{3/2}}h'(r(x)) < 0
\end{align*}for $x\ge 1$. Therefore, $c(x)$ is decreasing on $[1,\infty)$, which implies $(c_k)_{k\ge 1}$ is decreasing sequence. This implies
\begin{equation}
    \limsup_{n\to\infty} (\mathrm{I}) \le c_1 e_0^2.\label{thm2:decreasing-cn-e0}
\end{equation}

We now analyze the term~\textup{(II)}
\begin{align*}
\frac{9\sigma^2}{2} \sum_{k=1}^n
\gamma_k\left(\frac{p_k^2 +(1-p_k)(1-\beta_k)^2}{1-(1-p_k)\beta_k}\right) &\leq 
\frac{9\sigma^2}{2} \sum_{k=1}^n
\gamma_k\left(\frac{p_k^2 +(1-p_k)(1-\beta_k)^2}{p_k}\right) \\ &\leq 
\frac{9\sigma^2}{2} \sum_{k=1}^n
\gamma_k\left(\frac{p_k^2 +(1-\beta_k)^2}{p_k}\right) \\
&= \frac{9\sigma^2}{2} \sum_{k=1}^n
\gamma_k\left(\frac{p_k^2 + p_k^2}{p_k}\right) \\
&=9\sigma^2 \sum_{k=1}^n
\gamma_kp_k.
\end{align*}
Using the definitions of $\gamma_k$ and $p_k$, we obtain the following upper bound for (II):
\begin{align*}
    \frac{9\sigma^2}{2} \sum_{k=1}^n
\gamma_k\left(\frac{p_k^2 +(1-p_k)(1-\beta_k)^2}{1-(1-p_k)\beta_k}\right) &\leq 
9C_{\gamma}C_p\sigma^2\sum_{k=1}^{n}\frac{1}{k^2}.
\end{align*}
Taking the limit as $n \rightarrow \infty$ yields
\begin{align}
    \lim_{n\rightarrow\infty}{(\textup{II})} \leq 9C_{\gamma}C_p\sigma^2\sum_{k=1}^{\infty}\frac{1}{k^2} \eqqcolon S_{2}
< \infty.\label{proof:ml-vrld-termII}
\end{align}
Finally, we analyze (III)
\begin{align*}
9L^2d \sum_{k=1}^n\gamma_k^2\left(1 + \frac{4(1-p_k)\beta_k^2}{(1-(1-p_k)\beta_k)^2}\right) \leq 72L^2d \sum_{k=1}^n\frac{\gamma_k^2}{p_k^2} 
= \frac{72C_\gamma^2L^2d}{C_p^2}\sum_{k=1}^{n}\frac{1}{k^2}.
\end{align*}
Taking the limit as $n \rightarrow \infty$ yields
\begin{align}
    \lim_{n\rightarrow\infty}{(\textup{III})} \leq \frac{36C_\gamma^2L^2d}{C_p^2}\sum_{n=1}^{\infty}\frac{1}{n^2} \eqqcolon S_{3}
< \infty.\label{proof:ml-vrld-termIII}
\end{align}
Now that we have shown the terms (I), (II) and (III) are uniformly bounded, we bound (\ref{eq:ml-vrld-id-kl-bound}) as
\begin{align}
\mathrm{KL}(\mu_{\tau_n} \| \pi) - \mathrm{KL}(\mu_{0} \| \pi)
&\le
- \frac{1}{2}\int_{0}^{\tau_n}\mathrm{FI}(\mu_t \| \pi)\,dt
+ c_1e_0^2 +S_{2} +S_{3}. \label{proof:weak-confergence-first-kl-bound}
\end{align} 
Since KL divergence is nonnegative, we may discard $\mathrm{KL}(\mu_{\tau_n}\|\pi)$ and rearrange the terms to get
\begin{align*}
\frac{1}{2}\int_{0}^{\tau_n}\mathrm{FI}(\mu_t \| \pi)\,dt
&\le
\mathrm{KL}(\mu_0 \| \pi)
+ c_1e_0^2 +S_{2} +S_{3}.
\end{align*} 
Dividing both sides by $\frac{\tau_n}{2}$, we obtain
\begin{align*}
\frac{1}{\tau_n}\int_0^{\tau_n}\mathrm{FI}(\mu_t \| \pi)\,dt
&\le
\frac{2\mathrm{KL}(\mu_0 \| \pi)}{\tau_n}
+\frac{2c_1e_0^2}{\tau_n} +\frac{2S_{2}}{\tau_n} +\frac{2S_{3}}{\tau_n}.
\end{align*} 
Since $\mathrm{FI}(\cdot\|\pi)$ is convex, we use Jensen's inequality to obtain
\begin{align}\label{eq:ml-vrld-id-avg-law-fi}
\mathrm{FI}(\bar \mu_{\tau_n} \| \pi) \leq
\frac{1}{\tau_n}\int_{0}^{\tau_n}\mathrm{FI}(\mu_t \| \pi)\,dt
&\le
\frac{2\mathrm{KL}(\mu_0 \| \pi)}{\tau_n}
+\frac{2c_1e_0^2}{\tau_n} +\frac{2S_{2}}{\tau_n} +\frac{2S_{3}}{\tau_n},
\end{align} 
where $\mu_{\tau_n} \coloneqq \frac{1}{\tau_n}\int_0^{\tau_n}\mu_t\,dt$. Given that $\mathrm{KL}(\mu_0\|\pi)< \infty$, we have $\mathrm{FI}(\bar{\mu}_{\tau_n}\|\pi)<\infty$. On the other hand, if $t\in[\tau_n, \tau_{n+1}]$, considering~(\ref{thm2:diff-kl-ineq}) for the interval $[\tau_n, \tau_{n+1}]$, integrating both sides from $\tau_n$ to $t$, and dropping the negative integral over the Fisher information gives us
\begin{align*}
\mathrm{KL}(\mu_t \| \pi)
&\leq\,
\mathrm{KL}(\mu_{\tau_n} \| \pi)
-
\frac{(t-\tau_n)}{1-(1-p_{n+1})\beta_{n+1}}\left[4 + \frac{27}{2}\gamma_{n+1}^2L^2\phi(p_{n+1}, \beta_{n+1})\right](e_{n+1}^2 - e_n^2) \\
&\quad + \frac{9\,\sigma^2(t-\tau_n)}{2\bigl(1-(1-p_{n+1})\beta_{n+1}\bigr)}
\left[\frac{p_{n+1}}{b_{n+1}} +(1-p_{n+1})(1-\beta_{n+1})^2\right]
+ 9(t-\tau_n)\gamma_{n+1}L^2\phi(p_{n+1},\beta_{n+1})\,d \\
&\leq
\mathrm{KL}(\mu_{0} \| \pi) + \frac{9\,\sigma^2(t-\tau_n)}{2\bigl(1-(1-p_{n+1})\beta_{n+1}\bigr)}
\left[\frac{p_{n+1}}{b_{n+1}} +(1-p_{n+1})(1-\beta_{n+1})^2\right]
\\
& \quad + 9(t-\tau_{n+1})\gamma_{n+1} L^2\phi(p_{n+1},\beta_{n+1})d + c_1e_0^2 + S_{2} + S_{3} \\
&\leq
\mathrm{KL}(\mu_{0} \| \pi)  + c_1e_0^2 + 2S_{2} + 2S_{3} < \infty,
\end{align*}where the second inequality follows from~(\ref{eq:ml-vrld-id-kl-bound}), the bounds on $\mathrm{(II)}$ and $\mathrm{(III)}$, and the fact that 
\begin{equation}
    c_{n+1} \ge \frac{(t-\tau_n)}{1-(1-p_{n+1})\beta_{n+1}}\left[4 + \frac{27}{2}\gamma_{n+1}^2L^2\phi(p_{n+1},\beta_{n+1})\right]
\end{equation}for $t\in[\tau_n,\tau_{n+1}]$. To obtain the last inequality, we again use the bounds on $\mathrm{(II)}$ and $\mathrm{(III)}$. 

This proves $\mathrm{KL}(\mu_t \| \pi)$ is uniformly bounded for all $t \ge 0$. By the convexity of the $\mathrm{KL}$ divergence, it implies that $\{\mathrm{KL(\bar\mu_{\tau_n} \| \pi)}\,|\,n\in\mathbb{N}\}$ is uniformly bounded as well. Since the sub-level sets of $\mathrm{KL}(\cdot \|\pi)$ are weakly compact, $(\bar \mu_{\tau_n})_{n\in\mathbb{N}}$ is tight. To show that $\bar \mu_{\tau_n} \rightarrow \pi$ weakly, it is sufficient to show that every cluster point of $(\bar \mu_{\tau_n})_{n\in\mathbb{N}}$ is equal to $\pi$. 

Consider a subsequence of $(\bar{\mu}_{\tau_n})_{n\in\sN}$ converging to some limit $\bar{\mu}$. As $n\rightarrow\infty$, $\tau_n \rightarrow \infty$, and by inequality~(\ref{eq:ml-vrld-id-avg-law-fi}) we also have $\mathrm{FI}(\bar \mu_{\tau_n} \| \pi) \rightarrow0$. Therefore, this is still true along the subsequence. By the weak lower semi-continuity of the Fisher information along the subsequence, $\mathrm{FI}(\bar \mu_{\tau_n} \| \pi)=0$. That implies; for $\psi\coloneqq \frac{d\mu}{d\pi}$, we have $\sqrt{\psi} \in \mathrm{dom(\mathcal{E})}$ and $\mathcal{E}(\sqrt{\psi}) = 0$, where $\mathcal{E}$ denotes the Dirichlet energy (i.e. the squared $L^2(\pi)$-norm of the gradient; see Section 3 in~\citet{balasubramanian2022towards}). Since $\nabla f$ is Lipschitz by Assumption \ref{ass:potential-lipschitz}, then $\pi$ has a continuous and strictly positive density on $\mathbb{R}^d$, so $\mathcal{E}(\sqrt{\psi}) = 0$ implies that $\psi$ is a constant $\pi$ with probability 1. Hence, $\bar \mu = \pi$ as $n\rightarrow\infty$.
\qedbox 

\subsubsection{Proof of Theorem~\ref{thm:sl-vrld-ficonv} (SL-VRLD: FI convergence)}\label{proof:sl-vrld-ficonv}
\paragraph{Theorem~\ref{thm:sl-vrld-ficonv} (SL-VRLD: FI convergence)}\textit{
Let $\pi\propto e^{-f}$ be the target distribution, where the potential function $f$ satisfies Assumptions~\ref{ass:sfo} and~\ref{ass:stochastic-potential-lipschitz}. Let $(\mu_t)_{t\geq 0}$ denote the law of continuous-time interpolation generated by~(\ref{def:ml-vrld:interpolation}) with estimator~(\ref{def:sl-vrld:estimator}). Then, for any step size $\gamma \in \left(0, \frac{1}{L\sqrt{27\phi(p,\beta)}}\right]$, where $\phi(p,\beta)= 1+\frac{2(1-p)}{1-(1-p)\beta}$, and for any $N\geq 1$, it holds that
\begin{equation}
    \frac{1}{N\gamma}\int_{0}^{N\gamma} \mathrm{FI}(\mu_t || \pi)\, dt\leq \frac{2C_0}{N\gamma} + 9\sigma^2\left[\frac{\frac{p}{b} + 2(1-p)(1-\beta)^2}{2(1-(1-p)\beta)}\right] + \frac{30\gamma L^2 d}{1-(1-p)\beta},
\end{equation}where 
\begin{equation}
    C_0 = \mathrm{KL}(\mu_0\|\pi) + \frac{9\gamma}{4(1-(1-p)\beta)}\E[\|\nabla f(\vx_0) - \vg_0\|^2].
\end{equation}Furthermore, choose
\[
\gamma=\left(\frac{d}{\sigma_m^2 N^2 L^2}\right)^{1/3},\qquad 
p=1-\beta=\left(\frac{d^2L^2}{\sigma_m^4 N}\right)^{1/3},\qquad 
b=\Bigl\lceil \tfrac{1}{p}\Bigr\rceil,
\quad \text{where } \sigma_m = \max\{\sigma,1\}.
\]
Then, the time-averaged law $\bar{\mu}_{N\gamma}=\frac{1}{N\gamma}\int_{0}^{N\gamma}\mu_t\,dt$ satisfies $\mathrm{FI}(\bar{\mu}_{N\gamma}\|\pi)\le \varepsilon$ with~$\BigO\!\left(\frac{\sigma_m^{2}L^{2}d^{2}}{\varepsilon^{3}}\right)$ iterations and $\BigO(1)$ gradient computations per iteration.}

\textit{Proof.} As in the proof of Theorem \ref{thm:ml-vrld-ficonv}, we start by invoking Lemma \ref{lemma:sampling_first}
\begin{align}
\nonumber
\frac{d}{dt}\mathrm{KL}(\mu_t \| \pi)
&\leq
-\frac{3}{4}\mathrm{FI}(\mu_t \| \pi)
+
\E \left[\|\nabla f(\vx_t) - \vg_k\|^2\right] \\ \nonumber
&\leq
-\frac{3}{4}\mathrm{FI}(\mu_t \| \pi) +
2\E \left[\|\nabla f(\vx_t) - \nabla f(\vx_{k\gamma})\|^2\right] +
2\E \left[\|\nabla f(\vx_{k\gamma}) - \vg_k\|^2\right] \\ \label{eq:sl-main-ineq}
&\leq
-\frac{3}{4}\mathrm{FI}(\mu_t \| \pi) +
2L^2 \E\left[\|\vx_t - \vx_{k\gamma}\|^2\right] +
2e_k^2.
\end{align}
where $e_k^2 \coloneq \E\!\left[\|\nabla f(\vx_{k\gamma}) - \vg_k\|^2\right]$. We decompose the expectation by adding and subtracting $\nabla f(\vx_{k\gamma})$, apply Young's inequality, and use Assumption \ref{ass:potential-lipschitz} to obtain~(\ref{eq:sl-main-ineq}). The next step is to bound $e_k^2$. Using the definition of the SL-VRLD estimator (\ref{def:sl-vrld:estimator}),
$e_{k+1}^2$ can be expressed as follows.
\begin{align}
    e_{k+1}^2 
    &=
    p\,\E\!\left[
    \left\|
    \nabla f(\vx_{(k+1)\gamma})
    -\frac{1}{b}\sum_{i=1}^b \nabla f(\vx_{(k+1)\gamma}, \xi_{k+1}^i)
    \right\|^2
    \right] \notag\\
    &\quad
    + (1-p)\,\E\!\left[
    \left\|
    \nabla f(\vx_{(k+1)\gamma})
    -\nabla f( \vx_{(k+1)\gamma}, \xi_{k+1} ) -
            \beta 
                \left(
                    \vg_k - \nabla f( \vx_{k\gamma},\xi_{k+1} )
                \right)
    \right\|^2
    \right]\notag \\
    & \le \frac{p\sigma^2}{b} + (1-p)\,\E\!\left[
    \left\|
    \nabla f(\vx_{(k+1)\gamma})
    -\nabla f( \vx_{(k+1)\gamma}, \xi_{k+1} ) -
            \beta 
                \left(
                    \vg_k - \nabla f( \vx_{k\gamma},\xi_{k+1} )
                \right)
    \right\|^2\right] \label{thm3:sl-vrld-err-step1}\\
    & = \frac{p\sigma^2}{b} + (1-p)\,\E\!\biggl[
    \bigl\|
    \nabla f(\vx_{(k+1)\gamma})
    -\nabla f( \vx_{(k+1)\gamma}, \xi_{k+1} ) -
            \beta 
                \left(
                    \vg_k - \nabla f( \vx_{k\gamma})
                \right) - \beta\bigl(\nabla f(\vx_{k\gamma})\label{thm3:sl-vrld-err-step2}\\
    &\quad\quad\quad\quad\quad\quad\quad\quad\quad\quad\quad\quad\quad\quad\quad\quad\quad\quad\quad\quad\quad\quad\quad\quad\quad\quad\quad\quad\quad\quad\quad\quad\quad\quad\quad\quad - \nabla f(\vx_{k\gamma}, \xi_{k+1})\bigr)\bigr\|^2\biggr] \notag\\
    &= \frac{p\sigma^2}{b} + (1-p)\,\E\!\biggl[
    \bigl\|
    \underbrace{\nabla f(\vx_{(k+1)\gamma}) - \nabla f(\vx_{k\gamma})
    -\left(\nabla f( \vx_{(k+1)\gamma}, \xi_{k+1} ) -\nabla f(\vx_{k\gamma}, \xi_{k+1})\right)}_{\eqqcolon\mA_k} -
            \beta 
                \left(
                    \vg_k - \nabla f( \vx_{k\gamma})
                \right)\notag \\
    & \quad\quad\quad\quad\quad\quad\quad\quad\quad\quad\quad\quad\quad\quad\quad\quad\quad\quad\quad\quad\quad\quad\quad\quad + (1-\beta)\underbrace{\left(\nabla f(\vx_{k\gamma}) - \nabla f(\vx_{k\gamma}, \xi_{k+1})\right)}_{\eqqcolon\mB_k}\bigr\|^2\biggr] \label{thm3:sl-vrld-err-step3}\\
    & = \frac{p\sigma^2}{b} + (1-p)\,\E\left[\|\mA_k + (1-\beta)\mB_k - \beta(\vg_k - \nabla f(\vx_{k\gamma}))\|^2\right]\notag \\
    & = \frac{p\sigma^2}{b} +(1-p)\beta^2 e_k^2 + (1-p)\,\E\left[\|\mA_k + (1-\beta)\mB_k\|^2\right]\label{thm3:sl-vrld-err-step4}\\
    & \le \frac{p\sigma^2}{b} +(1-p)\beta^2 e_k^2 + 2(1-p)\,\E\left[\|\mA_k\|^2\right] + 2(1-p)(1-\beta)^2\,\E\left[\|\mB_k\|^2\right] \label{thm3:sl-vrld-err-step5}\\
    & \le \frac{p\sigma^2}{b} +(1-p)\beta e_k^2 + 2(1-p)L^2\Delta_k + 2(1-p)(1-\beta)^2\sigma^2, \label{thm3:sl-vrld-err-step6}
\end{align}where $\Delta_k\coloneqq \E[\|\vx_{(k+1)\gamma} - \vx_{k\gamma}\|^2]$. We first use Assumption~\ref{ass:sfo} to bound the first term to get~(\ref{thm3:sl-vrld-err-step1}). Next, we add and subtract $\beta \nabla f(\vx_{k\gamma})$ inside the expected squared norm to obtain~(\ref{thm3:sl-vrld-err-step2}). We then apply the decomposition $-\beta a = (1-\beta)a - a$ for any $a\in\sR$ and $\beta\in[0,1]$ inside the expected squared norm to derive~(\ref{thm3:sl-vrld-err-step3}). Defining $\mA_k$ and $\mB_k$, and using the fact that $\E[\mA_k - (1-\beta)\mB_k \mid \gF_k] = 0$, we obtain~(\ref{thm3:sl-vrld-err-step4}). Here, $\gF_k$ denotes the filtration defined as
\[
    \gF_k \coloneqq \sigma(\vx_0, \mZ_0, \xi_0, p_0, \ldots, \mZ_k, \xi_k, p_k),
\]
where $\xi_k$ denotes all $\xi_k^i$ with $i\in\{1,\ldots,b\}$ if the batch is calculated, and $\mZ_k \coloneqq \mB_{k+1} - \mB_k$. Following this, we use Young's inequality $\left((a+b)^2 \leq 2a^2 + 2b^2\right)$ to obtain~(\ref{thm3:sl-vrld-err-step5}). Lastly, we use the fact that $\beta \leq 1$ together with the Assumptions~\ref{ass:sfo} and~\ref{ass:stochastic-potential-lipschitz}, and the second-moment variance bound, to get~(\ref{thm3:sl-vrld-err-step6}).
Rearranging the terms yields
\begin{align}
    \nonumber
    e_k^2 &\leq
    \frac{p\sigma^2}{b(1-(1-p)\beta)} -
    \frac{1}{1-(1-p)\beta}(e_{k+1}^2-e_k^2) +
    \frac{2(1-p)L^2\Delta_k}{1-(1-p)\beta} +
    \frac{2(1-p)(1-\beta)^2\sigma^2}{1-(1-p)\beta}\\
    \label{eq:sl-vrld-ek-upper-bound}
    &=-\frac{1}{1-(1-p)\beta}(e_{k+1}^2-e_k^2) + 
    \frac{2(1-p)L^2\Delta_k}{1-(1-p)\beta} +
    \left[\frac{p}{b}+2(1-p)(1-\beta)^2\right]\frac{\sigma^2}{1-(1-p)\beta}.
\end{align}
We next bound $\Delta_k$ by exploiting the interpolation in~(\ref{def:ml-vrld:interpolation}).
For any $t \in [k\gamma,(k+1)\gamma]$, the interpolation satisfies
\begin{align*}
    \vx_t - \vx_{k\gamma}
    &= -(t-k\gamma)\vg_k + \sqrt{2}\bigl(\mB_t - \mB_{k\gamma}\bigr).
\end{align*}
Taking squared norms and expectations, and using the independence of
$\vg_k$ and the Brownian increment, we obtain
\begin{align*}
    \E\!\left[\|\vx_t - \vx_{k\gamma}\|^2\right]
    &= (t-k\gamma)^2 \E\!\left[\|\vg_k\|^2\right]
    + 2\,\E\!\left[\|\mB_t - \mB_{k\gamma}\|^2\right] \\
    &= (t-k\gamma)^2 \E\!\left[\|\vg_k\|^2\right]
    + 2(t-k\gamma)d,
\end{align*}
where we use the fact that $\E\!\left[\|\mB_t-\mB_{k\gamma}\|^2\right]=(t-k\gamma)d$.
Since $t\in[k\gamma,(k+1)\gamma]$, it follows that
\begin{align*}
    \E\!\left[\|\vx_t - \vx_{k\gamma}\|^2\right]
    &\le \gamma^2 \E\!\left[\|\vg_k\|^2\right] + 2\gamma d
    = \Delta_k.
\end{align*}
Then writing
\(
\vg_k
=
\left[\vg_k-\nabla f(\vx_{k\gamma})\right]
+\left[\nabla f(\vx_{k\gamma})-\nabla f(\vx_t)\right]
+ \nabla f(\vx_t),
\) and using Assumption~\ref{ass:sfo} with Young's inequality, we obtain
\begin{align*}
\Delta_k
&= \gamma^2\E\!\left[\|\vg_k\|^2\right] + 2\gamma d \\
&\le
3\gamma^2\E\!\left[\|\vg_k-\nabla f(\vx_{k\gamma})\|^2\right]
+ 3\gamma^2\E\!\left[\|\nabla f(\vx_{k\gamma})-\nabla f(\vx_t)\|^2\right]
+ 3\gamma^2\E\!\left[\|\nabla f(\vx_t)\|^2\right]
+ 2\gamma d .\\
&\leq3\gamma^2 e_k^2
+ 3\gamma^2 L^2 \Delta_k
+ 3\gamma^2\E\!\left[\|\nabla f(\vx_t)\|^2\right]
+ 2\gamma d .
\end{align*}
We now substitute the upper bound~(\ref{eq:sl-vrld-ek-upper-bound}) on $e_k^2$ to obtain
\begin{align*}
\Delta_k
&\le
-\frac{3\gamma^2}{1-(1-p)\beta}\,
\bigl(e_{k+1}^2 - e_k^2\bigr)
+ 3\gamma^2 L^2
\left(
1 + \frac{2(1-p)}{1-(1-p)\beta}
\right)\Delta_k
\\
&\quad
+ \frac{3\gamma^2\sigma^2}{1-(1-p)\beta}
\left[
\frac{p}{b} + 2(1-p)(1-\beta)^2
\right]
+ 3\gamma^2\,\E\!\left[\|\nabla f(\vx_t)\|^2\right]
+ 2\gamma d .
\end{align*}
Defining $\phi(p, \beta) \coloneqq 1 + \frac{2(1-p)}{1-(1-p)\beta}$ and rearranging the terms, we obtain
\begin{align}
    \left(1-3\gamma^2L^2\phi(p,\beta)\right)\Delta_k \leq & -\frac{3\gamma^2}{1-(1-p)\beta}(e_{k+1}^2 - e_k^2) + \frac{3\gamma^2\sigma^2}{1-(1-p)\beta}\left[\frac{p}{b} + 2(1-p)(1-\beta)^2\right] \notag\\
    & + 3\gamma^2 \E[\|\nabla f(\vx_t)\|^2] + 2\gamma d.  \notag
\end{align}Setting $\gamma$ such that $\gamma \leq \frac{1}{L\sqrt{27\phi(p, \beta)}}$ and dividing both sides by $\frac{8}{9}$, we have
\begin{align} \label{eq:sl-delta-k-bound}
\Delta_k
&\le
-\frac{27\gamma^2}{8\bigl(1-(1-p)\beta\bigr)}
\bigl(e_{k+1}^2 - e_k^2\bigr)
+ \frac{27\gamma^2\sigma^2}{8\bigl(1-(1-p)\beta\bigr)}
\left[
\frac{p}{b} + 2(1-p)(1-\beta)^2
\right]
+ \frac{27\gamma^2}{8}\,
\E\!\left[\|\nabla f(\vx_t)\|^2\right]
+ \frac{9}{4}\,\gamma d .
\end{align} Given the upper bounds (\ref{eq:sl-vrld-ek-upper-bound} , \ref{eq:sl-delta-k-bound}) on $e_k^2$ and $\Delta_k$, respectively, we return back to the main inequality (\ref{eq:sl-main-ineq}). Recall that
\begin{align*}
\frac{d}{dt}\mathrm{KL}(\mu_t \| \pi)
&\leq
-\frac{3}{4}\mathrm{FI}(\mu_t \| \pi) +
2L^2\Delta_k +
2e_k^2.
\end{align*} Plugging in the bounds (\ref{eq:sl-vrld-ek-upper-bound} , \ref{eq:sl-delta-k-bound}) for $e_k^2$ and $\Delta_k$ yields
\begin{align*}
\frac{d}{dt}\mathrm{KL}(\mu_t \| \pi)
&\le
-\frac{3}{4}\,\mathrm{FI}(\mu_t \| \pi) 
-\frac{1}{1-(1-p)\beta}
\left(
2 + \frac{27}{4}\gamma^2 L^2 \phi(p,\beta)
\right)
\bigl(e_{k+1}^2 - e_k^2\bigr) \\
&\quad
+\frac{\sigma^2}{1-(1-p)\beta}
\left(
2 + \frac{27}{4}\gamma^2 L^2 \phi(p,\beta)
\right)
\left[
\frac{p}{b} + 2(1-p)(1-\beta)^2
\right] \\
&\quad
+ \frac{27}{4}\gamma^2 L^2 \phi(p,\beta)\,
\E\!\left[\|\nabla f(\vx_t)\|^2\right]
+ \frac{9}{2}\gamma L^2 \phi(p,\beta)d.
\end{align*}
Applying Lemma~\ref{lemma:sampling_final} with the step-size condition $\gamma\leq \frac{1}{L\sqrt{27\phi(p,\beta)}}$, we obtain
\begin{align*}
\nonumber
\frac{d}{dt}\mathrm{KL}(\mu_t \| \pi)
&\le
-\frac{1}{2}\,\mathrm{FI}(\mu_t \| \pi)
-\frac{1}{1-(1-p)\beta}
\left(
2 + \frac{27}{4}\gamma^2 L^2 \phi(p,\beta)
\right)
\bigl(e_{k+1}^2 - e_k^2\bigr) \\[0.4em]
&\quad
+ \frac{9\sigma^2}{4\bigl(1-(1-p)\beta\bigr)}
\left[
\frac{p}{b} + 2(1-p)(1-\beta)^2
\right]
+ \left[\frac{27}{2}\gamma L + \frac{9}{2}\right]\gamma L^2 d\,\phi(p,\beta).
\end{align*}To simplify the last term, we note that since $\phi(p,\beta)\ge 1$, $\gamma L\le \frac{1}{\sqrt{27}}$, which implies 
\begin{equation*}
    \frac{27}{2}\gamma L + \frac{9}{2}\le \frac{15}{2}. 
\end{equation*}Then, we have 
\begin{align}\label{eq:sl-vrld-kl-bound-for-id}
\nonumber
\frac{d}{dt}\mathrm{KL}(\mu_t \| \pi)
&\le
-\frac{1}{2}\,\mathrm{FI}(\mu_t \| \pi)
-\frac{1}{1-(1-p)\beta}
\left(
2 + \frac{27}{4}\gamma^2 L^2 \phi(p,\beta)
\right)
\bigl(e_{k+1}^2 - e_k^2\bigr) \\[0.4em]
&\quad
+ \frac{9\sigma^2}{4\bigl(1-(1-p)\beta\bigr)}
\left[
\frac{p}{b} + 2(1-p)(1-\beta)^2
\right]
+ \frac{15}{2}\gamma L^2 d\,\phi(p,\beta).
\end{align}

We now integrate both sides over $t \in [k\gamma,(k+1)\gamma]$
and define
\(
\mathcal{L}_k \coloneqq \mathrm{KL}(\mu_{k\gamma}\|\pi)
+ 
\frac{\gamma}{1-(1-p)\beta}\left[2 + \frac{27}{4}\gamma^2L^2\phi(p, \beta)\right]e_k^2 .
\)
This yields,
\begin{align*}
\mathcal{L}_{k+1} - \mathcal{L}_k
&\le
-\frac{1}{2}
\int_{k\gamma}^{(k+1)\gamma} \mathrm{FI}(\mu_t \| \pi)\,dt
+ 
\frac{9\,\gamma\sigma^2}{4\bigl(1-(1-p)\beta\bigr)}
\left[\frac{p}{b} + 2(1-p)(1-\beta)^2\right]
+ \frac{15}{2}\gamma^2 L^2 \phi(p,\beta)\, d .
\end{align*}
Iterating the inequality for $k=0,1,\dots,N-1$ we obtain
\begin{align*}
\mathcal{L}_{N} - \mathcal{L}_0
&\le
-\frac{1}{2}
\int_{0}^{N\gamma} \mathrm{FI}(\mu_t \| \pi)\,dt
+ 
\frac{9N\,\gamma\sigma^2}{4\bigl(1-(1-p)\beta\bigr)}
\left[\frac{p}{b} + 2(1-p)(1-\beta)^2\right]
+ \frac{15}{2}N\gamma^2 L^2 \phi(p,\beta)\, d .
\end{align*} Rearranging the terms and multiplying both sides by $\frac{2}{N\gamma}$, we get
\begin{align*}
\frac{1}{N\gamma}\int_{0}^{N\gamma}\mathrm{FI}(\mu_t \| \pi)\,dt
&\le
\frac{2\mathcal{L}_0}{N\gamma}
+ 9\sigma^2\,
\left[\frac{\frac{p}{b} + 2(1-p)(1-\beta)^2}{2(1-(1-p)\beta)}\right]
+ 15\gamma L^2 d
\left[1 + \frac{2(1-p)}{1-(1-p)\beta}\right].
\end{align*}
Note that $\mathcal{L}_N>0$ and $\mathcal{L}_0 = 
\mathrm{KL}(\mu_0 \| \pi) + 
\frac{\gamma}{1-(1-p)\beta}\left(2 + \frac{27}{4}\gamma^2L^2\phi(p, \beta)\right)e_0^2
\leq 
\mathrm{KL}(\mu_0 \| \pi) + \frac{9\gamma e_0^2}{4(1-(1-p)\beta)}\eqqcolon C_0$, then
\begin{align}\label{eq:integral-bound-sl}
\frac{1}{N\gamma}\int_{0}^{N\gamma}\mathrm{FI}(\mu_t \| \pi)\,dt
&\le
\frac{2C_0}{N\gamma}
+ 9\sigma^2\,
\left[\frac{\frac{p}{b} + 2(1-p)(1-\beta)^2}{2(1-(1-p)\beta)}\right]
+ \frac{30\gamma L^2d}{1-(1-p)\beta}
\end{align} where we additionally used the fact that $1 + \frac{2(1-p)}{1-(1-p)\beta} = \frac{1+(1-p)\beta}{1-(1-p)\beta} \leq \frac{2}{1-(1-p)\beta}$ to rewrite the last term. This completes the first part of the proof. Furthermore, choose 
\[
    \gamma=\left(\frac{d}{\sigma_m^2 N^2 L^2}\right)^{1/3},\qquad 
    p=1-\beta=\left(\frac{d^2L^2}{\sigma_m^4 N}\right)^{1/3},\qquad 
    b=\Bigl\lceil \tfrac{1}{p}\Bigr\rceil,
    \quad \text{where } \sigma_m \coloneqq \max\{\sigma,1\}.
\]
and define the time-averaged law $\bar{\mu}_{N\gamma} \coloneqq \frac{1}{N\gamma}\int_0^{N\gamma}\mu_t\,dt$. Since $\mathrm{FI}(\cdot\|\pi)$ is convex, using Jensen's inequality, we obtain the upper bound for $\mathrm{FI}(\bar{\mu}_{N\gamma}\|\pi)$
\begin{align*}
\mathrm{FI}(\bar{\mu}_{N\gamma}\|\pi)
\leq
\frac{1}{N\gamma}\int_0^{N\gamma}\mathrm{FI}(\mu_t\|\pi)\,dt &\leq \underbrace{\frac{2C_0}{N\gamma}}_{\mathrm{(I)}}
+ \underbrace{9\sigma^2\,
\left[\frac{\frac{p}{b} + 2(1-p)(1-\beta)^2}{2(1-(1-p)\beta)}\right]}_{\mathrm{(II)}}
+ \underbrace{\frac{30\gamma L^2d}{1-(1-p)\beta}}_{\mathrm{(III)}}.
\end{align*}

Substituting the definition of $\gamma$ into $\mathrm{(I)}$, we have
\begin{equation}\label{thm3:term-i-bound}
    \mathrm{(I)} = \frac{2C_0\sigma_m^{2/3}L^{2/3}}{N^{1/3}d^{1/3}}.
\end{equation}
To bound $\mathrm{(II)}$, we 
\begin{align}
    \mathrm{(II)} & \le 9\sigma_m^2\left[\frac{p^2 + (1-p)p^2}{1-(1-p)^2}\right] \notag \\ 
    & \le 18\sigma_m^2p, \label{thm3:term-ii-bound-uniq}
\end{align}where we use $\sigma \le \sigma_m$, $b\ge 1/p$, and $1-\beta = p \le 1$ to get~(\ref{thm3:term-ii-bound-uniq}). Substituting the definition of $p$, we get
\begin{equation}\label{thm3:term-ii-bound}
    \mathrm{(II)} \le \frac{18\sigma_m^{2/3}L^{2/3}d^{2/3}}{N^{1/3}}. 
\end{equation}

Similarly, to bound $\mathrm{(III)}$, 
\begin{align}
    \mathrm{(III)} & = \frac{30\gamma L^2 d}{1-(1-p)\beta}\notag \\
    &\le \frac{30\gamma L^2 d}{p} \notag\\
    & = \frac{30\sigma_m^{2/3}L^{2/3}d^{2/3}}{N^{1/3}} \label{thm3:term-iii-bound}.
\end{align}Combining the bounds for $\mathrm{(I)}$, $\mathrm{(II)}$, and $\mathrm{(III)}$, we have 
\begin{equation*}
    \mathrm{FI}(\bar{\mu}_{N\gamma}\|\pi) \le \left(\frac{2C_0}{d} + 48\right)\frac{\sigma_m^{2/3}L^{2/3}d^{2/3}}{N^{1/3}},
\end{equation*}where $C_0$ is usually taken to be of order $d$, see e.g. Lemma 1 in~\citet{vempala2019rapid} or Appendix A in~\citet{chewi2025analysis}. Thus, the above bound implies that, to ensure
$\mathrm{FI}(\bar{\mu}_{N\gamma}\|\pi) \le \varepsilon$,
\begin{equation*}
    N= \BigO\left(\frac{\sigma_m^2L^2d^2}{\varepsilon^3}\right)
\end{equation*}iterations is sufficient. And the average number of gradient computations per iteration is 
\begin{equation*}
    pb + (1-p) = \BigO(1). 
\end{equation*}We next verify that the constraints on $N$ imposed throughout the analysis are consistent with the stated iteration complexity.

First, the requirement $p = 1-\beta = \frac{d^{2/3}L^{2/3}}{\sigma_m^{4/3}N^{1/3}}\le 1$ imposes the lower bound $N\ge \frac{d^2L^2}{\sigma_m^4}$, which is satisfied by the above choice of $N$. 

Second, we show that the step-size condition $\gamma \le \frac{1}{L\sqrt{27\phi(p,\beta)}}$ is satisfied. We note that
\begin{align*}
    \phi(p,\beta) & = 1 + \frac{2(1-p)}{1-(1-p)\beta} \\
    & = \frac{1 + (1-p)\beta}{1-(1-p)\beta}\\
    & \le \frac{2}{1-(1-p)\beta}. 
\end{align*}Then, $\gamma$ can be chosen as 
\begin{align*}
    \gamma \le \frac{\sqrt{1-(1-p)\beta}}{L\sqrt{54}} \le \frac{1}{L\sqrt{27\phi(p,\beta)}}. 
\end{align*}Substituting the definitions of $\beta = 1-p$ and $p=\frac{d^{2/3}L^{2/3}}{\sigma_m^{4/3}N^{1/3}}$, we get 
\begin{align*}
    \frac{\sqrt{1-(1-p)\beta}}{L\sqrt{54}} &=\frac{\sqrt{1-(1-p)^2}}{L\sqrt{54}} \\
    & \ge \frac{\sqrt{p}}{L\sqrt{54}}\\
    & = \frac{d^{1/3}}{\sigma_m^{2/3}L^{2/3}N^{1/6}\sqrt{54}}.
\end{align*}Thus, the step-size condition holds provided that
\begin{equation*}
    \gamma \le \frac{d^{1/3}}{\sigma_m^{2/3}L^{2/3}N^{1/6}\sqrt{54}}.
\end{equation*}Using the definition of step size
\begin{equation*}
    \gamma = \frac{d^{1/3}}{\sigma_m^{2/3}L^{2/3}N^{2/3}},
\end{equation*}this condition reduces to the following lower bound $N\ge 54$, which is satisfied by $N=\BigO\left(\frac{\sigma_m^2L^2d^2}{\varepsilon^3}\right)$. \qedbox

\subsubsection{Proof of Corollary~\ref{proofs:theorem-and-lemmas}.1 (SL-VRLD: squared TV distance convergence)}\label{proof:sl-vrld-tvconv}

\paragraph{Corollary~\ref{proofs:theorem-and-lemmas}.1 (SL-VRLD: squared TV distance convergence)}\label{corr:sl-vrld-tvconv} \textit{
Let $\pi\propto e^{-f}$ be the target distribution, where the potential function $f$ satisfies Assumptions~\ref{ass:sfo},~\ref{ass:poincare-inequality}, and~\ref{ass:stochastic-potential-lipschitz}. Let $(\mu_t)_{t\geq 0}$ denote the law of continuous-time interpolation generated by~(\ref{def:ml-vrld:interpolation}) with estimator~(\ref{def:sl-vrld:estimator}). Then, for any step size $\gamma \in \left(0, \frac{1}{L\sqrt{27\phi(p,\beta)}}\right]$, where $\phi(p,\beta)= 1+\frac{2(1-p)}{1-(1-p)\beta}$, and for any $N\geq 1$, it holds that
\begin{equation}
    \|\Bar{\mu}_{N\gamma} - \pi\|_{\mathrm{TV}}^2\leq \frac{8C_0C_{\mathrm{PI}}}{N\gamma} + 18C_{\mathrm{PI}}\sigma^2\left[\frac{\frac{p}{b} + 2(1-p)(1-\beta)^2}{1-(1-p)\beta}\right] + \frac{80C_{\mathrm{PI}}\gamma L^2 d}{1-(1-p)\beta},
\end{equation}where 
\begin{equation}
    C_0 = \mathrm{KL}(\mu_0\|\pi) + \frac{9\gamma}{4(1-(1-p)\beta)}\E[\|\nabla f(\vx_0) - \vg_0\|^2].
\end{equation}Furthermore, choose
\[
\gamma=\left(\frac{d}{\sigma_m^2 N^2 L^2}\right)^{1/3},\qquad 
p=1-\beta=\left(\frac{d^2L^2}{\sigma_m^4 N}\right)^{1/3},\qquad 
b=\Bigl\lceil \tfrac{1}{p}\Bigr\rceil,
\quad \text{where } \sigma_m = \max\{\sigma,1\}.
\]
Then, the time-averaged law $\bar{\mu}_{N\gamma}=\frac{1}{N\gamma}\int_{0}^{N\gamma}\mu_t\,dt$ satisfies $\|\Bar{\mu}_{N\gamma} - \pi\|_{\mathrm{TV}}^2\le \varepsilon$ with $\BigO\!\left(\frac{C_\mathrm{PI}^3\sigma_m^{2}L^{2}d^{2}}{\varepsilon^{3}}\right)$ iterations and $\BigO(1)$ gradient computations per iteration.}

\textit{Proof.} Recall from the proof (Appendix~\ref{proof:sl-vrld-ficonv}) that the Fisher information of the time-averaged law $\bar{\mu}_{N\gamma}$ satisfies (\ref{eq:integral-bound-sl}),
\begin{align*}
\frac{1}{N\gamma}\int_{0}^{N\gamma}\mathrm{FI}(\mu_t \| \pi)\,dt
&\le
\frac{2C_0}{N\gamma}
+ 9\sigma^2\,
\left[\frac{\frac{p}{b} + 2(1-p)(1-\beta)^2}{2(1-(1-p)\beta)}\right]
+ \frac{30\gamma L^2d}{1-(1-p)\beta}.
\end{align*}Using Lemma~\ref{lemma:tv_lemma} and Assumption~\ref{ass:poincare-inequality} we get 
\begin{equation*}
        \|\Bar{\mu}_{N\gamma} - \pi\|_{\mathrm{TV}}^2\leq \frac{8C_0C_{\mathrm{PI}}}{N\gamma} + 36C_{\mathrm{PI}}\frac{\left(p^2 + 2(1-p)(1-\beta)^2\right)}{(1-(1-p)\beta)} + \frac{120C_{\mathrm{PI}}\gamma L^2 d}{1-(1-p)\beta}.
\end{equation*}Following the same steps in~(\ref{thm3:term-i-bound},\ref{thm3:term-ii-bound},\ref{thm3:term-iii-bound}), we obtain

\begin{equation*}
    \|\bar{\mu}_{N\gamma} - \pi\|_{\mathrm{TV}}^2 \le \left(\frac{2C_0}{d} + 48\right) \frac{4C_{\mathrm{PI}}\sigma_m^{2/3}L^{2/3}d^{2/3}}{N^{1/3}},
\end{equation*}where $C_0$ is in the order of $d$. This implies it suffices to choose $N$ such that $N^{1/3}\ge \frac{\tilde{C}C_{\mathrm{PI}}\sigma_m^{2/3}L^{2/3}d^{2/3}}{\varepsilon}$, where $\tilde{C}\coloneqq \frac{8C_0}{d} + 192> 0$. Equivalently, this yields iteration complexity of $\BigO\left(\frac{C_{\mathrm{PI}}^3\sigma_m^2L^2d^2}{\varepsilon^3}\right)$ with $\BigO(1)$ gradient computations per iteration. Since the constraints on $N$ in Theorem~\ref{thm:sl-vrld-ficonv} remain unchanged, they continue to hold. 
\qedbox

\subsubsection{Proof of Theorem~\ref{proofs:theorem-and-lemmas}.1 (SL-VRLD: weak convergence)}\label{proof:sl-vrld-weakconv}

\paragraph{Theorem~\ref{proofs:theorem-and-lemmas}.1 (SL-VRLD: weak convergence)}\label{thm:sl-vrld-weakconv}\textit{
Let $\pi\propto e^{-f}$ be the target distribution, where the potential function $f$ satisfies Assumptions~\ref{ass:sfo} and~\ref{ass:stochastic-potential-lipschitz}. Let $(\mu_t)_{t\geq 0}$ denote the law of continuous-time interpolation generated by~(\ref{def:ml-vrld:interpolation}) with estimator~(\ref{def:sl-vrld:estimator}). Define the time-varying parameters as follows
\[
\gamma_k=\frac{C_\gamma}{k^{3/2}},\qquad 
p_k=\frac{C_p}{k^{1/2}},\qquad 
b_k=\left\lceil \frac{1}{p_k}\right\rceil, \qquad \forall k\geq 1,
\]
where \(0<C_p<1\) and \(C_\gamma>0\) are numerical constants. 
Then, the time-averaged law $\bar{\mu}_{\tau_n}= \frac{1}{\tau_n}\int_0^{\tau_n}\mu_t\,dt$, where $\tau_n = \sum_{k=1}^n\gamma_k$, converges weakly to $\pi$.}

\textit{Proof.} Given time-varying parameters \(\gamma_k,b_k,p_k\) at iteration \(k\), define the cumulative time \(\tau_n\) and averaged law \(\bar{\mu}_{\tau_n}\) at iteration \(n\) as
\[
\tau_{n}\coloneqq\sum_{k=1}^{n}\gamma_{k},
\qquad
\bar{\mu}_{\tau_{n}}\coloneqq
\frac{1}{\tau_{n}}\int_{0}^{\tau_{n}}\mu_{t}\,dt,
\]
where \(\mu_t\) denotes the law of the process \(\vx_t\) under the following continuous-time interpolation
\begin{equation}
\vx_{t}\coloneqq 
\vx_{\tau_{n}}
-(t-\tau_{n})\,\vg_n
+\sqrt{2}\,\bigl(\mB_{t}-\mB_{\tau_{n}}\bigr),
\qquad t\in[\tau_{n},\tau_{n+1}].
\end{equation}$\vg_n$ is defined as follows

\begin{equation}
\vg_{n} :=
\begin{cases}
\displaystyle
\frac{1}{b_{n}}\sum_{i=1}^{b_{n}} \nabla f(\vx_{\tau_n}, \xi_{n}^i)
& \text{w.p. } p_{n}, \\[1ex]
\displaystyle
\nabla f(\vx_{\tau_n}, \xi_n)
+ \beta_n (\vg_{n-1} - \nabla f(\vx_{\tau_{n-1}}, \xi_n))
& \text{w.p. } 1-p_{n},
\end{cases}
\end{equation}for all $n \geq 1$. At initialization, we choose $\gamma_0, b_0>0$ and $(\beta_0,p_0)\in(0,1]^2 \setminus\{(1,1)\}$, and define $\vg_0$ as
\begin{equation*}
\vg_0 \coloneqq \frac{1}{b_0}\sum_{i=1}^{b_0}\nabla f(\vx_0, \xi_0^i).
\end{equation*}

We establish weak convergence by adapting the proof of Theorem~\ref{thm:sl-vrld-ficonv}. To this end, we first verify the step sizes $(\gamma_k)_{k\ge 1}$ satisfy the step-size condition of Theorem~\ref{thm:sl-vrld-ficonv}, namely
\begin{equation}\label{thm3:gamma-bound-verification}
    \gamma_k\in\left(0, \frac{1}{L\sqrt{\phi(p_k,\beta_k)}}\right], \quad \text{where}\quad \phi(p_k,\beta_k) = 1 + \frac{2(1-p)}{1-(1-p)\beta}
\end{equation}for all $k\ge 1$. Using the definitions of $p_k$ and $\beta_k$, we have
\begin{align*}
\phi(p_k,\beta_k) & = 1 + \frac{2(1-p_k)}{1-(1-p_k)\beta_k} \\ 
& \le 1 + \frac{2}{1-(1-p_k)^2} \\
& \le 1 + \frac{2}{p_k}\\
& \le 1 + \frac{2k^{1/2}}{C_p}\\
& \le \frac{4k^{1/2}}{C_p},
\end{align*}where we use the fact that $0<C_p<1$. Then, we have 
\begin{equation*}
    \frac{1}{L\sqrt{27\phi(p_k,\beta_k)}} \ge \frac{\sqrt{C_p}}{L\sqrt{108}k^{1/4}}.
\end{equation*}Choosing $C_\gamma = \frac{\sqrt{C_p}}{L\sqrt{108}}$, we have 
\begin{equation*}
    \gamma_k = \frac{C_\gamma}{k^{3/2}} \le \frac{\sqrt{C_p}}{L\sqrt{108}k^{1/4}} \le \frac{1}{L\sqrt{27\phi(p_k,\beta_k)}}
\end{equation*}for $k\ge 1$. This verifies~(\ref{thm3:gamma-bound-verification}). 
Therefore, we can follow the same steps in Theorem~\ref{thm:sl-vrld-ficonv} (Appendix~\ref{proof:sl-vrld-ficonv}) up to~(\ref{eq:sl-vrld-kl-bound-for-id}) since the same arguments remain valid for $t\in[\tau_{n-1}, \tau_n]$ with time-varying parameters. We obtain
\begin{align}
\nonumber
\frac{d}{dt}\mathrm{KL}(\mu_t \| \pi)
&\le
-\frac{1}{2}\,\mathrm{FI}(\mu_t \| \pi)
-\frac{1}{1-(1-p_n)\beta_n}
\left(
2 + \frac{27}{4}\gamma_n^2 L^2 \phi(p_n,\beta_n)
\right)
\bigl(e_{n}^2 - e_{n-1}^2\bigr) \\[0.4em]
&\quad
+ \frac{9\sigma^2}{4\bigl(1-(1-p_n)\beta_n\bigr)}
\left[
\frac{p_n}{b_n} + 2(1-p_n)(1-\beta_n)^2
\right]
+ \frac{15}{2}\gamma_n L^2 \,\phi(p_n,\beta_n)d \label{thm4:dt-kl-bound-unique}
\end{align} for $t\in[\tau_{n-1}, \tau_n]$, where $\phi(p_n, \beta_n)\coloneqq 1 + \frac{2(1-p_n)}{1-(1-p_n)\beta_n}$. Integrating both sides of the inequality over the interval $[\tau_{n-1}, \tau_n]$ yields
\begin{align*}
\nonumber
\mathrm{KL}(\mu_{\tau_n} \| \pi) - \mathrm{KL}(\mu_{\tau_{n-1}} \| \pi)
&\le
-\frac{1}{2}\int_{\tau_{n-1}}^{\tau_n}\mathrm{FI}(\mu_t \| \pi)\,dt
-\frac{\gamma_n}{1-(1-p_n)\beta_n}
\left[2+\frac{27}{4}\gamma_n^2L^2\phi(p_n,\beta_n)\right]
\bigl(e_{n}^2-e_{n-1}^2\bigr) \\
&\quad
+ \frac{9\gamma_n\,\sigma^2}{4\bigl(1-(1-p_n)\beta_n\bigr)}
\left[p_n^2 +2(1-p_n)(1-\beta_n)^2\right]
+ \frac{15}{2}\gamma_n^2 L^2\phi(p_n,\beta_n)d,
\end{align*} 
where we use $b_n \ge \frac{1}{p_n}$. Iterating the above bound, we obtain
\begin{align}
\nonumber
\mathrm{KL}(\mu_{\tau_n} \| \pi) - \mathrm{KL}(\mu_{0} \| \pi)
&\le
- \frac{1}{2}\int_{0}^{\tau_n}\mathrm{FI}(\mu_t \| \pi)\,dt
\underbrace{-\sum_{k=1}^n\frac{\gamma_k}{1-(1-p_k)\beta_k}
\left[2+\frac{27}{4}\gamma_k^2L^2\phi(p_k,\beta_k)\right]
\bigl(e_{k}^2-e_{k-1}^2\bigr)}_{\text{(I)}} \\ \label{eq:sl-vrld-id-kl-bound}
&\quad
+\underbrace{ \frac{9\sigma^2}{4} \sum_{k=1}^n
\gamma_k\left(\frac{p_k^2 +2(1-p_k)(1-\beta_k)^2}{1-(1-p_k)\beta_k}\right)}_{\text{(II)}}
+\underbrace{\frac{15L^2d}{2} \sum_{k=1}^n\gamma_k^2\phi(p_k,\beta_k)}_{\text{(III)}}.
\end{align} 
It remains to control the terms labeled \textup{(I)}--\textup{(III)} and show that their contribution remains uniformly bounded as \(k\to\infty\). This will imply that the upper bound (\ref{eq:sl-vrld-id-kl-bound}) stays bounded, which is a key step toward establishing weak convergence. 

First, analyze the term (I). Define 
\begin{align*}
    c_k \coloneqq \frac{\gamma_k}{1-(1-p_k)\beta_k}\left(4 + \frac{27}{2}\gamma_k^2L^2\phi(p_k, \beta_k)\right).
\end{align*}
If $(c_k)_{k\ge 1}$ is nonnegative and decreasing sequence, we have
\begin{align*}
    (\textup{I})=
    -\sum_{k=1}^nc_k(e_k^2-e_{k-1}^2) = c_1e_0^2 + \sum_{k=1}^{n-1}(c_{k+1}-c_{k})e_{k}^2 - c_ne_n^2 \leq c_1e_0^2
\end{align*}
Thus, it remains to show that $(c_k)_{k\ge 1}$ is nonnegative and decreasing sequence. Nonnegativity is obvious as all the terms are positive and $p_k< 1$ and $\beta_k < 1$. To prove that the sequence is decreasing, we substitute $\gamma_k, \beta_k$, and $p_k$. As shown in the proof of Theorem~\ref{thm:ml-vrld-weakconv}, we have
\begin{equation*}
    \frac{\gamma_k}{1-(1-p_k)\beta_k} = \frac{C_\gamma}{2C_pk - C_p^2k^{1/2}},
\end{equation*}which is decreasing for $k\ge 1$. In addition, 
\begin{equation*}
    \frac{27}{2}\gamma_k^2L^2\phi(p_k, \beta_k) = \frac{27C_\gamma^2L^2}{2k^3}\left(1 + \frac{2\left(1-\frac{C_p}{k^{1/2}}\right)}{1 - \left(1 -\frac{C_p}{k^{1/2}}\right)^2}\right).
\end{equation*}We can write this sequence as a continuous function
\begin{equation*}
    c(x) = \frac{27C_\gamma^2L^2}{2x^3}\left(1 + \frac{2\left(1-\frac{C_p}{x^{1/2}}\right)}{1 - \left(1 -\frac{C_p}{x^{1/2}}\right)^2}\right)
\end{equation*}for $x\ge 1$. We use change of variables $r=\frac{C_p}{x^{1/2}}$ and get 
\begin{equation*}
    c(x) = \frac{27C_\gamma^2L^2}{2C_p^6}\underbrace{\left(\frac{r^5(2-r^2)}{2-r}\right)}_{\coloneqq h(r)}.
\end{equation*}Differentiating $h(r)$ with respect to $r$, we have
\begin{equation*}
    h'(r) = \frac{2r^4(10-4r -7r^2+3r^3)}{(2-r)^2}.
\end{equation*}Defining $q(r)\coloneqq 10-4r-7r^2+3r^3$, we note that, for $r\in(0,1)$, $q'(r)=-4-14r+9r^2<-4-5r<0$. Therefore, 
\begin{equation*}
    q(r)> q(1) = 2 > 0,
\end{equation*}which implies $h'(r) > 0$ for $r\in(0,1)$. Hence, 
\begin{align*}
    c'(x)& \le \frac{27C_\gamma^2L^2}{2C_p^6}h'(r(x))r'(x) \\ 
    & \le -\frac{27C_\gamma^2 L^2}{4C_p^5x^{3/2}}h'(r(x)) < 0
\end{align*}for $x\ge 1$. Therefore, $c(x)$ is decreasing on $[1,\infty)$, which implies $(c_k)_{k\ge 1}$ is decreasing sequence. This implies
\begin{equation}\label{sl-vrld-weak-term-i-is-bounded}
    \limsup_{n\to\infty} (\mathrm{I}) \le c_1 e_0^2.
\end{equation}
We now analyze the term $\mathrm{(II)}$

\begin{align*}
\mathrm{(II)} &= 
\frac{9\sigma^2}{4} \sum_{k=1}^n
\gamma_k\left(\frac{p_k^2 +2(1-p_k)(1-\beta_k)^2}{p_k}\right) \\ &\leq 
\frac{9\sigma^2}{4} \sum_{k=1}^n
\gamma_k\left(\frac{p_k^2 +2(1-\beta_k)^2}{p_k}\right) \\
&= \frac{9\sigma^2}{4} \sum_{k=1}^n
\gamma_k\left(\frac{p_k^2 + 2p_k^2}{p_k}\right) \\
&=\frac{27}{4}\sigma^2 \sum_{k=1}^n
\gamma_kp_k
\end{align*}
Using the definitions of $\gamma_k$ and $p_k$, we obtain the following upper bound for (II):
\begin{align*}
    \frac{9\sigma^2}{4} \sum_{k=1}^n
\gamma_k\left(\frac{p_k^2 +2(1-p_k)(1-\beta_k)^2}{1-(1-p_k)\beta_k}\right) &\leq 
\frac{27C_{\gamma}C_p\sigma^2}{4}\sum_{k=1}^{n}\frac{1}{k^2}.
\end{align*}
Taking the limit as $n \rightarrow \infty$ yields
\begin{align}
    \lim_{k\rightarrow\infty}{(\textup{II})} \leq \frac{27C_{\gamma}C_p\sigma^2}{4}\sum_{n=1}^{\infty}\frac{1}{n^2} \eqqcolon S_{2}
< \infty.\label{eq:sl-vrld-bound-on-s2}
\end{align}
Finally, we analyze (III)
\begin{align*}
\mathrm{(III)} =\frac{15L^2d}{2} \sum_{k=1}^n\gamma_k^2\left(1 + \frac{2(1-p_k)}{1-(1-p_k)\beta_k}\right) \leq 
\frac{15L^2d}{2} \sum_{k=1}^n\gamma_k^2\left(1 + \frac{2(1-p_k)}{p_k}\right)
&\leq
\frac{15L^2d}{2} \sum_{k=1}^n\frac{2\gamma_k^2}{p_k}.
\end{align*}
Using the definitions of $\gamma_k$ and $p_k$, and taking the limit as $n \rightarrow \infty$ yields
\begin{align}
    \lim_{k\rightarrow\infty}{(\textup{III})} \leq \frac{15C_\gamma^2L^2d}{C_p}\sum_{k=1}^{\infty}\frac{1}{k^{\frac{5}{2}}} \eqqcolon S_{3}
< \infty.\label{eq:sl-vrld-bound-on-s3}
\end{align}
Now that we have shown the terms (I), (II) and (III) are uniformly bounded, we bound (\ref{eq:sl-vrld-id-kl-bound}) as
\begin{align}
\mathrm{KL}(\mu_{\tau_n} \| \pi) - \mathrm{KL}(\mu_0 \| \pi)
&\le
- \frac{1}{2}\int_{0}^{\tau_n}\mathrm{FI}(\mu_t \| \pi)\,dt
+ c_1e_0^2 +S_{2} +S_{3}.\label{eq:sl-vrld-weak-convergence-last-kl-inequality-s1s2s3}
\end{align} 
Since KL divergence is nonnegative, we may discard $\mathrm{KL}(\mu_{\mu_{\tau_n}}\|\pi) $ and rearrange the terms to get 
\begin{align*}
\frac{1}{2}\int_{0}^{\tau_n}\mathrm{FI}(\mu_t \| \pi)\,dt
&\le
\mathrm{KL}(\mu_{0} \| \pi)
+ c_1e_0^2 +S_{2} +S_{3}.
\end{align*} 
Dividing both sides by $\frac{\tau_n}{2}$, we obtain 
\begin{align*}
\frac{1}{\tau_n}\int_{0}^{\tau_n}\mathrm{FI}(\mu_t \| \pi)\,dt
&\le
\frac{2\mathrm{KL}(\mu_{0} \| \pi)}{\tau_n}
+\frac{2c_1e_0^2}{\tau_n} +\frac{2S_{2}}{\tau_n} +\frac{2S_{3}}{\tau_n}.
\end{align*} 
Since $\mathrm{FI}(\cdot\|\pi)$ is convex, we use Jensen's inequality to obtain
\begin{align}\label{eq:sl-vrld-id-avg-law-fi}
\mathrm{FI}(\bar \mu_{\tau_n} \| \pi) \leq
\frac{1}{\tau_n}\int_{0}^{\tau_n}\mathrm{FI}(\mu_t \| \pi)\,dt
&\le
\frac{2\mathrm{KL}(\mu_{0} \| \pi)}{\tau_n}
+\frac{2c_1e_0^2}{\tau_n} +\frac{2S_{2}}{\tau_n} +\frac{2S_{3}}{\tau_n},
\end{align} 
where $\mu_{\tau_n} \coloneqq \frac{1}{\tau_n}\int_0^{\tau_n}\mu_t\,dt$. Given that $\mathrm{KL}(\mu_0\|\pi)< \infty$, we have $\mathrm{FI}(\bar{\mu}_{\tau_n}\|\pi)<\infty$. On the other hand, if $t\in[\tau_n,\tau_{n+1}]$, considering~(\ref{thm4:dt-kl-bound-unique}) for the interval $[\tau_n, \tau_{n+1}]$, integrating both sides from $\tau_n$ to $t$, and dropping the negative integral over the Fisher information gives us

\begin{align*}
\nonumber
\frac{d}{dt}\mathrm{KL}(\mu_t \| \pi)
&\le
-\frac{1}{1-(1-p_{n+1})\beta_{n+1}}
\left(
2 + \frac{27}{4}\gamma_{n+1}^2 L^2 \phi(p_{n+1},\beta_{n+1})
\right)
\bigl(e_{n+1}^2 - e_n^2\bigr) \\[0.4em]
&\quad
+ \frac{9\sigma^2}{4\bigl(1-(1-p_{n+1})\beta_{n+1}\bigr)}
\left[
p_{n+1}^2+ 2(1-p_{n+1})(1-\beta_{n+1})^2
\right]
+ \frac{15}{2}\gamma_{n+1} L^2 d\,\phi(p_{n+1},\beta_{n+1}).
\end{align*}
Furthermore, integrating both sides over the interval $[\tau_n,t]$, we obtain
\begin{align*}
\mathrm{KL}(\mu_t \| \pi)
&\leq\,
\mathrm{KL}(\mu_{\tau_n} \| \pi)
-
\frac{(t-\tau_n)}{1-(1-p_{n+1})\beta_{n+1}}\left[2 + \frac{27}{4}\gamma_{n+1}^2L^2\phi(p_{n+1}, \beta_{n+1})\right](e_{n+1}^2 - e_n^2) \\
&\quad + \frac{9\,\sigma^2(t-\tau_n)}{4\bigl(1-(1-p_{n+1})\beta_{n+1}\bigr)}
\left[\frac{p_{n+1}}{b_{n+1}} +2(1-p_{n+1})(1-\beta_{n+1})^2\right]
+ \frac{15}{2}(t-\tau_{n})\gamma_{n+1} L^2\phi(p_{n+1},\beta_{n+1})\,d \\
&\leq
\mathrm{KL}(\mu_{0} \| \pi) + \frac{9\,\sigma^2(t-\tau_n)}{4\bigl(1-(1-p_{n+1})\beta_{n+1}\bigr)}
\left[\frac{p_{n+1}}{b_{n+1}} +2(1-p_{n+1})(1-\beta_{n+1})^2\right] \\
& \quad + \frac{15}{2}(t-\tau_{n})\gamma_{n+1} L^2\phi(p_{n+1},\beta_{n+1})\,d + c_1e_0^2 + S_2 + S_3\\
&\leq
\mathrm{KL}(\mu_{0} \| \pi)  + c_1e_0^2 + 2S_{2} + 2S_{3} < \infty.
\end{align*}where the second inequality follows from~(\ref{eq:sl-vrld-id-kl-bound}), the bounds on $\mathrm{(II)}$ and $\mathrm{(III)}$, and the fact that 
\begin{equation*}
    c_{n+1}\ge \frac{(t-\tau_n)}{1-(1-p_{n+1})\beta_{n+1}}
\end{equation*}for $t\in[\tau_n, \tau_{n+1}]$. To obtain the last inequality, we again use the bounds on $\mathrm{(II)}$ and $\mathrm{(III)}$. Thus $\{\mathrm{KL(\mu_t \| \pi)}\,|\,t\geq 0\}$ is bounded. The weak convergence of $\bar{\mu}$ to $\pi$ then follows by the same argument as in the proof of Theorem~\ref{thm:ml-vrld-weakconv}.
\qedbox 

\subsection{Posterior Sampling with SGM Prior Proofs}

\subsubsection{Proof of Theorem~\ref{thm:ficonv-SGM-prior}\textit{a} (ML-VRLD with SGM prior: FI convergence)}\label{proof:ml-vrld:ficonv-SGM-prior}
\paragraph{Theorem~\ref{thm:ficonv-SGM-prior}\textit{a} (ML-VRLD with SGM prior: FI convergence)}\textit{
    Let $\pi(\vx|\vy) \propto \ell(\vy|\vx)p(\vx)$ be the posterior with the likelihood $\ell(\vy|\vx)\propto e^{-f(\vx)}$ and the prior $p(\vx)\propto e^{-h(\vx)}.$ Suppose the likelihood potential \(f\) satisfies Assumptions~\ref{ass:potential-lipschitz} and~\ref{ass:sfo} with Lipschitz constant $L_f$, the prior potential \(h\) satisfies Assumption~\ref{ass:potential-lipschitz} and~\ref{ass:smoothing-err-bounded} with Lipschitz constant $L_h$, and the SGM satisfies Assumption~\ref{ass:score-network} with decreasing error $\varepsilon_{\sigma_k}=\BigO(k^{-1/2})$ for $k\ge 1$. Define $L_\pi = L_f + L_h$. Let $(\mu_t)_{t\geq 0}$ denote the law of the interpolation generated by~(\ref{def:SGM-prior-interpolation}) with the estimator~(\ref{def:ml-vrld:estimator}) and the schedules defined in~(\ref{def:annealing-smoothing-schedules}). For any step size $\gamma \in \left(0, \frac{1}{L_\pi\sqrt{52\phi(p,\beta)}}\right]$, where $\phi(p,\beta)= 1+\frac{4(1-p)\beta^2}{(1-(1-p)\beta)^2}$, and for any $N\ge 1$, it holds that
    \begin{align}\label{eq:ml-prior-fi-integral-bound}
        \frac{1}{N\gamma}\int_{0}^{N\gamma} \mathrm{FI}(\mu_t || \pi)\, dt\leq \frac{2C_0}{N\gamma} + 13\sigma^2\left[\frac{\frac{p}{b} + (1-p)(1-\beta)^2}{1-(1-p)\beta}\right] + \frac{84\gamma L_\pi^2 d}{(1-(1-p)\beta)^2}+\Bar{\sigma}^2 + \Bar{\varepsilon}_{\sigma}^2 + \Bar{\alpha}^2,
    \end{align}where 
    \begin{equation} \label{eq:ml-prior-defs}
        \Bar{\sigma}^2 = \frac{39C_1^2}{2N}\sum_{k=0}^{N-1}\sigma_k^2, \quad \Bar{\varepsilon}^2_{\sigma} = \frac{39}{2N}\sum_{k=0}^{N-1}\varepsilon_{\sigma_k}^2, \quad \Bar{\alpha}^2 = \frac{39C_2^2}{2N}\sum_{k=0}^{N-1}\frac{(\alpha_k - 1)^2}{\sigma_k^2},
    \end{equation}}
    \textit{
    \begin{equation}
        C_0 = \mathrm{KL}(\mu_0\|\pi) + \frac{13\gamma}{2(1-(1-p)\beta)^2}\E[\|\nabla f(\vx_0) - \vg_0\|^2],
    \end{equation}and $C_0, C_1, C_2$ are positive constants. Furthermore, let
    \begin{equation}
        \gamma=\frac{\sqrt d}{\sigma_m N^{3/4}\sqrt{L_\pi}},\qquad 
        p=1-\beta=\frac{\sqrt{L_\pi d}}{\sigma_m N^{1/4}},\qquad 
        b=\Bigl\lceil \tfrac{1}{p}\Bigr\rceil,
        \quad \text{where} \quad \sigma_m = \max\{\sigma,1\},
    \end{equation}and let the initial parameters satisfy $\sigma_0\ge\sigma_{\mathrm{min}}$, $\varepsilon_{\sigma_0}>0$, $\alpha_0\ge 1$, and $\sigma_{\mathrm{min}}> 0$ is the minimum noise level. Then, the time-averaged law $\bar{\mu}_{N\gamma}=\frac{1}{N\gamma}\int_{0}^{N\gamma}\mu_t\,dt$ satisfies $\mathrm{FI}(\bar{\mu}_{N\gamma}\|\pi)\le \varepsilon$ with $\BigO\!\left(\frac{\sigma_m^{4}L_\pi^{2}d^{2}}{\varepsilon^{4}}\right)$ iterations and $\BigO(1)$ gradient computations per iteration. }

\textit{Proof.} In the case of posterior sampling with SGM priors recall that the SDE in interest is:
\[
d \vx_t
=
- \big( \nabla f(\vx_t) + \nabla h(\vx_t) \big)\, dt
+
\sqrt{2}\, d\mB_t.
\]
We again begin by invoking Lemma \ref{lemma:sampling_first},
\begin{align*}
\frac{d}{dt}\mathrm{KL}(\mu_t \| \pi)
&\leq
-\frac{3}{4}\mathrm{FI}(\mu_t \| \pi)
+
\E \left[\|\nabla \log \pi(\vx_t | \vy) + \vg_k - \alpha_k\gS_{\theta}(\vx_{k\gamma}, \sigma_k)
\|^2\right] \\
&\leq
-\frac{3}{4}\mathrm{FI}(\mu_t \| \pi)
+
\E \left[\| -\nabla h(\vx_t) - \nabla f(\vx_t)  + \vg_k - \alpha_k\gS_{\theta}(\vx_{k\gamma}, \sigma_k)
\|^2\right] \\
&\leq
-\frac{3}{4}\mathrm{FI}(\mu_t \| \pi)
+
\E [\| \nabla f( \vx_{k\gamma}) + \nabla h( \vx_{k\gamma}) - \nabla f( \vx_{t}) - \nabla h( \vx_{t}) - \nabla f( \vx_{k\gamma}) + \vg_k \\
& \qquad - \nabla h( \vx_{k\gamma}) + \nabla h_{\sigma_k}( \vx_{k\gamma}) -
\nabla h_{\sigma_k}( \vx_{k\gamma}) - \gS_{\theta}(\vx_{k\gamma}, \sigma_k)
- (\alpha_k -1)\gS_{\theta}(\vx_{k\gamma}, \sigma_k)
 \|^2 ].
\end{align*} 
Applying Young's inequality together with Assumptions~\ref{ass:potential-lipschitz}, \ref{ass:smoothing-err-bounded}, and~\ref{ass:score-network}, we get
\begin{align} \label{eq:ml-prior-main-inequality}
    \frac{d}{dt}\mathrm{KL}(\mu_t \| \pi)
&\leq
-\frac{3}{4}\mathrm{FI}(\mu_t \| \pi) + 3L_{\pi}\Delta_k + 3e_k^2 + 3(\sigma_kC_1 + \varepsilon_{\sigma_k} + (\alpha_k -1)C_2\sigma_k^{-1})^2
\end{align}
where $\Delta_k \coloneqq \E\left[\|\vx_{(k+1)\gamma} - \vx_{k\gamma}\|^2\right] \geq \E\left[\|\vx_{t} - \vx_{k\gamma}\|^2\right]$, $L_{\pi} \coloneqq L_f + L_h$ and $e_k \coloneqq \E \left[\|\vg_k - \nabla f(\vx_{k\gamma})\|^2\right]$. We use the recursive upper bound on the estimation error $e_k^2$ in~(\ref{eq:upper-bound-ek}):
\begin{align}\label{sl-prior-error-bound}
    e_k^2 
    \leq\,
    -\frac{2}{1-(1-p)\beta}\bigl(e_{k+1}^2 - e_k^2\bigr)
    + \frac{2\sigma^2}{1-(1-p)\beta}
    \left[\frac{p}{b}+(1-p)(1-\beta)^2\right]
    + \frac{4(1-p)\beta^2 L_f^2}{\bigl(1-(1-p)\beta\bigr)^2}\,\Delta_k .
\end{align}
Plugging this into~(\ref{eq:ml-prior-main-inequality}),  we obtain
\begin{align}\label{eq:main-ineq-sl-prior-v2}
    \nonumber
    \frac{d}{dt}\mathrm{KL}(\mu_t \| \pi)
    &\leq
    -\frac{3}{4}\mathrm{FI}(\mu_t \| \pi) +
    3L_{\pi}^2\phi(p, \beta)\Delta_k + \frac{6\sigma^2}{1-(1-p)\beta}\left[\frac{p}{b}+(1-p)(1-\beta)^2\right] \\
    &\qquad - \frac{6}{1-(1-p)\beta}\bigl(e_{k+1}^2 - e_k^2\bigr) + 3(\sigma_kC_1 + \varepsilon_{\sigma_k} + (\alpha_k -1)C_2\sigma_k^{-1})^2,
\end{align} where $\phi(p, \beta) \coloneqq 1 + \frac{4(1-p)\beta^2}{(1-(1-p)\beta)^2}$. We bound $\Delta_k$ using the interpolation (\ref{def:SGM-prior-interpolation}):
\begin{align*}
    \Delta_k &= 
    \ E \left[\| \vx_{(k+1)\gamma} - \vx_{k\gamma}\|^2\right] =
    \gamma^2 \ E \left[\| \vg_k - \alpha_k\gS_{\theta}(\vx_{k\gamma}, \sigma_k)
\|^2\right] + 2\gamma d \\
&\qquad=\gamma^2\ E [\|    \vg_k - \nabla f( \vx_{k\gamma}) + \nabla f( \vx_{t}) + \nabla h( \vx_{t}) + \nabla f( \vx_{k\gamma})  - \nabla f( \vx_{t}) + \nabla h( \vx_{k\gamma}) \\
&\quad-  \nabla h( \vx_{t}) - \nabla h( \vx_{k\gamma}) + \nabla h_{\sigma_k}( \vx_{k\gamma}) - \nabla h_{\sigma_k}( \vx_{k\gamma}) - \gS_{\theta}(\vx_{k\gamma}, \sigma_k) - (\alpha_k - 1)\gS_{\theta}(\vx_{k\gamma}, \sigma_k)
    \|^2] + 2\gamma d.
\end{align*}
Applying Young’s inequality together with Assumptions~\ref{ass:potential-lipschitz}, \ref{ass:smoothing-err-bounded} and~\ref{ass:score-network} yields
\begin{align}\label{eq:ml-prior-delta-k-intermediate-step}
    \Delta_k 
    \leq 
    4\gamma^2e_k^2
    +
    4\gamma^2\E\left[\|\nabla \log \pi(\vx_t | \vy)\|^2\right]
    +
    4\gamma^2L_{\pi}^2\Delta_k
    +
    4\gamma^2(\sigma_kC_1 + \varepsilon_{\sigma_k} + (\alpha_k -1)C_2\sigma_k^{-1})^2
    +
    2\gamma d.
\end{align}
We now substitute the bound (\ref{sl-prior-error-bound}) on $e_k^2$ and rearrange the terms as follows
\begin{align*}
    \bigl(1-4\gamma^2L^2_{\pi}\phi(p, \beta)\bigr)\Delta_k 
    \leq
    &- \frac{8\gamma^2}{1-(1-p)\beta}(e_{k+1}^2 - e_k^2) 
    + \frac{8\gamma^2\sigma^2}{1-(1-p)\beta}\left[\frac{p}{b}+(1-p)(1-\beta)^2\right]  \\
    &+4\gamma^2\E\left[\|\nabla \log \pi(\vx_t | \vy)\|^2\right]
    +4\gamma^2(\sigma_kC_1 + \varepsilon_{\sigma_k} + (\alpha_k -1)C_2\sigma_k^{-1})^2
    +
    2\gamma d.
\end{align*} Setting $\gamma \leq \frac{1}{L_\pi\sqrt{52\phi(p, \beta)}}$ and dividing both sides by $\frac{12}{13}$, we obtain
\begin{align}\label{eq:sl-prior-delta-k-upper-bound}
    \nonumber
    \Delta_k
    \leq
    &-\frac{26\gamma^2}{3\bigl(1-(1-p)\beta\bigr)}(e_{k+1}^2 - e_k^2) 
    + \frac{26\gamma^2\sigma^2}{3\bigl(1-(1-p)\beta\bigr)}\left[\frac{p}{b}+(1-p)(1-\beta)^2\right]\\
    &+\frac{13}{3}\gamma^2\E\left[\|\nabla \log\pi(\vx_t | \vy)\|^2\right]
    + \frac{13}{3}\gamma^2 (\sigma_kC_1 + \varepsilon_{\sigma_k} + (\alpha_k -1)C_2\sigma_k^{-1})^2 +\frac{13}{6}\gamma d. 
\end{align}
Substituting the above bound into~(\ref{eq:main-ineq-sl-prior-v2}), applying Lemma~\ref{lemma:sampling_final}, using the step-size condition $\gamma \le \tfrac{1}{L_\pi\sqrt{52\phi(p,\beta)}}$, and using Young's inequality for $(\sigma_k C_1 + \varepsilon_{\sigma_k} + (\alpha_k - 1) C_2 \sigma_k^{-1})^2$ yields
\begin{align}
\nonumber
    \frac{d}{dt}\mathrm{KL}(\mu_t \| \pi)
    \le
    &-\frac{1}{2}\,\mathrm{FI}(\mu_t \| \pi) -
    \frac{6+26\gamma^2L_{\pi}^2\phi(p, \beta)}{1-(1-p)\beta}(e_{k+1}^2 - e_k^2)
    +\frac{13\sigma^2}{2\bigl(1-(1-p)\beta\bigr)}\left[\frac{p}{b}+(1-p)(1-\beta)^2\right] \\
    &+\left(26\gamma L_\pi + \frac{13}{2}\right)\gamma L_{\pi}^2d\phi(p, \beta) + \frac{39}{4}\sigma_k^2C_1^2 + \frac{39}{4}\varepsilon_{\sigma_k}^2 + \frac{39C_2^2(\alpha_k-1)^{2}}{4\sigma_k^2}.
\end{align}Since $\phi(p,\beta) \ge 1 $, we have $\gamma L_\pi\ge \frac{1}{\sqrt{52}}$, which implies 
\begin{equation*}
    26\gamma L_\pi + \frac{13}{2}\le \frac{21}{2}. 
\end{equation*}Then, we have
\begin{align} \label{eq:ml-vrld-sgm-prior-kl-bound-for-id}
\nonumber
    \frac{d}{dt}\mathrm{KL}(\mu_t \| \pi)
    \le
    &-\frac{1}{2}\,\mathrm{FI}(\mu_t \| \pi) -
    \frac{6+26\gamma^2L_{\pi}^2\phi(p, \beta)}{1-(1-p)\beta}(e_{k+1}^2 - e_k^2)
    +\frac{13\sigma^2}{2\bigl(1-(1-p)\beta\bigr)}\left[\frac{p}{b}+(1-p)(1-\beta)^2\right] \\
    &+\frac{21}{2}\gamma L_{\pi}^2d\phi(p, \beta) + \frac{39}{4}\sigma_k^2C_1^2 + \frac{39}{4}\varepsilon_{\sigma_k}^2 + \frac{39C_2^2(\alpha_k-1)^{2}}{4\sigma_k^2}.
\end{align}
Next, we define
\(
    \mathcal{L}_k \coloneqq \mathrm{KL}(\mu_{k\gamma} \| \pi) + \frac{\gamma\bigl(6 + 26\gamma^2L_{\pi}^2\phi(p, \beta)\bigr)}{1-(1-p)\beta}e_k^2
\)
and integrate both sides over $t \in [k\gamma,(k+1)\gamma]$ to get
\begin{align*}
    \mathcal{L}_{k+1} - \mathcal{L}_k
    \le
    &-\frac{1}{2}\int_{k\gamma}^{(k+1)\gamma} \mathrm{FI}(\mu_t \| \pi)\,dt
    +\frac{13\gamma\sigma^2}{2\bigl(1-(1-p)\beta\bigr)}\left[\frac{p}{b}+(1-p)(1-\beta)^2\right]
    + \frac{21}{2}\gamma L_{\pi}^2d\phi(p, \beta) \\
    & + \frac{39}{4}\gamma\sigma_k^2C_1^2 + \frac{39}{4}\gamma\varepsilon_{\sigma_k}^2 + \frac{39\gamma C_2^2(\alpha_k-1)^{2}}{4\sigma_k^2}
\end{align*}
Note that $\mathcal{L}_0 = \mathrm{KL}(\mu_0 \| \pi) + \frac{\gamma\bigl(6 + 26\gamma^2L_{\pi}^2\phi(p, \beta)\bigr)}{1-(1-p)\beta}e_0^2 \leq \mathrm{KL}(\mu_0 \| \pi) + \frac{13\gamma e_0^2}{2(1-(1-p)\beta)} \eqqcolon C_0 $. Iterating the above inequality for $k=0,1,\dots,N-1$ and rearranging the terms, we obtain
\begin{align*}
    \frac{1}{N\gamma}\int_{0}^{N\gamma}\mathrm{FI}(\mu_t \| \pi)\,dt
    \le
    &\frac{2C_0}{N\gamma} + 13\sigma^2\left[\frac{\frac{p}{b}+(1-p)(1-\beta)^2}{\bigl(1-(1-p)\beta\bigr)}\right] + 
    \frac{84\gamma L_{\pi}^2d}{\bigl(1-(1-p)\beta\bigr)^2} + \bar{\sigma}^2 + \bar{\varepsilon}_{\sigma}^2 + \bar{\alpha}^2,
\end{align*} where $\bar{\sigma}^2$, $\bar{\varepsilon}_{\sigma}^2$, and $\bar{\alpha}^2$
are defined in (\ref{eq:ml-prior-defs}). Note that we use $\phi(p, \beta) = 1 + \frac{4(1-p)\beta^2}{(1-(1-p)\beta)^2} \le 1 + \frac{4(1-p)\beta}{(1-(1-p)\beta)^2} = \frac{(1+(1-p)\beta)^2}{(1-(1-p)\beta)^2} \leq \frac{4}{(1-(1-p)\beta)^2}$ to bound $\phi(p, \beta)$ in the above inequality. This completes the first part of the proof. 

To establish the iteration complexity, we note that, replacing $L$ by $L_\pi$, all the terms in the upper bound of Theorem~\ref{thm:ml-vrld-ficonv} (except the bias terms $\bar{\varepsilon}_{\sigma}^2$, $\bar{\alpha}^2$, $\bar{\sigma}^2$) are the same with different numerical coefficients. Therefore, choosing the same parameters (replacing $L$ with $L_\pi$) as in Theorem~\ref{thm:ml-vrld-ficonv}
\[
    \gamma=\frac{\sqrt{d}}{\sigma_mN^{\frac{3}{4}}\sqrt{L_\pi}},\qquad 
    p=1-\beta=\frac{\sqrt{L_\pi d}}{\sigma_mN^{\frac{1}{4}}},
    \quad \text{where}\quad \sigma_m \coloneqq \max\{\sigma,1\},
\]we have 
\begin{equation*}
    \frac{1}{N\gamma}\int_{0}^{N\gamma}\mathrm{FI}(\mu_t\|\pi)\,dt\le \BigO\left(\frac{\sigma_mL_\pi^{1/2}d^{1/2}}{N^{1/4}}\right) + \bar{\varepsilon}_{\sigma}^2+\bar{\alpha}^2+\bar{\sigma}^2. 
\end{equation*}We recall that the noise $(\sigma_k)_{k=0}^{N-1}$ and $(\alpha_k)_{k=0}^{N-1}$ schedules are defined as 
\begin{equation*}
    \alpha_k = \max\{\alpha_0\rho_1^k, 1\} \quad \text{and} \quad \sigma_k = \max\{\sigma_0\rho_2^k, \sigma_{\mathrm{min}}\},
\end{equation*}where $\rho_1,\rho_2\in(0,1)$ denote decay rates, $\sigma_0>0$, $\alpha_0>0$, and $\sigma_{\mathrm{min}}>0$ is the minimum noise level. By definition in~(\ref{def:annealing-smoothing-schedules}), there exists indices $K_\alpha < N-1$ and $K_\sigma< N-1$ independent of $N$ such that $\alpha_k =1$ for $\forall k \ge K_\alpha$ and $\sigma_k = \sigma_{\mathrm{min}}$ for $\forall k\ge K_\sigma$. Using this definition, we next analyze the convergence behavior of the bias contributions due to the noise schedule, annealing schedule, and SGM estimation error separately. 

To bound $\bar{\sigma}^2$, we proceed as 
\begin{align*}
    \bar{\sigma}^2 &= \frac{39C_1^2}{2N} \sum_{k=0}^{N-1} \max\{\sigma_0^2\rho_2^{2k}, \sigma_{\mathrm{min}}^2\}\\
    & \overset{(*)}{=} \frac{39C_1^2}{2N} \sum_{k=0}^{K_{\sigma}-1} \sigma_0^2\rho_2^{2k} + \frac{39C_1^2}{2N} \sum_{k=K_\sigma}^{N-1} \sigma_{\mathrm{min}}^2 \\
    &=\underbrace{\frac{39C_1^2\sigma_0^2(1-\rho_2^{2K_\sigma})}{2N(1-\rho_2^2)}}_{\BigO(N^{-1})}+\frac{39C_1^2(N-K_\sigma)\sigma_{\mathrm{min}}^2}{2N},
\end{align*}where $(*)$ follows from the definition of the annealing schedule~(\ref{def:annealing-smoothing-schedules}). Setting $\sigma_{\mathrm{min}} = \BigO(N^{-1/8})$, we obtain
\begin{equation}\label{thm3:sigma-bound-bias}
    \bar{\sigma}^2 = \BigO(N^{-1}) + \BigO(N^{-1/4}) = \BigO(N^{-1/4}). 
\end{equation}

To bound $\Bar{\alpha}^2$, we use the fact that $(\alpha_k)^{N-1}_{k=0}$ is nonincreasing sequence and that there exists a constant $K_\alpha < N-1$, independent of $N$, such that $\alpha_k=1$ for all $k\ge K_\alpha$. Thus, we obtain

\begin{align}
    \bar{\alpha}^2 & = \frac{39C_2^2}{2N}\sum_{k=0}^{N-1}\frac{(\alpha_k - 1)^2}{\sigma_k^2} \notag \\
    & = \frac{39C_2^2}{2N} \sum_{k=0}^{K_\alpha-1} \frac{(\alpha_0\rho_1^k -1)^2}{\sigma_k^2} \notag \\
    & = \frac{39C_2^2}{2N\sigma_{\mathrm{min}}^2} \sum_{k=0}^{K_\alpha-1}(\alpha_0\rho_1^k -1)^2  \notag \\
    & = \frac{39C_2^2}{2N\sigma_{\mathrm{min}}^2} \sum_{k=0}^{K_\alpha-1}(\alpha_0^2\rho_1^{2k}-2\alpha_0\rho_1^k + 1)  \notag \\
    & \le \frac{39C_2^2\alpha_0^2}{2N\sigma_{\mathrm{min}}^2}\sum_{k=0}^{K_\alpha-1}\rho_1^{2k} + \frac{39C_2^2 K_\alpha}{2N\sigma_{\mathrm{min}}^2} \notag \\
    & \le \underbrace{\frac{39C_2^2\alpha_0^2(1-\rho_1^{2K_\alpha})}{2N\sigma_{\mathrm{min}}^2(1-\rho_1^2)}}_{\BigO(N^{-3/4})} + \underbrace{\frac{39C_2^2 K_\alpha}{2N\sigma_{\mathrm{min}}^2}}_{\BigO(N^{-3/4})}\notag \\
    & = \BigO(N^{-3/4}). \label{thm3:alpha-bound-bias}
\end{align}

Finally, to bound $\bar{\varepsilon}_{\sigma}^2$, we use the assumption that the SGM estimation error decays as $\varepsilon_{\sigma_k}^2\le \frac{C'}{k}$ for some constant $C'>0$. We proceed as 
\begin{align}
    \bar \varepsilon_{\sigma}^2 & = \frac{39}{2N} \sum_{k=0}^{N-1}\varepsilon_{\sigma_k}^2 \notag \\ 
    & = \frac{39\varepsilon_{\sigma_0}^2}{2N} + \frac{39C'}{2N}\sum_{k=1}^{N-1}\frac{1}{k}\notag \\
    & = \BigO\left(\frac{\log N}{N}\right). \label{thm3:err-bound-bias}
\end{align}

Combining the upper bounds~(\ref{thm3:sigma-bound-bias}),~(\ref{thm3:alpha-bound-bias}), and~(\ref{thm3:err-bound-bias}), and using the Jensen's inequality with the convexity of $\mathrm{FI}(\cdot\|\pi)$, we obtain
\begin{align*}
\mathrm{FI}(\bar{\mu}_{N\gamma}\|\pi)& = \BigO\left(\frac{\sigma_mL_\pi^{1/2}d^{1/2}}{N^{1/4}}\right) + \BigO\left(\frac{1}{N^{1/4}}\right) + \BigO\left(\frac{1}{N^{3/4}}\right) + \BigO\left(\frac{\log N}{N}\right)\\
& = \BigO\left(\frac{\sigma_mL_\pi^{1/2}d^{1/2}}{N^{1/4}}\right).
\end{align*}This implies that $\mathrm{FI}(\bar{\mu}_{N\gamma}\|\pi) \le \varepsilon$ is achieved with $\BigO\left(\frac{\sigma_m^4L_\pi^{2}d^{2}}{\varepsilon^4}\right)$ iterations with $pb + (1-p) = \BigO(1)$ gradient computations per iteration. 
\qedbox

\subsubsection{Proof of Theorem~\ref{thm:ficonv-SGM-prior}\textit{b} (SL-VRLD with SGM prior: FI convergence)}\label{proof:sl-vrld:ficonv-SGM-prior}
\paragraph{Theorem~\ref{thm:ficonv-SGM-prior} (SL-VRLD with SGM prior: FI convergence)}\textit{
    Let $\pi(\vx|\vy)\propto \ell(\vy|\vx)p(\vx)$ be the posterior with the likelihood $\ell(\vy|\vx)\propto e^{-f(\vx)}$ and the prior $p(\vx)\propto e^{-h(\vx)}$. Suppose the likelihood potential \(f\) satisfies Assumptions~\ref{ass:sfo} and~\ref{ass:stochastic-potential-lipschitz} with Lipschitz constant $L_f$, the prior potential \(h\) satisfies Assumptions~\ref{ass:potential-lipschitz} and~\ref{ass:smoothing-err-bounded} with Lipschitz constant $L_h$, and the SGM satisfies Assumptions~\ref{ass:score-network} with decreasing error $\varepsilon_{\sigma_k}=\BigO(k^{-1/2})$ for $k\ge 1$. Define $L_\pi = L_f + L_h$. Let $(\mu_t)_{t\geq 0}$ denote the law of the interpolation generated by~(\ref{def:SGM-prior-interpolation}) with the estimator~(\ref{def:sl-vrld:estimator}) and the schedules defined in~(\ref{def:annealing-smoothing-schedules}). For any step size $\gamma \in \left(0, \frac{1}{L_\pi\sqrt{52\phi(p,\beta)}}\right]$, where $\phi(p,\beta)= 1+\frac{2(1-p)}{1-(1-p)\beta}$, and for any $N\ge 1$, it holds that  
    \begin{align}\label{eq:sl-sgm-prior-integral-bound}
        \frac{1}{N\gamma}\int_{0}^{N\gamma} \mathrm{FI}(\mu_t || \pi)\, dt\leq \frac{2C_0}{N\gamma} + \frac{13\sigma^2}{2}\left[\frac{\frac{p}{b} + 2(1-p)(1-\beta)^2}{(1-(1-p)\beta)}\right] + \frac{63\gamma L_\pi^2 d}{1-(1-p)\beta}+\Bar{\sigma}^2 + \Bar{\varepsilon}_{\sigma}^2 + \Bar{\alpha}^2,
    \end{align}where 
    \begin{equation}\label{sl-sgm-prior-defs}
        \Bar{\sigma}^2 = \frac{39C_1^2}{2N}\sum_{k=0}^{N-1}\sigma_k^2, \quad \Bar{\varepsilon}^2_{\sigma} = \frac{39}{2N}\sum_{k=0}^{N-1}\varepsilon_{\sigma_k}^2, \quad \Bar{\alpha}^2 = \frac{39C_2^2}{2
        N}\sum_{k=0}^{N-1}\frac{(\alpha_k - 1)^2}{\sigma_k^2},
    \end{equation}}
    \textit{
    \begin{equation}
        C_0 = \mathrm{KL}(\mu_0\|\pi) + \frac{13\gamma}{4(1-(1-p)\beta)}\E[\|\nabla f(\vx_0) - \vg_0\|^2],
    \end{equation}and $C_0,C_1,C_2$ are positive constants. Furthermore, let
    \begin{equation}
    \gamma=\left(\frac{d}{\sigma_m^2N^2L_\pi^2}\right)^{1/3},\qquad 
    p=1-\beta=\left(\frac{d^2L_\pi^2}{\sigma_m^4N}\right)^{1/3},\qquad 
    b=\Bigl\lceil \tfrac{1}{p}\Bigr\rceil,
    \quad \text{where } \sigma_m = \max\{\sigma,1\},
    \end{equation}and let the initial parameters satisfy $\sigma_0\ge\sigma_{\mathrm{min}}$, $\varepsilon_{\sigma_0}>0$, $\alpha_0\ge 1$, and $\sigma_{\mathrm{min}}> 0$ is the minimum noise level. Then, the time-averaged law $\bar{\mu}_{N\gamma}\coloneqq \frac{1}{N\gamma}\int_{0}^{N\gamma}\mu_t\,dt$ satisfies $\mathrm{FI}(\bar{\mu}_{N\gamma}\|\pi)\le \varepsilon$ with ~$\BigO\!\left(\frac{\sigma_m^{2}L_\pi^{2}d^{2}}{\varepsilon^{3}}\right)$ iterations and $\BigO(1)$ gradient computations per iteration.}

\textit{Proof.} Similar to the proof in Appendix~\ref{proof:ml-vrld:ficonv-SGM-prior}, we start by using Lemma \ref{lemma:sampling_first},
\begin{align} \label{eq:sl-prior-main-inequality}
    \frac{d}{dt}\mathrm{KL}(\mu_t \| \pi)
&\leq
-\frac{3}{4}\mathrm{FI}(\mu_t \| \pi) + 3L_{\pi}^2\Delta_k + 3e_k^2 + 3(\sigma_kC + \varepsilon_{\sigma_k} + (\alpha_k -1)C\sigma_k^{-1})^2.
\end{align}
From the analysis of Theorem~\ref{thm:sl-vrld-ficonv} in Appendix \ref{proof:sl-vrld-ficonv}, we have the upper bound~(\ref{eq:sl-vrld-ek-upper-bound}) on $e_k^2$
\begin{align}\label{eq:sl-prior-bound-on-ek}
    e_k^2 &\leq -\frac{1}{1-(1-p)\beta}(e_{k+1}^2-e_k^2) + 
    \frac{2(1-p)L_\pi^2}{1-(1-p)\beta}\Delta_k +
    \left[\frac{p}{b}+2(1-p)(1-\beta)^2\right]\frac{\sigma^2}{1-(1-p)\beta},
\end{align}where we use the fact that $L_\pi\ge L_f$. Plugging this bound into~(\ref{eq:sl-prior-main-inequality}), we obtain
\begin{align}
\label{eq:sl-sgm-prior-main-inequality-v2}
\nonumber
    \frac{d}{dt}\mathrm{KL}(\mu_t \| \pi) 
    \leq
    &-\frac{3}{4}\mathrm{FI}(\mu_t \| \pi)
    +
    3L_{\pi}^2\phi(p, \beta)\Delta_k
    -
    \frac{3}{1-(1-p)\beta}(e_{k+1}^2-e_k^2) \\
    &+
    \frac{3\sigma^2}{1-(1-p)\beta}\left[\frac{p}{b}+2(1-p)(1-\beta)^2\right]
    +
    3(\sigma_kC + \varepsilon_{\sigma_k} + (\alpha_k -1)C\sigma_k^{-1})^2,
\end{align} where $\phi(p, \beta) \coloneqq 1 + \frac{2(1-p)}{1-(1-p)\beta}$. Following the analysis of Theorem~\ref{thm:ficonv-SGM-prior}a, from the inequality~(\ref{eq:ml-prior-delta-k-intermediate-step}), we have the following upper bound on $\Delta_k$:
\begin{align*}
    \Delta_k 
    \leq 
    4\gamma^2e_k^2
    +
    4\gamma^2\E\left[\|\nabla \log \pi(\vx_t | \vy)\|^2\right]
    +
    4\gamma^2L_{\pi}^2\Delta_k
    +
    4\gamma^2(\sigma_kC_1 + \varepsilon_{\sigma_k} + (\alpha_k -1)C_2\sigma_k^{-1})^2
    +
    2\gamma d.
\end{align*}
Plugging in the upper~(\ref{eq:sl-prior-bound-on-ek}) on $e_k^2$ and rearranging the terms yields
\begin{align*}
    \left(1 - 4\gamma^2L_{\pi}^2\phi(p, \gamma) \right)\Delta_k
    \leq 
    &- \frac{4\gamma^2}{1-(1-p)\beta}(e_{k+1}^2-e_k^2)    
    +
    \frac{4\gamma^2\sigma^2}{1-(1-p)\beta}\left[\frac{p}{b}+2(1-p)(1-\beta)^2\right] \\
    &+
    4\gamma^2\E\left[\|\nabla \log \pi(\vx_t | \vy)\|^2\right]
    +
    4\gamma^2(\sigma_kC + \varepsilon_{\sigma_k} + (\alpha_k -1)C\sigma_k^{-1})^2
    +
    2\gamma d.
\end{align*}
Setting $\gamma$ such that $\gamma \le \frac{1}{L_\pi\sqrt{52\,\phi(p,\beta)}}$ and dividing both sides by $\frac{12}{13}$, we obtain
\begin{align}\label{eq:sl-sgm-prior-delta_k-bound}
    \nonumber
    \Delta_k
    \leq 
    &- \frac{13\gamma^2}{3\left(1-(1-p)\beta\right)}(e_{k+1}^2-e_k^2)    
    +
    \frac{13\gamma^2\sigma^2}{3\left(1-(1-p)\beta\right)}\left[\frac{p}{b}+2(1-p)(1-\beta)^2\right] \\
    &+
    \frac{13}{3}\gamma^2\E\left[\|\nabla \log p(\vx_t | \vy)\|^2\right]
    +
    \frac{13}{3}\gamma^2(\sigma_kC + \varepsilon_{\sigma_k} + (\alpha_k -1)C\sigma_k^{-1})^2
    +
    \frac{13}{6}\gamma d.
\end{align}Substituting the above bound into~(\ref{eq:sl-sgm-prior-main-inequality-v2}), applying Lemma~\ref{lemma:sampling_final}, using the step-size condition $\gamma \le \frac{1}{L_\pi\sqrt{52\phi(p,\beta)}}$, and using Young's inequality for $(\sigma_kC_1 + \varepsilon_{\sigma_k} + (\alpha_k - 1)C_2\sigma_k^{-1})^2$ yields
\begin{align}\label{eq:sl-vrld-sgm-prior-kl-bound-for-id}
\nonumber
    \frac{d}{dt}\mathrm{KL}(\mu_t \| \pi)
    \le
    &-\frac{1}{2}\,\mathrm{FI}(\mu_t \| \pi) 
    -
    \frac{3+13\gamma^2L_{\pi}^2\phi(p, \beta)}{1-(1-p)\beta}(e_{k+1}^2 - e_k^2) + \frac{13\sigma^2}{4(1-(1-p)\beta)}\left[\frac{p}{b}+2(1-p)(1-\beta)^2\right]\\
    &+ \frac{21}{2}\gamma L_{\pi}^2 d \phi(p, \beta) + \frac{39}{4}\sigma_k^2C_1^2 + \frac{39}{4}\varepsilon_{\sigma_k}^2 + \frac{39C_2^2(\alpha_k-1)^{2}}{4\sigma_k^2}
\end{align}
Next, we define 
\(
    \mathcal{L}_k \coloneqq \mathrm{KL}(\mu_{k\gamma} \| \pi) + \frac{\gamma\bigl(3 + 13\gamma^2L_{\pi}^2\phi(p, \beta)\bigr)}{1-(1-p)\beta}e_k^2
\)
and integrate both sides over $t \in [k\gamma,(k+1)\gamma]$.
\begin{align*}
    \mathcal{L}_{k+1} - \mathcal{L}_k
    \le
    &-\frac{1}{2}\int_{k\gamma}^{(k+1)\gamma} \mathrm{FI}(\mu_t \| \pi)\,dt
    +\frac{13\gamma\sigma^2}{4\bigl(1-(1-p)\beta\bigr)}\left[\frac{p}{b}+2(1-p)(1-\beta)^2\right]
    + \frac{21}{2}\gamma L_{\pi}^2d\phi(p, \beta) \\
    & + \frac{39}{4}\gamma\sigma_k^2C_1^2 + \frac{39}{4}\gamma\varepsilon_{\sigma_k}^2 + \frac{39\gamma C_2^2(\alpha_k-1)^{2}}{4\sigma_k^2}.
\end{align*}
Note that $\mathcal{L}_0 = \mathrm{KL}(\mu_0 \| \pi) + \frac{\gamma\bigl(3 + 13\gamma^2L_{\pi}^2\phi(p, \beta)\bigr)}{1-(1-p)\beta}e_0^2 \leq \mathrm{KL}(\mu_0 \| \pi) + \frac{13\gamma e_0^2}{4\left(1-(1-p)\beta\right)} \eqqcolon C_0 $. Then, iterating the above inequality for $k=0,1,\dots,N-1$, multiplying both sides by $\frac{2}{N\gamma}$, and rearranging the terms, we obtain
\begin{align*}
    \frac{1}{N\gamma}\int_{0}^{N\gamma}\mathrm{FI}(\mu_t \| \pi)\,dt
    \le
    &\frac{2C_0}{N\gamma} + \frac{13\sigma^2}{2\bigl(1-(1-p)\beta\bigr)}\left[\frac{p}{b}+(1-p)(1-\beta)^2\right] + 
    \frac{63\gamma L_{\pi}^2d}{1-(1-p)\beta} + \bar{\sigma}^2 + \bar{\varepsilon}_{\sigma}^2 + \bar{\alpha}^2,
\end{align*} where $\bar{\sigma}^2$, $\bar{\varepsilon}_{\sigma}^2$, and $\bar{\alpha}^2$
are defined in (\ref{sl-sgm-prior-defs}). Note that we use $\phi(p,\beta)= 1 + \frac{2(1-p)}{1-(1-p)\beta} \le 1 + \frac{2}{1-(1-p)\beta} \le \frac{3}{1-(1-p)\beta}$. Since the definitions of $\gamma$, $p$, $b$, and $\sigma_m$ are the same as in Theorem~\ref{thm:sl-vrld-ficonv}, we have
\begin{align*}
    \mathrm{FI}(\bar{\mu}_{N\gamma} \| \pi)\le \frac{1}{N\gamma}\int_{0}^{N\gamma}\mathrm{FI}(\mu_t \| \pi)\,dt
    \le
    &~\BigO\left(\frac{\sigma_m^{2/3}L_\pi^{2/3}d^{2/3}}{N^{1/3}}\right) + \bar{\sigma}^2 + \bar{\varepsilon}_{\sigma}^2 + \bar{\alpha}^2.
\end{align*}We note that the annealing and noise schedules~(\ref{def:annealing-smoothing-schedules}), and the SGM estimation error are the same as in Theorem~\ref{thm:ficonv-SGM-prior}a. Therefore, we set $\sigma_{\mathrm{min}}=\BigO(N^{-1/6})$ and follow the same steps in~(\ref{thm3:sigma-bound-bias}),~(\ref{thm3:alpha-bound-bias}), and~(\ref{thm3:err-bound-bias}) to obtain 
\begin{equation*}
    \bar{\sigma}^2 + \bar{\varepsilon}_{\sigma}^2 + \bar{\alpha}^2 = \BigO\left(\frac{1}{N^{1/3}}\right). 
\end{equation*}This implies $\mathrm{FI}(\bar{\mu}_{N\gamma}\|\pi) = \BigO\left(\frac{\sigma_m^{2/3}L_\pi^{2/3}d^{2/3}}{N^{1/3}}\right)$, that is, $\mathrm{FI}(\bar{\mu}_{N\gamma}\|\pi) \le \varepsilon$ is achieved with $\BigO\left(\frac{\sigma_m^2L_\pi^2 d^2}{\varepsilon^3}\right)$ iterations and $pb+(1-p) = \BigO(1)$ gradient computations per iteration. 
\qedbox

\subsubsection{Proof of Corollary~\ref{proofs:theorem-and-lemmas}.2 (ML-VRLD with SGM prior: squared TV distance convergence)}

\paragraph{Corollary~\ref{proofs:theorem-and-lemmas}.2 (ML-VRLD with SGM prior: squared TV distance convergence)}\label{thm:ml-vrld-tvconv-SGM-prior}
\textit{
    Let $\pi(\vx|\vy)\propto \ell(\vy|\vx)p(\vx)$ be the posterior with the likelihood $\ell(\vy|\vx)\propto e^{-f(\vx)}$ and the prior $p(\vx)\propto e^{-h(\vx)}$. Suppose that the likelihood potential \(f\) satisfies Assumptions~\ref{ass:potential-lipschitz} and~\ref{ass:sfo} with Lipschitz constant $L_f$, the prior potential \(h\) satisfies Assumption~\ref{ass:potential-lipschitz} and~\ref{ass:smoothing-err-bounded} with Lipschitz constant $L_h$, the target posterior $\pi$ satisfies Assumption~\ref{ass:poincare-inequality}, and the SGM satisfies Assumption~\ref{ass:score-network} with decreasing error $\varepsilon_{\sigma_k} = \BigO(k^{-1/2})$ for $k\ge 1$. Define $L_\pi= L_f + L_h$. Let $(\mu_t)_{t\geq 0}$ denote the law of the interpolation generated by~(\ref{def:SGM-prior-interpolation}) with the estimator~(\ref{def:ml-vrld:estimator}) and the schedules defined in~(\ref{def:annealing-smoothing-schedules}). For any step size $\gamma \in \left(0, \frac{1}{L_\pi\sqrt{52\phi(p,\beta)}}\right]$, where $\phi(p,\beta)= 1+\frac{4(1-p)\beta^2}{(1-(1-p)\beta)^2}$, and for any $N\ge 1$, it holds that
    \begin{align}
        \|\Bar{\mu}_{N\gamma} - \pi\|_{\mathrm{TV}}^2\leq \frac{8C_0C_{\mathrm{PI}}}{N\gamma} + 52C_{\mathrm{PI}}\sigma^2\left[\frac{\frac{p}{b} + (1-p)(1-\beta)^2}{1-(1-p)\beta}\right] + \frac{336C_{\mathrm{PI}}\gamma L_\pi^2 d}{(1-(1-p)\beta)^2}+4C_{\mathrm{PI}}(\Bar{\sigma}^2 + \Bar{\varepsilon}_{\sigma}^2 + \Bar{\alpha}^2),
    \end{align}where 
    \begin{equation}
        \Bar{\sigma}^2 = \frac{39C_1^2}{2N}\sum_{k=1}^N\sigma_k^2, \quad \Bar{\varepsilon}^2_{\sigma} = \frac{39}{2N}\sum_{k=1}^N\varepsilon_{\sigma_k}^2, \quad \Bar{\alpha}^2 = \frac{39C_2^2}{2N}\sum_{k=1}^N\frac{(\alpha_k - 1)^2}{\sigma_k^2},
    \end{equation}}
    \textit{
    \begin{equation}
        C_0 = \mathrm{KL}(\mu_0\|\pi) + \frac{13\gamma}{2(1-(1-p)\beta)}\E[\|\nabla f(\vx_0) - \vg_0\|^2],
    \end{equation}and $C_{\mathrm{PI}},C_0,C_1,C_2$ are positive constants. Furthermore, let
    \[
    \gamma=\frac{\sqrt d}{\sigma_m N^{3/4}\sqrt{L_\pi}},\qquad 
    p=1-\beta=\frac{\sqrt{L_\pi d}}{\sigma_m N^{1/4}},\qquad 
    b=\Bigl\lceil \tfrac{1}{p}\Bigr\rceil,
    \quad \text{where } \sigma_m = \max\{\sigma,1\},
    \]
    and let the initial parameters satisfy $\sigma_0\ge\sigma_{\mathrm{min}}$, $\varepsilon_{\sigma_0}>0$, $\alpha_0\ge 1$, and $\sigma_{\mathrm{min}}> 0$ is the minimum noise level. Then, the time-averaged law $\bar{\mu}_{N\gamma}=\frac{1}{N\gamma}\int_{0}^{N\gamma}\mu_t\,dt$ achieves $\|\mu_{N\gamma} - \pi\|_{\mathrm{TV}}^2\le \varepsilon$ with$~ \BigO\!\left(\frac{\sigma_m^{4}C_{\mathrm{PI}}^4L_\pi^{2}d^{2}}{\varepsilon^{4}}\right)$ iterations and $\BigO(1)$ gradient computations per iteration.}

\textit{Proof.} 
Applying Theorem~\ref{thm:ficonv-SGM-prior} from Appendix~\ref{proof:ml-vrld:ficonv-SGM-prior}, we obtain
\begin{align*}
    \mathrm{FI}(\bar{\mu}_{N\gamma}\|\pi) \leq
        \frac{1}{N\gamma}\int_{0}^{N\gamma} \mathrm{FI}(\mu_t || \pi)\, dt\leq \frac{2C_0}{N\gamma} + 13\sigma^2\frac{\left(\frac{p}{b} + (1-p)(1-\beta)^2\right)}{(1-(1-p)\beta)} + \frac{84\gamma L_\pi^2 d}{(1-(1-p)\beta)^2}+\Bar{\sigma}^2 + \Bar{\varepsilon}_{\sigma}^2 + \Bar{\alpha}^2.
    \end{align*}
Using Lemma \ref{lemma:tv_lemma} with Assumption~\ref{ass:poincare-inequality} yields
\begin{align*}
    \|\Bar{\mu}_{N\gamma} - \pi\|_{\mathrm{TV}}^2\leq 
    \frac{8C_0C_{\mathrm{PI}}}{N\gamma} + 52C_{\mathrm{PI}}\sigma^2\frac{\left(\frac{p}{b} + (1-p)(1-\beta)^2\right)}{(1-(1-p)\beta)} + 
    \frac{336C_{\mathrm{PI}}\gamma L_\pi^2 d}{(1-(1-p)\beta)^2} + 4C_{\mathrm{PI}}\left(\Bar{\sigma}^2 + \Bar{\varepsilon}_{\sigma}^2 + \Bar{\alpha}^2\right)
\end{align*}
Using the definitions of $\gamma$, $p$, $\beta$, $\sigma_k$, $\alpha_k$, and $\varepsilon_{\sigma_k}$ from the corollary statement and following the steps in Appendix~\ref{proof:ml-vrld:ficonv-SGM-prior}, we prove that an iteration complexity of order $\BigO\!\left(\frac{C_{\mathrm{PI}}^{4}\sigma_m^{4}L_\pi^{2} d^{2}}{\varepsilon^{4}}\right)$ suffices to ensure $\|\bar{\mu}_{N\gamma} - \pi\|_{\mathrm{TV}}^{2} \le \varepsilon$ with $pb + (1-p) = \BigO(1)$ gradient computations per iteration.
\qedbox

\subsubsection{Proof of Corollary~\ref{proofs:theorem-and-lemmas}.3 (SL-VRLD with SGM prior: squared TV distance convergence)}

\paragraph{Corollary~\ref{proofs:theorem-and-lemmas}.3 (SL-VRLD with SGM prior: squared TV distance convergence)}\label{thm:sl-vrld-tvconv-SGM-prior}
\textit{
    Let $\pi(\vx|\vy)\propto \ell(\vy|\vx)p(\vx)$ be the posterior with the likelihood $\ell(\vy|\vx)\propto e^{-f(\vx)}$ and the prior $p(\vx)\propto e^{-h(\vx)}$. Suppose that the likelihood potential \(f\) satisfies Assumptions~\ref{ass:sfo} and~\ref{ass:stochastic-potential-lipschitz} with Lipschitz constant $L_f$, the prior potential \(h\) satisfies Assumption~\ref{ass:potential-lipschitz} and~\ref{ass:smoothing-err-bounded} with Lipschitz constant $L_h$, the target posterior $\pi$ satisfies Assumption~\ref{ass:poincare-inequality}, and the SGM satisfies Assumption~\ref{ass:score-network} with decreasing error $\varepsilon_{\sigma_k} = \BigO(k^{-1/2})$ for $k\ge 1$. Define $L_\pi= L_f + L_h$. Let $(\mu_t)_{t\geq 0}$ denote the law of the interpolation generated by~(\ref{def:SGM-prior-interpolation}) with the estimator~(\ref{def:sl-vrld:estimator}) and the schedules defined in~(\ref{def:annealing-smoothing-schedules}). For any step size $\gamma \in \left(0, \frac{1}{L_\pi\sqrt{52\phi(p,\beta)}}\right]$, where $\phi(p,\beta)= 1+\frac{2(1-p)}{1-(1-p)\beta}$, and for any $N\ge 1$, it holds that 
    \begin{align}
        \|\mu_{N\gamma} - \pi\|_{\mathrm{TV}}^2\leq \frac{8C_0C_{\mathrm{PI}}}{N\gamma} + 26C_{\mathrm{PI}}\sigma^2\left[\frac{\frac{p}{b} + 2(1-p)(1-\beta)^2}{1-(1-p)\beta}\right] + \frac{252C_{\mathrm{PI}}\gamma L_\pi^2 d}{1-(1-p)\beta}+4C_{\mathrm{PI}}(\Bar{\sigma}^2 + \Bar{\varepsilon}_{\sigma}^2 + \Bar{\alpha}^2),
    \end{align}where 
    \begin{equation}
        \Bar{\sigma}^2 =\frac{39C_1^2}{2N}\sum_{k=1}^N\sigma_k^2, \quad \Bar{\varepsilon}^2_{\sigma} =\frac{39}{2N}\sum_{k=1}^N\varepsilon_{\sigma_k}^2, \quad \Bar{\alpha}^2 = \frac{39C_2^2}{
        2N}\sum_{k=1}^N\frac{(\alpha_k - 1)^2}{\sigma_k^2},
    \end{equation}}
    \textit{
    \begin{equation}
        C_0 = \mathrm{KL}(\mu_0\|\pi) + \frac{13\gamma}{4(1-(1-p)\beta)^2}\E[\|\nabla f(\vx_0) - \vg_0\|^2],
    \end{equation}and $C_\mathrm{PI},C_0,C_1,C_2$ are positive constants. Furthermore, let
    \[
    \gamma=\left(\frac{d}{\sigma_m^2N^2L_\pi^2}\right)^{1/3},\qquad 
    p=1-\beta=\left(\frac{d^2L_\pi^2}{\sigma_m^4 N}\right)^{1/3},\qquad 
    b=\Bigl\lceil \tfrac{1}{p}\Bigr\rceil,
    \quad \text{where } \sigma_m = \max\{\sigma,1\},
    \]
    and let the initial parameters satisfy $\sigma_0\ge\sigma_{\mathrm{min}}$, $\varepsilon_{\sigma_0}>0$, $\alpha_0\ge 1$, and $\sigma_{\mathrm{min}}> 0$ is the minimum noise level. Then, the time-averaged law $\bar{\mu}_{N\gamma}= \frac{1}{N\gamma}\int_{0}^{N\gamma}\mu_t\,dt$ achieves $\|\Bar{\mu}_{N\gamma} - \pi\|_{\mathrm{TV}}^2\le \varepsilon$ with$~ \BigO\!\left(\frac{C_{\mathrm{PI}}^3\sigma_m^{2}L_\pi^{2}d^{2}}{\varepsilon^{3}}\right)$ iterations and $\BigO(1)$ gradient computations per iteration.}

\textit{Proof.} Applying Theorem~\ref{thm:ficonv-SGM-prior}b, we have
\begin{align*}
    \mathrm{FI}(\bar{\mu}_{N\gamma}\|\pi) \leq
        \frac{1}{N\gamma}\int_{0}^{N\gamma} \mathrm{FI}(\mu_t || \pi)\, dt\leq \frac{2C_0}{N\gamma} + \frac{13\sigma^2\left(\frac{p}{b} + 2(1-p)(1-\beta)^2\right)}{2(1-(1-p)\beta)} + \frac{63\gamma L_\pi^2 d}{1-(1-p)\beta}+\Bar{\sigma}^2 + \Bar{\varepsilon}_{\sigma}^2 + \Bar{\alpha}^2.
        \end{align*}
Using Lemma \ref{lemma:tv_lemma} with Assumption~\ref{ass:poincare-inequality} yields
\begin{align*}
    \|\Bar{\mu}_{N\gamma} - \pi\|_{\mathrm{TV}}^2\leq 
    \frac{8C_0C_{\mathrm{PI}}}{N\gamma} + \frac{26C_{\mathrm{PI}}\sigma^2\left(\frac{p}{b} + (1-p)(1-\beta)^2\right)}{(1-(1-p)\beta)} + 
    \frac{252C_{\mathrm{PI}}\gamma L_\pi^2 d}{1-(1-p)\beta} + 4C_{\mathrm{PI}}\left(\Bar{\sigma}^2 + \Bar{\varepsilon}_{\sigma}^2 + \Bar{\alpha}^2\right)
\end{align*}
Using the definitions of $\gamma$, $p$, $\beta$, $\sigma_k$, $\alpha_k$, and $\varepsilon_{\sigma_k}$ from the corollary statement and following the steps in Appendix~\ref{proof:sl-vrld:ficonv-SGM-prior}, we prove that an iteration complexity of order $\BigO\!\left(\frac{C_{\mathrm{PI}}^{3}\sigma_m^{2}L_\pi^{2} d^{2}}{\varepsilon^{3}}\right)$ suffices to ensure $\|\bar{\mu}_{N\gamma} - \pi\|_{\mathrm{TV}}^{2} \le \varepsilon$ with $pb + (1-p) = \BigO(1)$ gradient computations per iteration.
\qedbox 

\subsubsection{Proof of Theorem~\ref{thm:weakconv-SGM-prior}\textit{a} (ML-VRLD with SGM prior: weak convergence)}\label{proof:ml-vrld:weakconv-SGM-prior}
\paragraph{Theorem~\ref{thm:weakconv-SGM-prior}\textit{a} (ML-VRLD with SGM prior: weak convergence)}
\textit{
    Let $\pi(\vx|\vy)\propto \ell(\vy|\vx)p(\vx)$ be the posterior with the likelihood $\ell(\vy|\vx)\propto e^{-f(\vx)}$ and the prior $p(\vx)\propto e^{-h(\vx)}$. Suppose that the likelihood potential \(f\) satisfies Assumptions~\ref{ass:potential-lipschitz} and~\ref{ass:sfo} with Lipschitz constant $L_f$, the prior potential \(h\) satisfies Assumption~\ref{ass:potential-lipschitz} and~\ref{ass:smoothing-err-bounded} with Lipschitz constant $L_h$, and the SGM satisfies Assumption~\ref{ass:score-network} with decreasing error $\varepsilon_{\sigma_k} = \BigO(k^{-1/2})$ for
    $k\ge 1$. Let $(\mu_t)_{t\geq 0}$ denote the law of the interpolation generated by~(\ref{def:SGM-prior-interpolation}) with the estimator~(\ref{def:ml-vrld:estimator}) and the schedules defined in~(\ref{def:annealing-smoothing-schedules}). Define the time-varying parameters as follows
        \[
        \gamma_k=\frac{C_\gamma}{k^{3/2}},\qquad 
        p_k=\frac{C_p}{k^{1/2}},\qquad 
        b_k=\left\lceil \frac{1}{p_k}\right\rceil, \qquad \forall k\geq 1,
        \]
        where \(0<C_p<1\) and \(C_\gamma>0\) are numerical constants. 
        Then, the time-averaged law $\bar{\mu}_{\tau_n} = \frac{1}{\tau_n}\int_{0}^{\tau_n}\mu_t\,dt$, where $\tau_n = \sum_{k=1}^n\gamma_k$ converges weakly to $\pi$. }

\textit{Proof.} We use the same definitions of the time-varying parameters and the definition of the estimator $\vg_k$ as in the proof of Theorem~\ref{thm:ficonv-SGM-prior}\textit{a}. Therefore, the time-varying step sizes $(\gamma_k)_{k\ge 1}$ satisfy the step-size condition of Theorem~\ref{thm:ficonv-SGM-prior}\textit{a}, that is
\begin{equation*}
    \gamma \in\left(0, \frac{1}{L_\pi\sqrt{52\phi(p,\beta)}}\right], \quad \text{where} \quad \phi(p_k,\beta_k)=1 + \frac{4(1-p)\beta^2}{(1-(1-p)\beta)^2}
\end{equation*}for all $k\ge 1$. Thus, we can use the same upper bound in~(\ref{eq:ml-vrld-sgm-prior-kl-bound-for-id}) (Theorem~\ref{thm:ficonv-SGM-prior}\textit{a}; Appendix~\ref{proof:ml-vrld:ficonv-SGM-prior}) with time-varying parameters, and integrate both sides of the inequality over \([\tau_{{n-1}},\tau_{n}]\) to get 
\begin{align*}
    \mathrm{KL}(\mu_{\tau_n} \| \pi) - \mathrm{KL}(\mu_{\tau_{n-1}} \| \pi)
    \leq
    &-\frac{1}{2}\int_{\tau_{n-1}}^{\tau_n}\mathrm{FI}(\mu_t \| \pi) \,dt-
    \frac{\gamma_n\bigl(6+26\gamma_n^2L_{\pi}^2\phi(p_n, \beta_n)\bigr)}{1-(1-p_n)\beta_n}(e_{n}^2 - e_{n-1}^2)\\
    &+\frac{13\gamma_n\sigma^2}{2\bigl(1-(1-p_n)\beta_n\bigr)}\left[p_n^2+(1-p_n)(1-\beta_n)^2\right]
    +\frac{21}{2}\gamma_n^2 L_{\pi}^2d\phi(p_n, \beta_n) \\
    &+ \frac{39}{4}\gamma_n\sigma_n^2C_1^2 + \frac{39}{4}\gamma_n\varepsilon_{\sigma_n}^2 + \frac{39C_2^2\gamma_n(\alpha_n-1)^{2}}{4\sigma_n^2},
\end{align*}
where we use $\frac{p_k}{b_k} \leq p_k^2 $. Iterating the above bound, we obtain 
\begin{align}\label{eq:ml-vrld-sgm-prior-id-kl-bound}
\nonumber
\mathrm{KL}(\mu_{\tau_n} \| \pi) - \mathrm{KL}(\mu_{{0}} \| \pi)
\leq
&-\frac{1}{2}\int_{0}^{\tau_n}\mathrm{FI}(\mu_t \| \pi) \,dt 
\underbrace{- \sum_{k=1}^{n} \frac{\gamma_k\bigl(6+26\gamma_k^2L_{\pi}^2\phi(p_k, \beta_k)\bigr)}{1-(1-p_k)\beta_k}(e_{k}^2 - e_{k-1}^2)}_\textup{(I)} \notag\\
& + \underbrace{\frac{13\sigma^2}{2}\sum_{k=1}^{n} \gamma_k \left(\frac{p_k^2 + (1-p_k)(1-\beta_k)^2}{1-(1-p_k)\beta_k} \right)}_\textup{(II)} +
\underbrace{\frac{21}{2}L_{\pi}^2d \sum_{k=1}^{n} \gamma_k^2 \phi(p_k, \beta_k)}_\textup{(III)}\notag \\
&+
\underbrace{\frac{39C^2}{4} \sum_{k=1}^{n} \gamma_k\sigma_k^2}_\textup{(IV)}
+
\underbrace{\frac{39}{4} \sum_{k=1}^{n} \gamma_k \varepsilon_{\sigma_k}^2}_\textup{(V)}
+
\underbrace{\frac{39C^2}{4} \sum_{k=1}^{n} \frac{\gamma_k(\alpha_k - 1)^2}{\sigma_k^2}}_\textup{(VI)}
\end{align}
In the proof of Theorem~\ref{thm:ml-vrld-weakconv} (Appendix~\ref{proof:ml-vrld-weakconv}), we show that with the same parameters for $\gamma_k$, $p_k$, $\beta_k$, $b_k$, we have
\begin{equation}\label{ml-vrld-sgm-prior-term-i}
\mathrm{(I)} = - \sum_{k=1}^{n} \frac{\gamma_k\bigl(6+26\gamma_k^2L_{\pi}^2\phi(p_k, \beta_k)\bigr)}{1-(1-p_k)\beta_k}(e_{k}^2 - e_{k-1}^2) \leq \underbrace{\frac{\gamma_1\bigl(6+26\gamma_1^2L_{\pi}^2\phi(p_1, \beta_1)\bigr)}{1-(1-p_1)\beta_1}}_{\eqqcolon c_1}e_0^2< \infty
\end{equation}
\begin{equation}\label{ml-vrld-sgm-prior-term-ii}
\mathrm{(II)} \le
\frac{13\sigma^2}{2}\sum_{k=1}^{\infty} \gamma_k \left(\frac{p_k^2 + (1-p_k)(1-\beta_k)^2}{1-(1-p_k)\beta_k} \right) \eqqcolon S_2 < \infty
\end{equation}
\begin{equation}\label{ml-vrld-sgm-prior-term-iii}
    \mathrm{(III)} \le
\frac{21L_{\pi}^2d}{2} \sum_{k=1}^{\infty} \gamma_k^2 \phi(p_k, \beta_k) \eqqcolon S_3
< \infty
\end{equation}
We next show that the terms $\mathrm{(IV)}$, $\mathrm{(V)}$, and $\mathrm{(VI)}$ are uniformly bounded. 

To bound the bias due to the noise schedule $(\sigma_k)_{k=1}^{n}$, we proceed as
\begin{align}
\mathrm{(IV)}=\frac{39C_1^2}{4} \sum_{k=1}^{n} \gamma_k\sigma_k^2
\le
\frac{39C_1^2C_{\gamma}}{4} \sum_{k=1}^{\infty} \frac{\max\{\sigma_0^2\rho_2^{2k}, \sigma_{\mathrm{min}}^2\}}{k^{3/2}}
\overset{(*)}{\le}
\frac{39C_1^2C_\gamma \sigma_0^2}{4} \sum_{k=1}^{\infty} k^{-3/2} \eqqcolon A_1 < \infty, \quad \label{ml-vrld-sgm-prior-term-iv}
\end{align}where we use the facts that $\sigma_0 \ge \sigma_{\mathrm{min}}$ and $\rho_2<1$ to get~$(*)$.

To bound the bias due to the SGM estimation errors $(\varepsilon_{\sigma_k})_{k=1}^n$, we use the fact that $\varepsilon_{\sigma_k}\le \frac{C_\varepsilon}{k^{-1/2}}$  for some $C_\varepsilon>0$ and proceed as 
\begin{align}
\mathrm{(V)} = \frac{39}{4} \sum_{k=1}^{n} \gamma_k \varepsilon_{\sigma_k}^2
\leq \frac{39}{4} \sum_{k=1}^{\infty} \gamma_k \varepsilon_{\sigma_k}^2
\leq
\frac{39C_\gamma C_\varepsilon^2}{4} \sum_{k=1}^{\infty} k^{-5/2} \eqqcolon A_2 < \infty. \quad \label{ml-vrld-sgm-prior-term-v}
\end{align}

To bound the bias due to the annealing schedule $(\sigma_k)_{k=1}^n$, we follow
\begin{align}
    \mathrm{(VI)} = \frac{39C_2^2}{4} \sum_{k=1}^{n} \frac{\gamma_k(\alpha_k - 1)^2}{\sigma_k^2}
    &= \frac{39C_2^2C_\gamma}{4} \sum_{k=1}^\infty\frac{1}{k^{\frac{3}{2}}} \frac{\mathrm{max}\{\alpha_0\rho_1^k-1,0\}^2}{\max\{\sigma_0^2\rho_2^{2k}, \sigma_{\mathrm{min}}^2\}} \notag \\
    & \overset{(*)}{\leq} \frac{39C_2^2C_\gamma(\alpha_0^2-1)}{4\sigma_{\mathrm{min}}^2}\sum_{k=1}^\infty k^{-3/2}  \eqqcolon A_3
    < \infty, \label{ml-vrld-sgm-prior-term-vi}
\end{align}where we use the facts that $\alpha_0\ge 1$, $\rho_1<1$, $\rho_2<1$, and $\sigma_0\ge \sigma_{\mathrm{min}}>0$ to obtain~$(*)$.
Plugging the bounds~$\mathrm{(I)}$--$\mathrm{(VI)}$ into~(\ref{eq:ml-vrld-sgm-prior-id-kl-bound}), using Jensen's inequality with the convexity of $\mathrm{FI}(\cdot\|\pi)$, and rearranging the terms, we have
\begin{align}\label{eq:ml-vrld-sgm-prior-id-avg-law-fi-unique}
\mathrm{FI}(\bar \mu_{\tau_n} \| \pi) \leq
\frac{1}{\tau_n}\int_{0}^{\tau_n}\mathrm{FI}(\mu_t \| \pi)\,dt
&\le
\frac{2\mathrm{KL}(\mu_{0} \| \pi)}{\tau_n} + 
\frac{2\left(c_1e_0^2 + S_2+S_3+A_1+A_2+A_3\right)}{\tau_n},
\end{align} 
where $\bar\mu_{\tau_n} \coloneqq \frac{1}{\tau_n}\int_0^{\tau_n}\mu_t\,dt$ is the time-averaged law up to time $\tau_n=\sum_{k=1}^n\gamma_k$. 

On the other hand, if $t\in[\tau_n,\tau_{n+1}]$, integrating~(\ref{eq:ml-vrld-sgm-prior-kl-bound-for-id}) (Theorem~\ref{thm:ficonv-SGM-prior}\textit{a}; Appendix~\ref{proof:ml-vrld:ficonv-SGM-prior}) from $\tau_n$ to $t$ and removing the negative Fisher information term yields
\begin{align}
    \mathrm{KL}(\mu_{t} \| \pi) 
    &\leq
    \mathrm{KL}(\mu_{\tau_{n}} \| \pi)
    -\frac{(t-\tau_{n})\bigl(6+26\gamma_{n+1}^2L_{\pi}^2\phi(p_{n+1}, \beta_{n+1})\bigr)}{1-(1-p_{n+1})\beta_{n+1}}(e_{n+1}^2 - e_n^2)\notag\\
    &\qquad\qquad\qquad +\frac{13(t-\tau_{n})\sigma^2}{2\bigl(1-(1-p_{n+1})\beta_{n+1}\bigr)}\left[\frac{p_{n+1}}{b_{n+1}}+(1-p_{n+1})(1-\beta_{n+1})^2\right] \notag \\
    &\qquad\qquad\qquad+\frac{21}{2}(t-\tau_n)\gamma_{n+1} L_{\pi}^2d\phi(p_{n+1}, \beta_{n+1}) + \frac{39}{4}(t-\tau_n)\sigma_{n+1}^2C_1^2 + \frac{39}{4}(t-\tau_n)\varepsilon_{\sigma_{n+1}}^2 \notag \\
    &\qquad\qquad\qquad+ \frac{39C_2^2(t-\tau_n)(\alpha_{n+1}-1)^{2}}{4\sigma_{n+1}^2} \notag\\
    &\le \mathrm{KL}(\mu_{\tau_n}\|\pi) -\frac{(t-\tau_{n})\bigl(6+26\gamma_{n+1}^2L_{\pi}^2\phi(p_{n+1}, \beta_{n+1})\bigr)}{1-(1-p_{n+1})\beta_{n+1}}(e_{n+1}^2 - e_n^2) + S_2 + S_3 + A_1 + A_2 + A_3 \notag \\ 
    & \le \mathrm{KL}(\mu_0\|\pi) - \sum_{k=1}^nc_{k}(e_k^2 - e_{k-1}^2) - c'_{n+1}(e_{n+1}^2 - e_n^2)+ 2S_2 + 2S_3 + 2A_1 + 2A_2 + 2A_3, \label{kl-t-upper-bound-terms}
\end{align}where
\begin{equation*}
    c'_{n+1}\coloneqq \frac{(t-\tau_{n})\bigl(6+26\gamma_{n+1}^2L_{\pi}^2\phi(p_{n+1}, \beta_{n+1})\bigr)}{1-(1-p_{n+1})\beta_{n+1}}.
\end{equation*}As shown in~(\ref{thm2:decreasing-cn-e0}) (Theorem~\ref{thm:ml-vrld-weakconv}; Appendix~\ref{proof:ml-vrld-weakconv}), $(c_k)_{k=1}^n$ is a nonnegative and decreasing sequence. We note that $c_n > c_{n+1}\ge c'_{n+1}$. Using this property, we have 
\begin{equation*}
    -\sum_{k=1}^nc_k(e_k^2 - e_{k-1}^2) - c'_{n+1}(e_{n+1}^2 - e_n^2) = c_1 e_0^2 + \sum_{k=1}^{n-1}(c_{k+1} - c_k)e_k^2 + (c'_{n+1} - c_n)e_n^2 - c'_{n+1}e_{n+1} \le c_1e_0^2.
\end{equation*}Applying this upper bound to~(\ref{kl-t-upper-bound-terms}), we obtain
\begin{align*}
    \mathrm{KL}(\mu_{t} \| \pi) 
    &\le 
    \mathrm{KL}(\mu_{0} \| \pi) + c_1e_0^2 + 2S_2 + 2S_3 + 2A_1 + 2A_2 + 2A_3 < \infty. 
\end{align*}Note that the initial mean estimation error $e_0^2 = \E[\|\vg_0 - \nabla f(\vx_0)\|^2]$ is bounded.
This proves $\mathrm{KL}(\mu_t \| \pi)$ is uniformly bounded for all $t \ge 0$. By the convexity of the $\mathrm{KL}$ divergence, it implies that $\{\mathrm{KL(\bar\mu_{\tau_n} \| \pi)}\,|\,n\in\mathbb{N}\}$ is uniformly bounded as well. Since the sub-level sets of $\mathrm{KL}(\cdot \|\pi)$ are weakly compact, $(\bar \mu_{\tau_n})_{n\in\mathbb{N}}$ is tight. To show that $\bar \mu_{\tau_n} \rightarrow \pi$ weakly, it is sufficient to show that every cluster point of $(\bar \mu_{\tau_n})_{n\in\mathbb{N}}$ is equal to $\pi$. 

Consider a subsequence of $(\bar{\mu}_{\tau_n})_{n\in\sN}$ converging to some limit $\bar{\mu}$. As $n\rightarrow\infty$, $\tau_n \rightarrow \infty$, and by~(\ref{eq:ml-vrld-sgm-prior-id-avg-law-fi-unique}) we also have $\mathrm{FI}(\bar \mu_{\tau_n} \| \pi) \rightarrow0$. Therefore, this is still true along the subsequence. By the weak lower semi-continuity of the Fisher information along the subsequence, $\mathrm{FI}(\bar \mu_{\tau_n} \| \pi)=0$. That implies; for $\psi\coloneqq \frac{d\mu}{d\pi}$, we have $\sqrt{\psi} \in \mathrm{dom(\mathcal{E})}$ and $\mathcal{E}(\sqrt{\psi}) = 0$, where $\mathcal{E}$ denotes the Dirichlet energy (i.e. the squared $L^2(\pi)$-norm of the gradient; see Section 3 in~\citet{balasubramanian2022towards}). Since $\nabla f$ is Lipschitz by Assumption \ref{ass:potential-lipschitz}, then $\pi$ has a continuous and strictly positive density on $\mathbb{R}^d$, so $\mathcal{E}(\sqrt{\psi}) = 0$ implies that $\psi$ is a constant $\pi$ with probability 1. Hence, $\bar \mu = \pi$ as $n\rightarrow\infty$.
\qedbox 

\subsubsection{Proof of Theorem~\ref{thm:weakconv-SGM-prior}\textit{b} (SL-VRLD with SGM prior: weak convergence)}\label{proof:sl-vrld:weakconv-SGM-prior}

\paragraph{Theorem~\ref{thm:weakconv-SGM-prior}\textit{b} (SL-VRLD with SGM prior: weak convergence)}
    \textit{
    Let $\pi(\vx|\vy)\propto \ell(\vy|\vx)p(\vx)$ be the posterior with the likelihood $\ell(\vy|\vx)\propto e^{-f(\vx)}$ and the prior $p(\vx)\propto e^{-h(\vx)}$. Suppose that the likelihood potential \(f\) satisfies Assumptions~\ref{ass:sfo} and~\ref{ass:stochastic-potential-lipschitz} with Lipschitz constant $L_f$, the prior potential \(h\) satisfies Assumption~\ref{ass:potential-lipschitz} and~\ref{ass:smoothing-err-bounded} with Lipschitz constant $L_h$, and the SGM satisfies Assumption~\ref{ass:score-network} with decreasing error $\varepsilon_{\sigma_k} = \BigO(k^{-1/2})$ for
    $k\ge 1$. Let $(\mu_t)_{t\geq 0}$ denote the law of the interpolation generated by~(\ref{def:SGM-prior-interpolation}) with the estimator~(\ref{def:sl-vrld:estimator}) and the schedules defined in~(\ref{def:annealing-smoothing-schedules}). Define the time-varying parameters as follows
        \[
        \gamma_k=\frac{C_\gamma}{k^{3/2}},\qquad 
        p_k=\frac{C_p}{k^{1/2}},\qquad 
        b_k=\left\lceil \frac{1}{p_k}\right\rceil, \quad \forall k\ge 1,
        \]
        where $0<C_p <1$ and $C_\gamma >0$ are numerical constants. Then, the time-averaged law $\bar{\mu}_{\tau_n}= \frac{1}{\tau_n}\int_{0}^{\tau_n} \mu_t\,dt$, where $\tau_n= \sum_{k=1}^n\gamma_k$ converges weakly to $\pi$. }

\textit{Proof.} We use the same definitions of the time-varying parameters and the definition of the estimator $\vg_k$ as in the proof of Theorem~\ref{thm:ficonv-SGM-prior}\textit{b}. Therefore, the time-varying step sizes $(\gamma_k)_{k\ge 1}$ satisfy the step-size condition of Theorem~\ref{thm:ficonv-SGM-prior}\textit{b}, that is
\begin{equation*}
    \gamma \in\left(0, \frac{1}{L_\pi\sqrt{52\phi(p,\beta)}}\right], \quad \text{where} \quad \phi(p_k,\beta_k)=1 + \frac{4(1-p)\beta^2}{(1-(1-p)\beta)^2}
\end{equation*}for all $k\ge 1$.Thus, we can use the same upper bound in~(\ref{eq:sl-vrld-sgm-prior-kl-bound-for-id}) (Theorem~\ref{thm:ficonv-SGM-prior}\textit{b}; Appendix~\ref{proof:sl-vrld:ficonv-SGM-prior}), integrate both sides of the inequality over $[\tau_{n-1},\tau_n]$ to get
\begin{align*}
    \mathrm{KL}(\mu_{\tau_n} \| \pi) - \mathrm{KL}(\mu_{\tau_{n-1}} \| \pi)
    \leq
    &-\frac{1}{2}\int_{\tau_{n-1}}^{\tau_n}\mathrm{FI}(\mu_t \| \pi) \,dt-
    \frac{\gamma_n\bigl(3+13\gamma_n^2L_{\pi}^2\phi(p_n, \beta_n)\bigr)}{1-(1-p_n)\beta_n}(e_{n}^2 - e_{n-1}^2)\\
    &+\frac{13\gamma_n\sigma^2}{4\bigl(1-(1-p_n)\beta_n\bigr)}\left[p_n^2+2(1-p_n)(1-\beta_n)^2\right]
    +\frac{21}{2}\gamma_n^2 L_{\pi}^2d\phi(p_n, \beta_n) \\
    &+ \frac{39}{4}\gamma_n\sigma_n^2C_1^2 + \frac{39}{4}\gamma_n\varepsilon_{\sigma_n}^2 + \frac{39C_2^2\gamma_n(\alpha_n-1)^{2}}{4\sigma_n^2},
\end{align*}
where we use the fact that $b_n\ge 1/p_n$. Iterating the above bound, we obtain 
\begin{align}\label{eq:sl-vrld-sgm-prior-id-kl-bound}
\nonumber
\mathrm{KL}(\mu_{\tau_n} \| \pi) - \mathrm{KL}(\mu_{0} \| \pi)
\leq
&-\frac{1}{2}\int_{0}^{\tau_n}\mathrm{FI}(\mu_t \| \pi) \,dt
\underbrace{- \sum_{k=1}^{n} \frac{\gamma_k\bigl(3+13\gamma_k^2L_{\pi}^2\phi(p_k, \beta_k)\bigr)}{1-(1-p_k)\beta_k}(e_{k}^2 - e_{k-1}^2)}_\textup{(I)}\\
\nonumber
&+\underbrace{\frac{13\sigma^2}{4}\sum_{k=1}^{n} \gamma_k \left(\frac{p_k^2 + 2(1-p_k)(1-\beta_k)^2}{1-(1-p_k)\beta_k} \right)}_\textup{(II)}+
\underbrace{\frac{21}{2}L_{\pi}^2d \sum_{k=1}^{n} \gamma_k^2 \phi(p_k, \beta_k)}_\textup{(III)}\\
&+
\underbrace{\frac{39C_1^2}{4} \sum_{k=1}^{n} \gamma_k\sigma_k^2}_\textup{(IV)}
+
\underbrace{\frac{39}{4} \sum_{k=1}^{n} \gamma_k\varepsilon_{\sigma_k}^2}_\textup{(V)}
+
\underbrace{\frac{39C_2^2}{4} \sum_{k=1}^{n} \frac{\gamma_k(\alpha_k - 1)^2}{\sigma_k^2}}_\textup{(VI)}. 
\end{align}In the proof of Theorem~\ref{proofs:theorem-and-lemmas}.1 (Appendix~\ref{proof:sl-vrld-weakconv}), we show that with the same parameters for $\gamma_k,p_k,\beta_k, b_k$, we have
\begin{equation}\label{sl-vrld-sgm-prior-term-i}
\mathrm{(I)} = -\sum_{k=1}^{n} \frac{\gamma_k\bigl(3+13\gamma_k^2L_{\pi}^2\phi(p_k, \beta_k)\bigr)}{1-(1-p_k)\beta_k}(e_{k}^2 - e_{k-1}^2) \leq \underbrace{\frac{\gamma_1\bigl(3+13\gamma_1^2L_{\pi}^2\phi(p_1, \beta_1)\bigr)}{1-(1-p_1)\beta_1}}_{\eqqcolon c_1}e_0^2< \infty
\end{equation}
\begin{equation}\label{sl-vrld-sgm-prior-term-ii}
\mathrm{(II)} \le
\frac{13\sigma^2}{4}\sum_{k=1}^{\infty} \gamma_k \left(\frac{p_k^2 + (1-p_k)(1-\beta_k)^2}{1-(1-p_k)\beta_k} \right) \eqqcolon S_2 < \infty
\end{equation}
\begin{equation}\label{sl-vrld-sgm-prior-term-iii}
    \mathrm{(III)} \le
\frac{21L_{\pi}^2d}{2} \sum_{k=1}^{\infty} \gamma_k^2 \phi(p_k, \beta_k) \eqqcolon S_3
< \infty
\end{equation}
Following the same steps in the proof of Theorem~\ref{thm:weakconv-SGM-prior}\textit{a} (Appendix~\ref{proof:ml-vrld:weakconv-SGM-prior}), we have
\begin{align}
\mathrm{(IV)} = \frac{39C_1^2}{4} \sum_{k=1}^{n} \gamma_k\sigma_k^2
=
\frac{39C_1^2C_\gamma \sigma_0^2}{4} \sum_{k=1}^{\infty} k^{-3/2} \eqqcolon A_1 < \infty,\label{eq:sl-vrld-weak-convergence-term-iv}
\end{align}
\begin{align}
\mathrm{(V)} = \frac{39}{4} \sum_{k=1}^{n} \gamma_k \varepsilon_{\sigma_k}^2
\leq
\frac{39C_\gamma C_\varepsilon^2}{4} \sum_{k=1}^{\infty} k^{-5/2}\eqqcolon A_2 < \infty, \label{eq:sl-vrld-weak-convergence-term-v}
\end{align}
\begin{align}
    \mathrm{(VI)} = \frac{39C_2^2}{4} \sum_{k=1}^{\infty} \frac{\gamma_k(\alpha_k - 1)^2}{\sigma_k^2}
    & \leq \frac{39C_2^2C_\gamma(\alpha_0^2 - 1)}{4\sigma_{\mathrm{min}}^2}\sum_{k=1}^\infty k^{-3/2}\eqqcolon A_3
    < \infty, \label{eq:sl-vrld-weak-convergence-term-vi}
\end{align}
Applying these upper bounds to~(\ref{eq:sl-vrld-sgm-prior-id-kl-bound}), dropping the negative $\mathrm{KL}$ term, and multiplying both sides by $2/\tau_n$, we obtain
\begin{align}\label{eq:ml-vrld-sgm-prior-id-avg-law-fi}
\mathrm{FI}(\bar \mu_{\tau_n} \| \pi) \leq
\frac{1}{\tau_n}\int_{0}^{\tau_n}\mathrm{FI}(\mu_t \| \pi)\,dt
&\le
\frac{2\mathrm{KL}(\mu_{0} \| \pi)}{\tau_n} + 
\frac{2\left(c_1e_0^2 + S_2+S_3+A_1+A_2+A_3\right)}{\tau_n},
\end{align} 
where $\bar\mu_{\tau_n} \coloneqq \frac{1}{\tau_n}\int_0^{\tau_n}\mu_t\,dt$ and $\tau_n\coloneqq \sum_{k=1}^n\gamma_k$.  

On the other hand, if $t\in[\tau_n,\tau_{n+1}]$, integrating~(\ref{eq:sl-vrld-sgm-prior-kl-bound-for-id}) (Theorem~\ref{thm:ficonv-SGM-prior}\textit{b}; Appendix~\ref{proof:sl-vrld:ficonv-SGM-prior}) from $\tau_n$ to $t$ and removing the negative Fisher information term yields
\begin{align}
    \mathrm{KL}(\mu_{t} \| \pi) 
    &\leq
    \mathrm{KL}(\mu_{\tau_{n}} \| \pi)
    -\frac{(t-\tau_n)\bigl(3+13\gamma_{n+1}^2L_{\pi}^2\phi(p_{n+1}, \beta_{n+1})\bigr)}{1-(1-p_{n+1})\beta_{n+1}}(e_{n+1}^2 - e_n^2)\notag\\
    &\qquad\qquad\qquad+\frac{13(t-\tau_n)\sigma^2}{4\bigl(1-(1-p_{n+1})\beta_{n+1}\bigr)}\left[p_{n+1}^2+2(1-p_{n+1})(1-\beta_{n+1})^2\right]
    \notag\\
    &\qquad\qquad\qquad+\frac{21}{2}(t-\tau_n)\gamma_{n+1} L_{\pi}^2d\phi(p_{n+1}, \beta_{n+1})+ \frac{39}{4}(t-\tau_n)\sigma_{n+1}^2C_1^2 + \frac{39}{4}(t-\tau_n)\varepsilon_{\sigma_{n+1}}^2 \notag\\
    &\qquad\qquad\qquad+ \frac{39C_2^2(t-\tau_n)(\alpha_{n+1}-1)^{2}}{4\sigma_{n+1}^2}.\notag\\
    &\leq
    \mathrm{KL}(\mu_{\tau_{n}} \| \pi)
    -\frac{(t-\tau_n)\bigl(3+13\gamma_{n+1}^2L_{\pi}^2\phi(p_{n+1}, \beta_{n+1})\bigr)}{1-(1-p_{n+1})\beta_{n+1}}(e_{n+1}^2 - e_n^2) + S_2 + S_3 + A_1 + A_2 + A_3\notag \\
    &\leq
    \mathrm{KL}(\mu_0 \| \pi)
    -\sum_{k=1}^nc_k(e_k^2 - e_{k-1}^2) - c'_{n+1}(e_{n+1}^2-e_n^2) + 2S_2 + 2S_3 + 2A_1 + 2A_2 + 2A_3,\label{sl-vrld-kl-bound-lastproof}
\end{align}where 
\begin{equation*}
    c'_{n+1} \coloneqq \frac{(t-\tau_n)\bigl(3+13\gamma_{n+1}^2L_{\pi}^2\phi(p_{n+1}, \beta_{n+1})\bigr)}{1-(1-p_{n+1})\beta_{n+1}}.
\end{equation*}As shown in~(\ref{sl-vrld-weak-term-i-is-bounded}) (Theorem~\ref{proofs:theorem-and-lemmas}.1; Appendix~\ref{thm:sl-vrld-weakconv}), $(c_k)_{k=1}^n$ is a nonnegative and decreasing sequence. We note that $c_n > c_{n+1}\ge c'_{n+1}$. Using this property, we have 
\begin{equation*}
    -\sum_{k=1}^nc_k(e_k^2 - c_{k-1}^2) - c'_{n+1}(e_{n+1}^2 - e_n^2) = c_1e_0^2 + \sum_{k=1}^{n-1}(c_{k+1} - c_k)e_k^2+ (c'_{n+1} - c_n)e_n^2 - c'_{n+1}e_{n+1}\le c_1 e_0^2. 
\end{equation*}Applying this upper bound to~(\ref{sl-vrld-kl-bound-lastproof}), we obtain
\begin{equation*}
    \mathrm{KL}(\mu_t\|\pi)\le \mathrm{KL}(\mu_0 \| \pi)
    + c_1e_0^2 + 2S_2 + 2S_3 + 2A_1 + 2A_2 + 2A_3 < \infty. 
\end{equation*}Thus $\{\mathrm{KL(\mu_t \| \pi)}\,|\,t\geq 0\}$ is bounded. Following the same argument as in the final part of the proof of Theorem~\ref{thm:weakconv-SGM-prior}~\textit{a} (Appendix~\ref{proof:sl-vrld:weakconv-SGM-prior}), we conclude that $\bar{\mu}_{\tau_n}$ converges weakly to $\pi$ as $n\to\infty$. \qedbox

\section{ADDITIONAL EXPERIMENTAL DETAILS}\label{appendix:additional-experiments}

We provide additional experimental details and results for the synthetic experiments, MRI reconstruction, and sparse-angle CT reconstruction tasks. Our code is available on the \href{https://github.com/mberk-sahin/variance-reduced-sampling.git}{\textcolor{blue}{project page}.}

%\subsection{Synthetic Gaussian Mixture Model Sampling}\label{appendix:synthetic-gmm8}

\subsection{Synthetic Experiments}\label{appendix:synthetic-experiments}

\begin{figure}[t]
    \centering
    \includegraphics[width=0.95\linewidth]{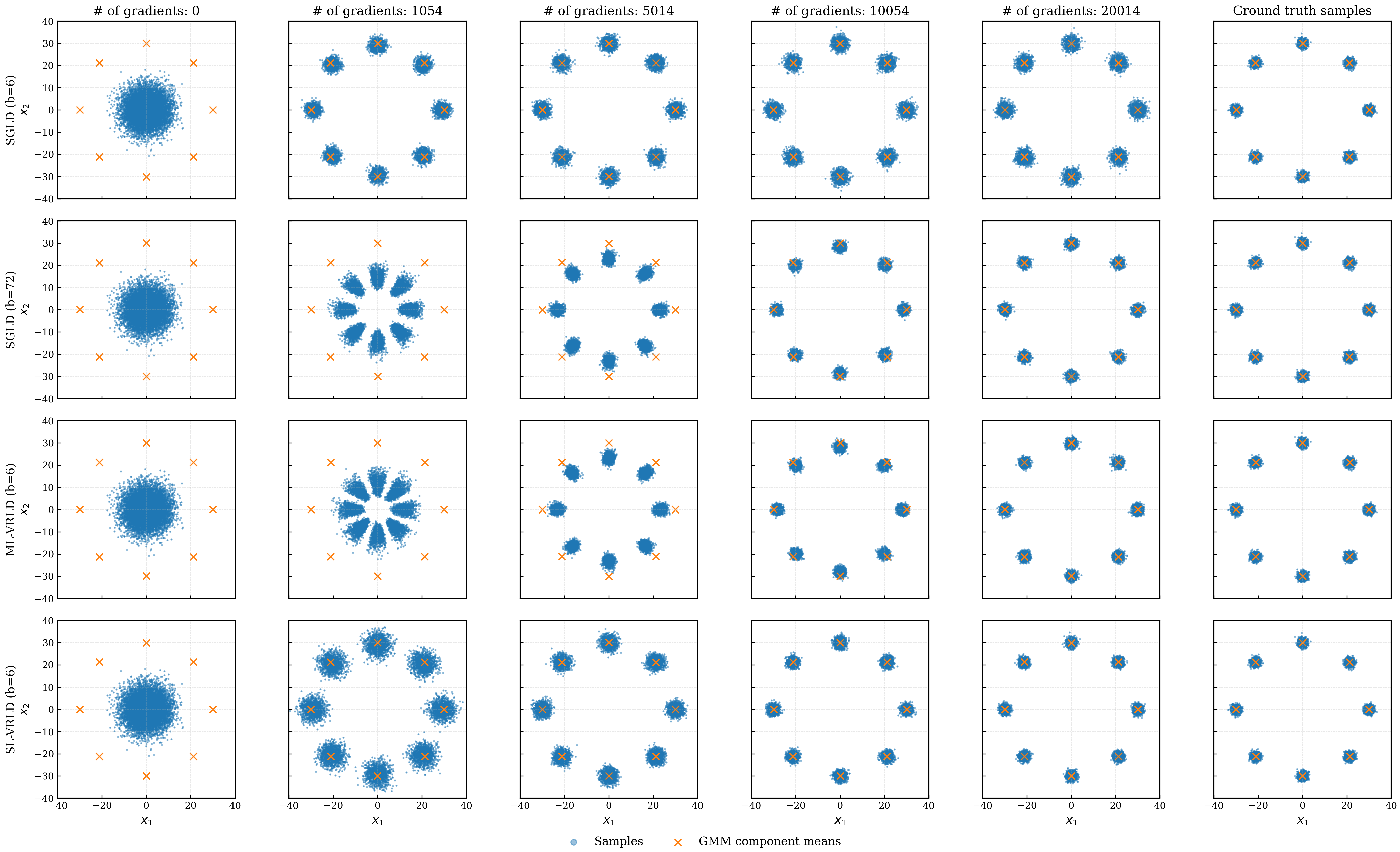}
    \caption{Visualization of samples generated by SGLD, ML-VRLD, and SL-VRLD on the 8-mode Gaussian mixture model at different numbers of gradient computations. Rows correspond to different methods and batch sizes, while columns show the evolution of the generated samples as the total number of gradient computations increases. The final column displays ground truth samples from the target GMM. SGLD with a fixed batch size ($b=6$) remains trapped in a suboptimal distribution and fails to recover the target modes. Increasing the SGLD batch size to $b=72$ enables convergence to the target distribution, but requires substantially more memory. In contrast, both ML-VRLD and SL-VRLD recover all target modes using the same fixed batch size ($b=6$), demonstrating their ability to operate under fixed memory requirements. Moreover, SL-VRLD reaches the target distribution with fewer gradient computations, consistent with its improved theoretical convergence rate.
}
    %\caption{Comparison of relative Fisher information convergence as a function of the total number of gradient computations for SGLD with batch sizes $b=6$ and $b=72$, and ML-VRLD and SL-VRLD with $b=6$. SL-VRLD achieves faster convergence than SGLD and ML-VRLD. Under a fixed batch size (memory budget) of $b=6$, ML-VRLD and SL-VRLD converge to a lower relative FI values, whereas SGLD converges to a suboptimal level. Increasing the SGLD batch size to $b=72$ enables comparable accuracy, but at the cost of substantially higher memory requirements.}
    \label{fig:gmm-visualization}
\end{figure}

%\subsection{Synthetic Inverse Problems}\label{appendix:synthetic-inverse-problems}

In this section, we give additional results and experimental details on our synthetic experiments, including computation of Fisher information, hyperparameters, and preprocessing steps. 

\textbf{Numerical Validation with GMM Sampling. }We run SGLD, ML-VRLD and SL-VRLD estimators with the following iterations:
\begin{equation}
    \vx_{(k+1)\gamma} = \vx_{k\gamma} - \gamma\vg_k + \sqrt{2\gamma}\mZ_k,
\end{equation}where $\mZ_k\sim \gN(0,I)$ and $\vg_k$ is defined according to the corresponding estimator definitions. All algorithms are run for $N=2000$ iterations to generate $10^4$ samples, which are initialized over $[-50,50]^2$ grid by sampling from a uniform distribution of $\mathcal{U}([-50,50]^2)$. We use a step size of $\gamma=0.02$ for SGLD and SL-VRLD, and $\gamma=0.0015$ for ML-VRLD. We set the momentum parameter to $\beta=0.9$ for ML-VRLD and $\beta=0.999$ for SL-VRLD, and set $p=0$ for both algorithm. We note that jointly using $\beta < 1$ and $p < 1$ does not improve the theoretical convergence rates. These parameters are combined primarily to provide a unified analysis of the underlying algorithms. Therefore, in our experiments, we set $p=0$. We approximately compute the relative Fisher information at each iteration. Specifically, we use kernel density estimation (KDE) to estimate the sampling distribution from the generated samples of each algorithm and then evaluate the relative Fisher information with respect to the target GMM using ground truth samples. In our experiments, we assign equal probability 1/8 to each mixture component and simulate stochastic gradients by adding Gaussian noise with standard deviation 30 to the gradient of the potential function.

In addition to the quantitative results in Fig.~\ref{fig:fi-convergence-gmm}, Fig.~\ref{fig:gmm-visualization} provides a qualitative comparison between the samples generated by each algorithm and ground truth samples drawn from the target GMM. SGLD with batch size $b=6$ fails to accurately recover the target distribution, exhibiting a more dispersed sample distribution than the ground truth. This behavior is consistent with the $\BigO(Ld\sigma^2/\varepsilon)$ batch size requirement of SGLD, which necessitates larger batch sizes to attain higher accuracy. Indeed, increasing the SGLD batch size to $b=72$ enables recovery of the target distribution. In contrast, both ML-VRLD and SL-VRLD closely match the ground truth distribution while operating with the same fixed batch size ($b=6$), highlighting their ability to achieve convergence under fixed memory requirements. Moreover, SL-VRLD reaches the target distribution with fewer gradient computations than SGLD, consistent with their respective computational complexities of $\BigO(L^2d^2\sigma^2/\varepsilon^3)$ and $\BigO(L^3d^3\sigma^2/\varepsilon^3)$.

\begin{figure}[t]
    \centering
    \includegraphics[width=0.5\linewidth]{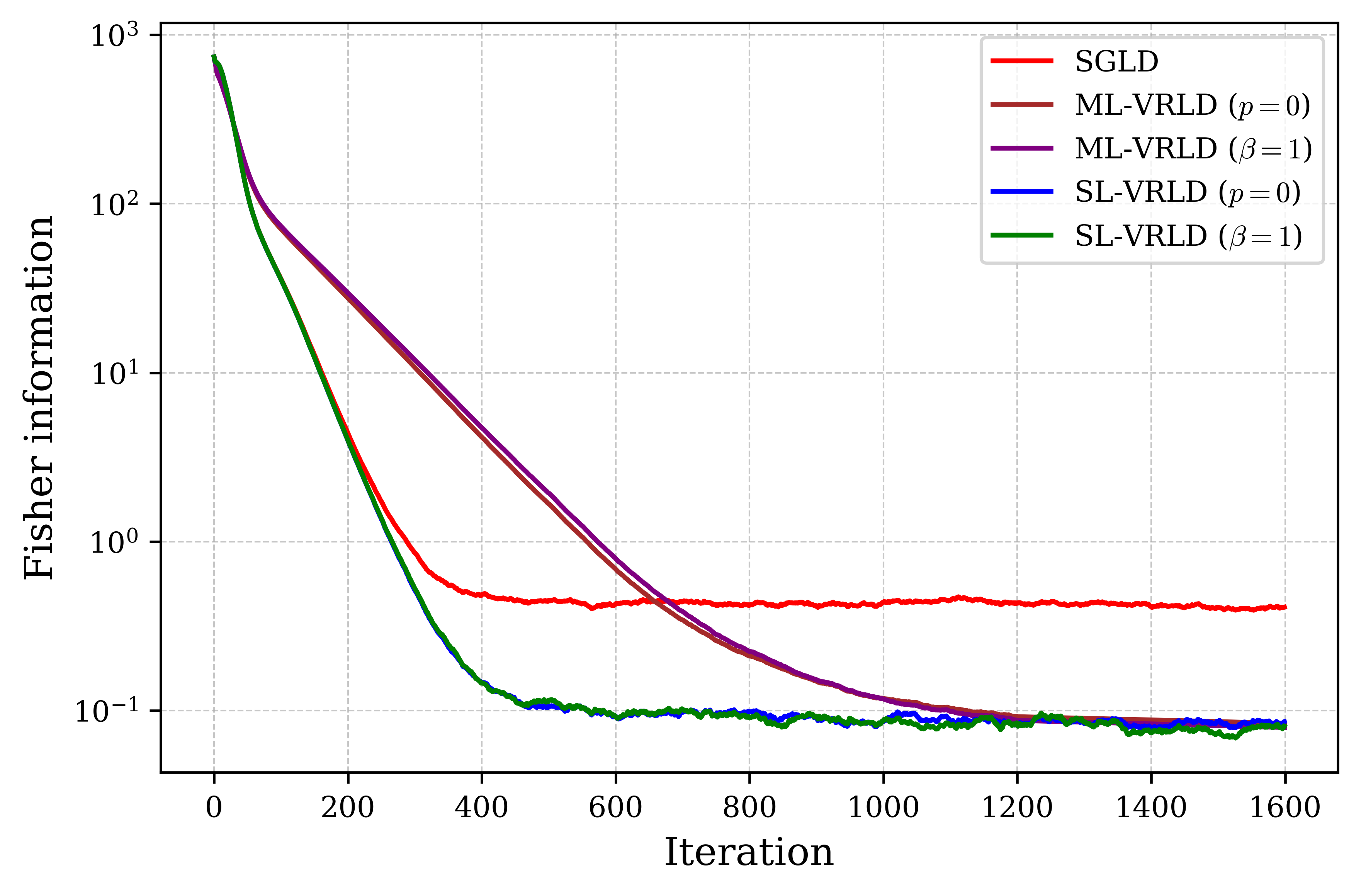}
    \caption{Results on a synthetic 2D inverse problem showing FI convergence for each method with batch size $b\!=\!6$. SGLD converges to a suboptimal solution, whereas ML-VRLD and SL-VRLD attain lower FI, consistent with our theory.}
    \label{fig:fi-convergence}
\end{figure}

\textbf{Numerical Validation with Compressed Sensing. } We run SGLD, ML-VRLD and SL-VRLD estimators with the following iterations:
\begin{equation}
    \vx_{(k+1)\gamma} = \vx_{k\gamma} - \gamma (\vg_k + \alpha_k\nabla h_{\sigma_k}(\vx_{k\gamma})) + \sqrt{2\gamma} \mZ_k,
\end{equation}where $\mZ_k\sim \gN(0,I)$ and $\vg_k$ is defined according to the estimator definition. We use the following annealing schedule defined in~(\ref{def:annealing-smoothing-schedules}):
\begin{equation}
    \alpha_k \coloneqq \max\{\alpha_0 \rho_1^k, 1\} \quad \text{and} \quad \sigma_k \coloneqq \max\{\sigma_0 \rho_2^k, \sigma_{\text{min}}\}, 
\end{equation}for $k\in\{0,\ldots,N\}$, where $\alpha_0,\sigma_0>0$ are initial values, $\sigma_{\mathrm{min}}>0$ is the minimum smoothing level, and $\rho_1,\rho_2\in(0,1)$ are decay rates. We recall that we consider a 2D Gaussian mixture prior for $p(\vx)$, thus $\nabla h_{\sigma_k}(\vx_{k\gamma})$ is analytically tractable by the following formula
\begin{equation}
    p_{\sigma_k}(\vx)\coloneqq \int_{\sR^d}p(\vz)\gN(\vx|\vz,\sigma_k^2I),
\end{equation}where $p_{\sigma_k}(\vx)\propto e^{-h_{\sigma_k}(\vx)}$. In our experiments, we set $\sigma_0=10$, $\alpha_0=10$, $\rho_2=0.975$, $\sigma_{\mathrm{min}}=0$, and define $\rho_1^k\coloneqq \sigma_k^2$. To simulate SGM approximation error, we assume $\|\gS_{\theta}(\vx, \sigma_k) + \nabla h_{\sigma_k} (\vx)\|\leq \varepsilon_{\sigma_k}$ and perturb the ground-truth score $-\nabla h_{\sigma_k}(\vx)$ with additive Gaussian noise, i.e., $-\nabla h_{\sigma_k}(\vx) + \varepsilon_{err}$, $\varepsilon_{err}\sim\gN(0, I)$, followed by element-wise clipping of the noise to a maximum magnitude of $\varepsilon_{max}=2.5$. We consider a linear inverse problem $\vy = \mA\vx + \vn$ with AWGN $\vn\sim\gN(0, I)$. The measurement operator $\mA\in\sR^{1\times 2}$ is constructed by sampling its entries independently from $\gU(-1,1)$, followed by normalization with respect to its Frobenius norm. We run each method (i.e. SGLD, ML-VRLD, and SL-VRLD) with 1000 sample points initialized according to a uniform distribution $\gU([-50,50]^2)$ over $[-50,50]^2$ grid and run for $N=1600$ iterations. Moreover, we add synthetic AWGN with standard deviation $\sigma=10.0$ to the likelihood score gradient to simulate stochastic gradient estimates. To estimate the Fisher information at each iteration, we fit a two-component Gaussian mixture model (GMM) to the generated samples, thereby obtaining a parametric approximation of the underlying distribution of samples. This fitted density enables evaluation of the probability at any point on a predefined grid over $[-50,50]^2$. Then, for each intermediate sample distribution, we compute the Fisher information for each unit area. The total sum over the grid gives the approximate relative Fisher information. 

For SGLD, ML-VRLD $(p=0)$, and SL-VRLD $(p=0)$, we use a batch size of $b=6$ and a step size of $\gamma = 0.01$. The momentum parameter is set to $\beta=0.9$ for ML-VRLD $(p=0)$ and $\beta=0.99$ for SL-VRLD $(p=0)$. For ML-VRLD ($\beta=1$) and SL-VRLD ($\beta=1$), we use large and small batch sizes of 8 and 5, respectively, with large batch size computation probability $p=0.4$, and set the step size to $\gamma=0.004$. This yields an expected batch size of 6. We intentionally use different step sizes for SL-VRLD and ML-VRLD, as ML-VRLD exhibits improved performance with a smaller step size, consistent with the theoretical results in Theorems~\ref{thm:ml-vrld-ficonv} and~\ref{thm:sl-vrld-ficonv}. Each method is evaluated over 20 independent realizations of $\mathbf{A}$, with all other settings fixed. The mean performance across runs is reported in Fig.~\ref{fig:fi-convergence-synthetic-problem}.

\textbf{Statistical Validation. } We validate that ML-VRLD and SL-VRLD can estimate the posterior statistics more accurately than SGLD under the same batch size. We construct a unimodal Gaussian prior by using the sample statistics of female images in CelebA~\citep{liu2018large}. Each image is downscaled to $32\times 32$ and normalized to $[-1,1]$. The measurement operator $\mathbf{A}\in\sR^{115\times 1024}$ is linear, with entries sampled independently from $\mathcal{N}(0,0.01)$. The measurement noise in~(\ref{def:inverse_problem}) is modeled as $\mathbf{n} \sim \mathcal{N}(0,0.01)$. Similar to our numerical experiments, we use the exact smoothed score $\nabla h_{\sigma_k}(\vx)$ with the following annealing schedule parameters: $\alpha_0=1$, $\sigma_0=1$, and $\sigma_{\min}=0.015$. For each method, 1000 initial samples are drawn uniformly from the hypercube $[-50,50]^{1024}$, i.e., $\mathbf{x}_0 \sim \mathcal{U}([-50,50]^{1024})$, and each algorithm is run for $N=5000$ iterations. To simulate stochastic gradients within a mini-batch, we add AWGN with standard deviation $\sigma=30.0$ to the likelihood score.

\begin{figure}[t]
    \centering
    \includegraphics[width=0.8\linewidth]{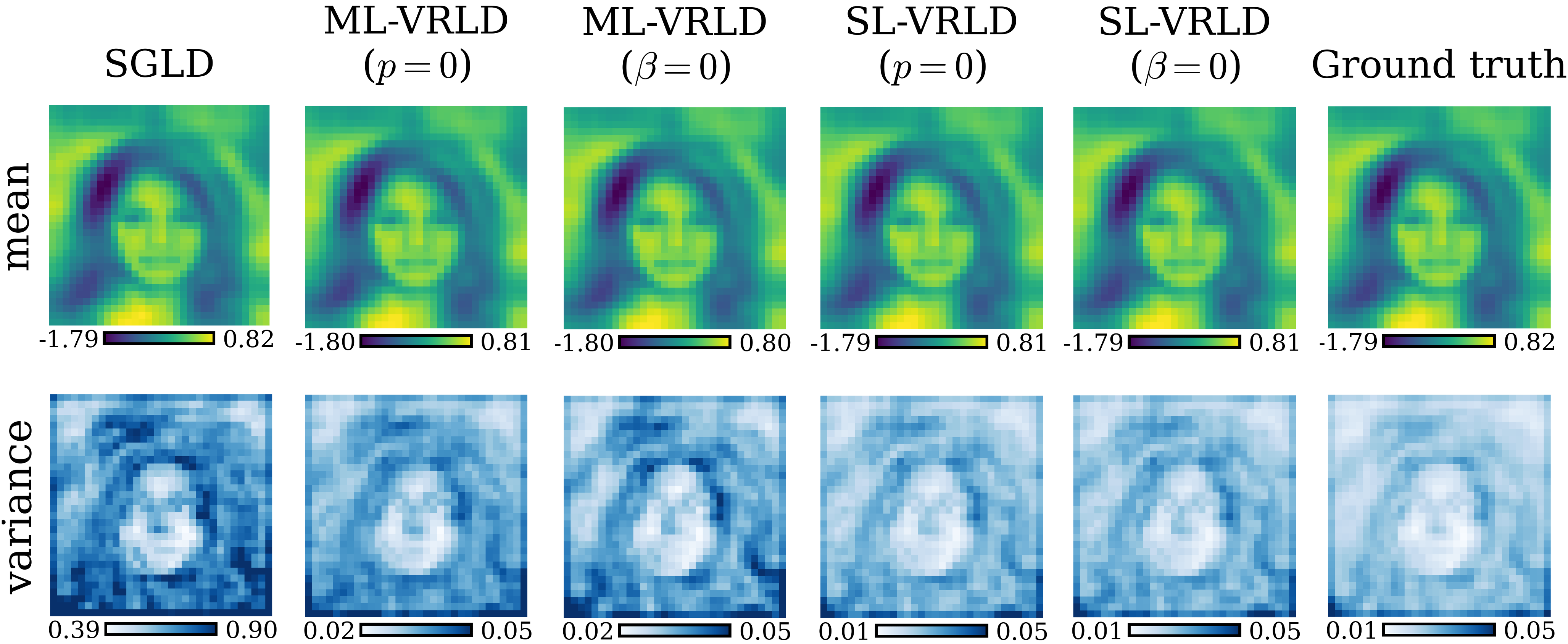}
    \caption{Results on a synthetic inverse problem used for statistical validation, where the posterior mean and variance are known in closed form. ML-VRLD and SL-VRLD yield more accurate variance estimates than SGLD under both the $\beta=1$ and $p=0$ configurations, whereas SGLD produces samples with substantially inflated variance.}
    \label{fig:statistical-validation-appendix}
\end{figure}

Although Fig.~\ref{fig:statistical-validation} reports results for the setting $p=0$, we also conduct experiments with $\beta=1$ and present those results in Fig.~\ref{fig:statistical-validation-appendix}. For SGLD, ML-VRLD $(p=0)$, and SL-VRLD $(p=0)$, we use a batch size of $b=6$ and a step size of $\gamma = 4\times 10^{-4}$. For ML-VRLD and SL-VRLD with $\beta=1$, we use large and small batch sizes $b_L=8$ and $b_S=5$, respectively, with large batch size computation probability $p=0.4$, yielding an expected batch size of $\E_p[b]=6$. The step size is set to $\gamma=10^{-4}$. As shown in Fig.~\ref{fig:statistical-validation-appendix}, SGLD produces samples with significantly higher variance than the ground truth variance. ML-VRLD and SL-VRLD substantially reduce this variance, with SL-VRLD achieving the closest match to the ground-truth posterior variance under both the $\beta=1$ and $p=0$ configurations, highlighting the effectiveness of the variance-reduction mechanism.

\subsection{Accelerated Parallel MRI Reconstruction}\label{appendix:mri-reconstruction}We consider a multi-coil MRI reconstruction task with a radial subsampling pattern. The inverse problem is defined as
\begin{equation}
    \vy = \sum_{i=1}^K \mP\mF\mS_i \vx + \vn,
\end{equation}
where $\mA_i \coloneqq \mP\mF\mS_i$ denotes the measurement operator corresponding to the $i$-th coil. 
Here, $\mP$ is a diagonal subsampling matrix, $\mF$ denotes the Fourier transform matrix, and $\mS_i$ represents the receiver coil sensitivity map. We set $\vn\sim\gN(0,\sigma_{\vn}^2I)$ such that the input signal-to-noise ratio (SNR) corresponds to 40 dB. Then the likelihood is $p(\vy|\vx)=\gN(\mA\vx, \sigma_{\vn}^2I)$, so the score of the likelihood at $k^{th}$ iteration is
\begin{equation}
    \nabla f(\vx_{k\gamma}) = -\frac{1}{\sigma_{\vn}^2}\sum_{i=1}^K\mA_i^H\left(\vy - \sum_{j=1}^K \mA_j\vx_{k\gamma}\right), 
\end{equation}where $\mA^H$ denotes the Hermitian adjoint (conjugate transpose) of $\mA$. In our experiments, for all methods, we sample a subset of coils $\xi_k\subset \{1,\ldots,K\}$ with $|\xi_k| = b$. Then, an unbiased mini-batch stochastic gradient estimate is given by
\begin{equation}
    \nabla f(\mathbf{x}_{k\gamma}, \xi_k)
    =
    -\frac{K}{b \sigma_{\vn}^2}
    \sum_{i \in \xi_k}
    \mA_i^H \left( \vy_i - \mA_i \vx_{k\gamma} \right),
\end{equation}
where $\vy_i$ denotes the measurements from the $i$-th coil, and 
$\xi_k \subset \{1,\ldots,K\}$ is a uniformly sampled subset of size $b$, i.e.,
\begin{equation}
\mathbb{P}(\xi_k = S)
=
\frac{1}{\binom{K}{b}},
\quad
\forall S \subset \{1,\ldots,K\},\ |S| = b.
\end{equation} We conduct our experiments with two different acceleration factors $R\in\{8, 12\}$, representing highly undersampled regimes used in fast and volumetric MRI acquisitions. We synthesize the multiple coils by using \texttt{sigpy.mri.birdcage\_maps} function from \texttt{SigPy}~\citep{ong2019sigpy}. We compare ML-VRLD and SL-VRLD against SGLD, used as a stochastic gradient baseline, as well as a diverse set of competing approaches spanning diffusion-based, denoiser-based, and classical regularization frameworks. Specifically, we include DPS, PnP-DM, PnP-FISTA, RED, and TV as representative state-of-the-art baselines. We use brain MRI data from the fastMRI dataset~\citep{zbontar1811fastmri} to train the denoiser and diffusion models used in the baseline methods, and to construct the test set. The training and test sets consist of 70,684 and 100 slices, respectively, where each slice has a shape of $256\times 256$. For ML-VRLD, SL-VRLD, and SGLD methods, we use the annealing schedule defined in~(\ref{def:annealing-smoothing-schedules}) with the following hyperparameters: $\alpha_0=10^4$, $\sigma_0=348$, $\rho_2 = 0.99$, $\sigma_{\mathrm{min}}=0.01$, and define $\rho_1^k \coloneqq \sigma_k^2$. These hyperparameters are selected based on dataset-specific characteristics, following the guidelines provided in~\citep{song2020improved}. We adopt these parameters from~\citep{song2020improved, sun2024provable}. In our experiments, we fix $p=0$ for both ML-VRLD and SL-VRLD, and tune the step size $\gamma$ and momentum parameter $\beta$ using \texttt{Optuna}, based on performance evaluated on a validation set of 5 samples. Specifically, for an acceleration factor of $8$, we set the step sizes for SGLD, ML-VRLD, and SL-VRLD to $\gamma=5\times10^{-6}$, $\gamma=7.787\times10^{-6}$, and $\gamma=1.83\times10^{-5}$, respectively. The corresponding momentum parameters for ML-VRLD and SL-VRLD are $\beta=0.318$ and $\beta=0.155$. For an acceleration factor of $12$, the step sizes are $\gamma=5\times10^{-6}$, $\gamma=7.899\times10^{-6}$, and $\gamma=3.452\times10^{-5}$ for SGLD, ML-VRLD, and SL-VRLD, respectively, with corresponding momentum parameters $\beta=0.208$ and $\beta=0.922$ for ML-VRLD and SL-VRLD. For each posterior sampling method, we generate 5 samples by using batch size of $b=32$ coils at each iteration and report the mean sample results. Furthermore, we implement the denoiser-based methods and TV regularization using \texttt{DeepInverse}~\citep{tachella2025deepinverse}. For the training and evaluation of DPS and PnP-DM, we use \texttt{InverseBench}~\citep{zheng2025inversebench}. Visual and quantitative results are presented in Fig.~\ref{fig:visual-comparison} and Table~\ref{tab:quantitative-comparison}, respectively.

```latex
\begin{table}[t]
\centering
\renewcommand\arraystretch{1.1}
\setlength{\tabcolsep}{5pt}

\caption{Peak GPU memory consumption (GB) per iteration for generating a single sample for each method. Lower values indicate better memory efficiency. All results were obtained from a single A100 NVIDIA GPU.}
\label{tab:memory}

\begin{tabular}{lcc cc}
\toprule
& \multicolumn{2}{c}{\textbf{MRI}}
& \multicolumn{2}{c}{\textbf{CT}} \\

\cmidrule(r{20pt}){2-3}
\cmidrule(r){4-5}

\textbf{Method}
& Accel. = $8\times$
& Accel. = $12\times$
& $360$ views
& $180$ views \\
\midrule

TV
& 0.38 & 0.38 & 1.6 & 0.81 \\

RED
& 0.71 & 0.71 & 1.96 & 1.16 \\

PnP-FISTA
& 0.71 & 0.71 & 1.96 & 1.16 \\

DPS
& 2.53 & 2.53 & 3.76 & 2.96 \\

PnP-DM
& 0.71 & 0.71 & 1.96 & 1.16 \\

SGLD
& 0.71 & 0.71 & 1.68 & 1.15 \\

\midrule

ML-VRLD
& 0.71 & 0.71
& 1.68 & 1.15 \\

SL-VRLD
& 0.71 & 0.71
& 1.68 & 1.15 \\

\bottomrule
\end{tabular}
\label{tab:mem-results}
\end{table}

\subsection{Sparse-angle CT Reconstruction}\label{appendix:ct-reconstruction}

We study sparse-angle CT reconstruction, where the goal is to recover high-quality images from a restricted set of projection views. The limited angular coverage renders the problem ill-posed due to incomplete sampling. We simulate parallel-beam CT measurements using 360 and 180 projection views, each with 258 detector elements. The forward operator $\mA$ is implemented via \texttt{deepinv.physics.Tomography} from the \texttt{DeepInverse} package. Additional details regarding the acquisition geometry and the implementation of the measurement operator are provided in the \texttt{DeepInverse} documentation. To state the inverse problem formally, let $\vx\in \sR^d$ denote the unknown CT image and let $\mR\in\sR^{(KM)\times N}$ denote the full discrete Radon (parallel-beam) projection operator over $M$ projection views and $K$ detector elements per view. We assume $\mR$ is stacked by views 
\begin{equation}
    \mR = \begin{bmatrix}
        \mR_1\\
        \mR_2\\
        \vdots\\
        \mR_M
    \end{bmatrix}, \quad \mR_m\in \sR^{K\times N}.
\end{equation}Thus the full sinogram is $\mR\vx \in \sR^{KM}$. To define the sparse-angle sampling operator, let $\xi_k \subset \{1, \ldots, M\}$ denote the index set of measured views at the $k$-th iteration, where $|\xi_k| = b \ll M$, and $b$ denotes the batch size. The measured vies are selected according to a uniform distribution given as 
\begin{equation}
    \sP(\xi_k = S) = \frac{1}{\binom{M}{b}}\quad \text{for every } \quad S\subset\{1,\ldots,M\}, |S|=b,
\end{equation}and the view-selection matrix $P_\xi\in\{0,1\}^{b\times M}$ is defined by
\begin{equation}
    (\mP_{\xi_k})_{i,m} = \begin{cases}
    1, & m = \omega_i, \\
    0, & \text{otherwise},
    \end{cases}
    \qquad
    \text{where } \xi_k = \{\omega_1, \ldots, \omega_{b}\}.
\end{equation}Then the sampling operator for projection views is
\begin{equation}
    \mS_{\xi_k}\coloneqq \mP_{\xi_k} \otimes \mI_K \in \{0,1\}^{(Kb)\times (KM)},
\end{equation}where $\otimes$ denotes the Kronecker product and $I_k$ is the $K\times K$identity matrix. Then the sparse-angle measurement operator is
\begin{equation}
    \mA \coloneqq \mS_{\xi_k}\mR\in \sR^{(Kb)\times d}. 
\end{equation} We consider the linear inverse problem in~(\ref{def:inverse_problem}) with AWGN $\vn\sim\gN(0, \sigma_{\vn}^2I)$. Thus the likelihood is given as $p(\vy|\vx)=\gN(\mA\vx, \sigma_{\vn}^2I)$ and the likelihood score at step $k$ is given as
\begin{equation}
    \nabla f(\vx_{k\gamma}) = -\frac{1}{\sigma_{\vn}^2}\mA^T(\vy - \mA\vx). 
\end{equation} In our experiments, we use mini-batch stochastic gradient estimates so we consider the following decompositions of $\mA$ and $\vy$:
\begin{equation}
    \mA = \begin{bmatrix}
        \mA_1\\
        \vdots\\
        \mA_M
    \end{bmatrix} \quad \text{and} \quad 
    \vy = \begin{bmatrix}
        \vy_1\\
        \vdots\\
        \vy_M
    \end{bmatrix},
\end{equation}where $\mA_i\in \sR^{K\times d}$ and $\vy_i\in \sR^K$. Then the likelihood score becomes a sum
\begin{equation}
    \nabla f(\vx_{k\gamma}) = -\frac{1}{\sigma_{\vn}^2}\sum_{i=1}^M\mA_i^T(\vy_i - \mA_i \vx).
\end{equation}Hence, given a mini-batch views $\xi_k\subset \{1,\ldots,M\}$ with $|\xi_k| = b$, sampled uniformly at random, we obtain an unbiased stochastic gradient estiamte
\begin{equation}
    \nabla f(\vx_{k\gamma}, \xi_k) = -\frac{M}{b\sigma_{\vn}^2}\sum_{i\in\xi_k}\mA_i^T(\vy_i - \mA_i\vx).
\end{equation}In our experiments, we use this formulation for mini-batch stochastic likelihood score estimates. Moreover, we adopt the same baselines as in the MRI experiments, except that zero-filled reconstruction is replaced by filtered back-projection (FBP) with a Ramp filter, which is more suitable for CT reconstruction. For all methods, we consider two batch sizes: 60 views out of 360 and 30 views out of 180 projection views, respectively. We use abdominal CT scans downscaled to $256\times256$ from the 2DeteCT dataset~\citep{kiss20232detect} for our training and test sets. The training set consists of 4,500 slices and is used to train: (i) the denoiser for RED and PnP-FISTA, (ii) the diffusion model for DPS and PnP-DM, and (iii) the SGM for SGLD, ML-VRLD, and SL-VRLD. All methods are evaluated on 100 slices and mean of 5 samples are reported for each posterior sampling algorithm. We use the same annealing schedule and parameters in MRI reconstruction experiments. For ML-VRLD and SL-VRLD, we use the configuration with $p=0$ and optimize the step size $\gamma$ and momentum parameter $\beta$ using \texttt{Optuna} over 5 samples as a validation set. Specifically, for total number of projection views 360, we set the step sizes for SGLD, ML-VRLD, and SL-VRLD, to $\gamma=5.76\times 10^{-6}$, $\gamma=1.11\times 10^{-5}$, and $\gamma=1.35\times 10^{-5}$, respectively. The corresponding momentum parameters for ML-VRLD and SL-VRLD are $\beta=0.583$ and $\beta=0.774$. For the experiments with total number of projection views 180, we set the step sizes for SGLD, ML-VRLD, and SL-VRLD, to $\gamma=5.16\times 10^{-6}$, $\gamma=2.56\times10^{-5}$, and $\gamma=1.7\times 10^{-5}$, respectively. The corresponding momentum parameters are for ML-VRLD and SL-VRLD are $\beta=0.41$ and $\beta=0.988$. We use the same implementations of baseline methods in MRI experiments. Visual and quantitative results are presented in Fig.~\ref{fig:visual-comparison} and Table~\ref{tab:quantitative-comparison}. The implementations of SGLD, ML-VRLD, and SL-VRLD are provided in the supplementary material, along with representative test cases to facilitate reproducibility.

\end{document}